\documentclass[a4paper,11pt]{article}
\RequirePackage{xcolor}
\usepackage{amsmath,amsthm}
\usepackage{comment}
\usepackage{multirow}
\usepackage[numbers]{natbib}
\usepackage{graphicx}
\usepackage[T1]{fontenc}
\usepackage[utf8]{inputenc}

\usepackage{threeparttable}
\usepackage{makecell}
\newcommand{\makecellnew}[1]{{\renewcommand{\arraystretch}{0.8}\begin{tabular}{c} #1 \end{tabular}}}

\usepackage[english]{babel}
\usepackage{amssymb,amsfonts}
\usepackage{bbm}
\usepackage{nicefrac}
\usepackage{lmodern}
\usepackage{epigraph}
\usepackage{subcaption}
\usepackage{csquotes}
\usepackage{booktabs}
\usepackage{geometry}
\usepackage{thm-restate}
\usepackage{empheq}
\usepackage{authblk}
\usepackage{color}
\definecolor{darkgreen}{rgb}{0.0,0.5,0.0}

\usepackage[pagebackref]{hyperref}
\hypersetup{
	colorlinks=true,        
	linkcolor=blue,         
	filecolor=magenta,
	citecolor=darkgreen,         
	urlcolor=blue           
}
\renewcommand*{\backrefalt}[4]{%
	\ifcase #1 \footnotesize{(Not cited.)}%
	\or        \footnotesize{(Cited on page~#2)}%
	\else      \footnotesize{(Cited on pages~#2)}%
	\fi}

\usepackage{framed}
\usepackage{algorithm}
\usepackage{algpseudocode}
\usepackage{enumerate}
\usepackage{cleveref}

\newtheorem{assumption}{Assumption}
\theoremstyle{definition}
\newtheorem{definition}{Definition}
\newtheorem{lemma}{Lemma}

\newtheorem{lemmasec}{Lemma}[section]
\newtheorem{propositionsec}{Proposition}[section]

\newtheorem{remark}{Remark}

\newcommand{\avg}{\text{avg}}

\newcommand{\E}[1]{\mathbb{E}\left[#1\right]}

\newcommand{\EE}{\mathbb{E}}

\newcommand{\cA}{\mathcal{A}}
\newcommand{\cR}{\mathcal{R}}
\newcommand{\cS}{\mathcal{S}}
\newcommand{\cD}{\mathcal{D}}

\newcommand{\cN}{\mathcal N}
\newcommand{\cO}{\mathcal O}

\newcommand{\R}{\mathbb R}
\newcommand{\wtilde}[1]{\widetilde{#1}}
\newcommand{\wtau}{\widetilde{\tau}}

\renewcommand{\empty}{\varnothing}
\newcommand{\eqdef}{\coloneqq}

\def\<#1,#2>{\langle #1,#2\rangle}
\newcommand{\squeeze}{}

\usepackage[colorinlistoftodos,bordercolor=purple,backgroundcolor=pink!20,linecolor=pink,textsize=scriptsize]{todonotes}

\def\toptitlebar{\hrule height1pt \vskip .25in} 
\def\bottomtitlebar{\vskip .22in \hrule height1pt \vskip .3in} 

\title{
\toptitlebar
{{\center\baselineskip 18pt
                      {\Large\bf {\tt AsGrad}: A Sharp Unified Analysis of\\ Asynchronous-SGD Algorithms}}
} 
\bottomtitlebar}
\date{}
\author[1]{Rustem Islamov\thanks{This research was conducted while this author was an intern at IST Austria and a Master student at IP Paris.}}
\author[2]{Mher Safaryan}
\author[2]{Dan Alistarh}
\affil[1]{Institut Polytechnique de Paris (IP Paris)}
\affil[2]{Institute of Science and Technology Austria (IST Austria)}

\begin{document}

\maketitle

\begin{abstract}
    We analyze asynchronous-type algorithms for distributed SGD in the heterogeneous setting, where each worker has its own computation and communication speeds, as well as data distribution. In these algorithms, workers compute possibly stale and stochastic gradients associated with their local data at some iteration back in history and then return those gradients to the server without synchronizing with other workers. We present a unified convergence theory for non-convex smooth functions in the heterogeneous regime. The proposed analysis provides convergence for pure asynchronous SGD and its various modifications. Moreover, our theory explains what affects the convergence rate and what can be done to improve the performance of asynchronous algorithms. In particular, we introduce a novel asynchronous method based on worker shuffling. As a by-product of our analysis, we also demonstrate convergence guarantees for gradient-type algorithms such as SGD with random reshuffling and shuffle-once mini-batch SGD. The derived rates match the best-known results for those algorithms, highlighting the tightness of our approach. Finally, our numerical evaluations support theoretical findings and show the good practical performance of our method.
\end{abstract}

\section{Introduction}

Modern machine learning  relies heavily on effective  optimization algorithms. The stochastic gradient descent (SGD) method \cite{Robbins1951SGD} and its various modifications, such as Adam \cite{ADAM} and AdaGrad \cite{ADAGRAD} are at the core of machine learning training, due to their easy implementation together with strong practical performance. However, in recent times, both the size of state-of-the-art models and the amount of data required for training have increased significantly. Due to this growth, optimization algorithms have also been adapted to the need for efficient training of large models, and thus distributed and parallel variants of SGD have started playing a crucial role in modern machine learning \cite{bekkerman2011scaling}. In the  distributed regime, training is performed using many computational  nodes (e.g., CPUs or GPUs on a cluster) working in parallel and orchestrated by a server. Every worker computes gradients based on available data, and then a server aggregates those gradients to perform one step of an algorithm. Distributed SGD-based algorithms are also adopted to Federated Learning applications \cite{FEDLEARN, FL-big} where local data is kept private and is not seen by other workers. 

Nevertheless, distributed variants of SGD suffer from many practical challenges. For example, approaches such as communication compression \cite{alistarh2018convergence, DCGD, mishchenko2019DIANA}, performing several local steps before communication \cite{mishchenko2022proxskip, gorbunov2021local, koloskova2021topology}, decentralized communication \cite{koloskova2019Decentralized, kovalev2021lower}, or their combinations \cite{condat2023provably, zakerinia2023communicationefficient} are designed to improve the efficiency of distributed training.  

Across all these approaches, workers are synchronized, i.e., the server must wait for the slowest worker before proceeding to the next algorithm iteration. This may drastically slow down the performance of SGD if workers have significantly different computational power. 
\emph{Asynchronous} communication, which breaks this lock-step behavior, is preferable in practice, since it enables a more efficient use of resources. In asynchronous-type algorithms, workers do not wait for each other, thus a server immediately updates the current model and assigns a new job to available workers. All of the aforementioned approaches are orthogonal to asynchronicity, and can be combined with it \cite{nguyen2022fedbuff, zakerinia2023communicationefficient}.

Recently, \cite{koloskova2022sharper, mishchenko2022asynchronous} improved theoretical analysis of pure asynchronous SGD in the \emph{homogeneous data} setting, i.e., when workers have access to the same data, and make a step towards better understanding in a \emph{heterogeneous} regime. Here, we specifically focus on the more challenging heterogeneous setting, and provide a unified convergence analysis of asynchronous SGD. 

We start from the observation that asynchronous SGD can be seen as an instance of SGD with arbitrary data orderings, which include random reshuffling at each step, shuffle once (at the beginning of training) or some incremental order. The only difference between these variants is in the order defined naturally according to the computation speeds of the workers. Data ordering might improve the performance of the algorithm. For example, several works \cite{mishchenko2020rr, nguyen2021unified} show that SGD with random reshuffling and shuffle once is always better than vanilla SGD in the strongly convex case, and can outperform it for both convex and non-convex objectives if the number of epochs if sufficiently large \cite{mohtashami2022characterizing, lu2022general}. Recently, \cite{koloskova2023shuffle} analyzed SGD with arbitrary data ordering as an algorithm with linearly correlated noise \cite{koloskova2023gradient}. Here, we go further by linking asynchronous SGD and SGD using various data orderings.

\subsection{Contributions}
In this section, we summarize the key contributions of our work. 

\begin{itemize}
\item We propose a theoretical framework {\tt AsGrad} that allows us to analyze various types of asynchronous SGD in a unified manner. Besides purely-asynchronous SGD, we analyze variants of asynchronous SGD when a server waits for the first $b$ fastest workers or assigns new jobs according to some scheduling procedure. The analysis is performed for a constant stepsize schedule, without bringing any additional hyper-parameters except for stepsize, as all the parameters used in the algorithm are known during training. 

\item Moreover, this unified framework enables us to design and analyze new asynchronous algorithms. In particular, we propose a new method called {\em shuffled asynchronous SGD} and show its superiority over competitive methods both theoretically and practically. 

\item Our framework also recovers popular synchronous variants of SGD, such as SGD with random reshuffling, shuffle once, and mini-batch SGD. For these methods, we derive the best-known convergence results, without any changes in the analysis, highlighting the tightness of our framework. 

\item All of our results have a better or similar dependencies on the maximum delay, compared with existing work. With bounded gradient assumption,  we remove entirely dependencies on maximum delay used by prior works.

\end{itemize}

\section{Related Works}

\subsection{Asynchronous SGD}
The literature on asynchronous-type SGD algorithms is extremely rich, starting from Baudet, 1978~\cite{baudet1978asynchronous}, and Bertsekas and Tsitsiklis, 1989~\cite{bertsekas1989parallel}. The source of asynchrony might be caused by various factors, such as different hardware~\cite{horvath2022fjord} or message passing failures \cite{nadiradze2020elastic}.

Most existing works are devoted to the analysis of SGD-based asynchronous algorithms, concentrating on the homogeneous regime \cite{cohen2021asynchronous, aviv2021learning, stich2021critical} which typically holds only for shared memory architectures. In our work, we concentrate on the more challenging heterogeneous setting. 

The first convergence guarantees were given for constant delays \cite{stich2021errorfeedback, arjevani2020tight} which usually is not true in practice. Follow-up studies provide an analysis depending on the maximum delay \cite{stich2021critical, zakerinia2023communicationefficient, nguyen2022fedbuff}, however, it can be much larger than the average delay. \cite{alistarh2018convergence} improved the dependency on the maximum delay under the bounded gradient assumption while \cite{cohen2021asynchronous, aviv2021learning} completely removed that dependency. \cite{cohen2021asynchronous} proposed a method which sends delayed models together with computed gradients, and additional parameter tuning while \cite{aviv2021learning} derived rates involving the variance of the delays. Nonetheless, in some cases, this quantity might be proportional to the maximum delay. Recently, \cite{tyurin2023optimal} introduced a new asynchronous algorithm with optimal convergence guarantees; however, under the assumption that the computational speeds of workers are \emph{fixed} throughout the execution. This assumption is too restrictive in many practical scenarios, where computational speeds usually fluctuate  during the training.

Several works consider  asynchronous SGD in Federated Learning, where workers frequently have heterogeneous computational power.  \cite{zheng2017asynchronous, zhang2016stalenessaware} utilize delay-adaptive stepsizes to mitigate the effect of asynchrony while \cite{yan2020distributed, glasgow2021asynchronous, gu2021fast} proposed variance-reduction-based mechanisms to handle different worker availability.  \cite{nguyen2022fedbuff} proposed a method that incorporates local steps and shows its practical superiority, but the analysis is done under unrealistic assumptions. 

Recently, \cite{koloskova2022sharper, mishchenko2022asynchronous} introduced novel analysis based on perturbed iterates framework \cite{mania2017virtual} which improves the convergence guarantees for asynchronous SGD for both heterogeneous and homogeneous regimes. They propose similar delay-adaptive stepsize schedules that allow to derive maximum-delay-free rates for asynchronous SGD. Besides, \cite{koloskova2022sharper} proposed an algorithm with a special random job assigning procedure to balance the workers in the heterogeneous setting. In the same setting \cite{mishchenko2022asynchronous} derived convergence guarantees for pure asynchronous SGD. Inspired by \cite{koloskova2023shuffle}\footnote{We aware of the technical gap in the proofs in \cite{koloskova2023shuffle}. However, a proper choice of correlation period $\tau$ allows to take conditional expectations correctly. We highlight for each special case the choice of  $\tau$ that leads to correct statements.}, we propose a new virtual-iterates-based analysis and cover the results of \cite{koloskova2022sharper, mishchenko2022asynchronous} as special cases.

\subsection{SGD with Arbitrary Data Ordering}
The most practical training SGD schemes utilize a certain order (usually random) of data samples. Schemes like random reshuffling or shuffle once are considered a default choice in many real-world applications. Sampling without replacement allows to leverage of the finite-sum structure of the problem since all data points will contribute equally to the solution. However, the theoretical guarantees for these methods are much less studied than for vanilla SGD \cite{rakhlin2012making, nguyen2021unified}. The main complication comes from the fact that the gradient estimator becomes biased. Biased SGD-type methods are typically considered more challenging to analyze than their unbiased counterparts. 

In the last years, particular attention was given to federated and distributed methods with random reshuffling \cite{malinovsky2022serverside, sadiev2022federated, Yun2021CanSS, cho2023convergence}. Data shuffling might provably improve the convergence in this case as well which is important for the training of large-scale models.

Here, \cite{mishchenko2020rr} introduced novel techniques to improve the convergence guarantees in the strongly convex case, whereas \cite{nguyen2021unified, ahn2020sgd} demonstrated the same rates showing the superiority of SGD with random reshuffling over vanilla SGD. However, in convex and non-convex regimes this remains an open problem. Existing results demonstrate better performance of SGD with random reshuffling  in some special cases \cite{mishchenko2020rr} in convex and non-convex scenarios.

\section{Setup}

\begin{table*}[t]
    \centering
    \caption{Asynchronous algorithms whose convergence analysis is covered by {\tt AsGrad} framework. Here constants $\zeta^2$, $G^2$ such that $\|\nabla f_i(x) - \nabla f(x)\|^2 \le \zeta^2$ and $\|\nabla f_i(x)\|^2 \le G^2$ for all $i\in[n]$ and $x\in\R^d.$ $\tau_C$ and $\tau_{\max}$ denote the maximum number of active jobs and the maximum delay respectively. $F_0$ is the initial functional suboptimality, the definitions of other constants are given in Section~\ref{sec:conv_analysis}. For shuffled asynchronous SGD $\tau_C = n.$ \textbf{BG} = requires {\bf B}ounded {\bf G}radients assumption.}
    \label{tab:table2}
    \resizebox{\textwidth}{!}{
        \begin{tabular}{ccccc}
            \toprule
            \textbf{Method} & 
            \textbf{Alg $\#$} &
            \textbf{Citation} &
            \textbf{BG}  & 
            \textbf{Rate} ${}^{\text{(a)}}$
            \\ \toprule
            
             & 
             &
             \cite{mishchenko2022asynchronous} &
            No & 
            $\frac{LF_0\tau_{C}}{T} 
            + \left(\frac{LF_0\sigma^2}{T}\right)^{1/2} 
            + \zeta^2$ ${}^{\text{(b)}}$
            \\
            \makecellnew{Pure\\ Asynchronous SGD} &
            Alg~\ref{alg:pure_asynchronous} &
            Ours &
            No &
            $ \frac{LF_0\sqrt{\tau_{\max}\tau_C}}{T} 
            + \left(\frac{LF_0\sigma^2}{T}\right)^{1/2}   
            + \zeta^2$
            \\
            &
            &
            Ours &
            Yes &  
            $ \frac{LF_0\tau_C}{T} 
            + \left(\frac{LF_0\sigma^2}{T}\right)^{1/2} 
            + \left(\frac{LF_0G\tau_C}{T}\right)^{2/3}
            + \zeta^2$
            \\ \midrule

            \multirow{2}{*}{\makecellnew{ Pure\\ Asynchronous SGD\\with waiting}} &
            \multirow{2}{*}{\makecellnew{\vspace{-6pt}Alg~\ref{alg:pure_asynchronous_waiting} }} &
            Ours &
            No & 
            $ \frac{LF_0\sqrt{\tau_{\max}\tau_{C}}}{T\sqrt{b}} 
            + \left(\frac{LF_0\sigma^2}{Tb}\right)^{1/2}
            + \zeta^2$\\
            &
             &
             Ours &
            Yes & 
            $\frac{LF_0\tau_{C}}{Tb} 
            + \left(\frac{LF_0\sigma^2}{Tb}\right)^{1/2}
            + \left(\frac{LF_0G\tau_C}{Tb}\right)^{2/3}
            + \zeta^2$
            \\ \midrule

            &
            &
            \cite{koloskova2022sharper} &
            No & 
            
            $ \frac{LF_0\sqrt{\tau_{\max}\tau_{C}}}{T} 
            + \left(\frac{LF_0\sigma^2}{T}\right)^{1/2} 
            + \left(\frac{LF_0\zeta^2}{T}\right)^{1/2} 
            + \left(\frac{LF_0\tau_C\zeta}{T}\right)^{2/3}$\\
            \makecellnew{Random\\ Asynchronous SGD} &
            \makecellnew{\vspace{4pt}Alg~\ref{alg:random_asynchronous}} &
            \cite{koloskova2022sharper} &
            Yes &
             
            $ \frac{LF_0\tau_C}{T} 
            + \left(\frac{LF_0\sigma^2}{T}\right)^{1/2} 
            + \left(\frac{LF_0\zeta^2}{T}\right)^{1/2} 
            + \left(\frac{LF_0\tau_CG}{T}\right)^{2/3}$\\
            &
            &
            Ours &
            Yes &
            $ \frac{LF_0\tau_C}{T} 
            + \left(\frac{LF_0\sigma^2}{T}\right)^{1/2} 
            + \left(\frac{LF_0\zeta^2}{T}\right)^{1/2} 
            + \left(\frac{LF_0\tau_CG}{T}\right)^{2/3} $
            \\ \midrule

            \multirow{2}{*}{\makecellnew{Random\\\ Asynchronous SGD\\with waiting (FedBuff)}} &
            \multirow{2}{*}{\makecellnew{\vspace{-6pt}Alg~\ref{alg:random_asynchronous_waiting} }} & 
            \cite{nguyen2022fedbuff} &
            Yes & 
            
            $ \frac{LF_0}{T} 
            + \left(\frac{LF_0\sigma^2}{T}\right)^{1/2} 
            + \left(\frac{LF_0\zeta\tau_{\max}}{T}\right)^{2/3}
            + \left(\frac{LF_0G\tau_{\max}}{T}\right)^{2/3}$ ${}^{\text{(c)}}$\\
            &
             &
             Ours & 
            Yes &
            
            $ \frac{LF_0\tau_C}{T} 
            + \left(\frac{LF_0\zeta^2}{Tb}\right)^{1/2}
            + \left(\frac{LF_0\sigma^2}{Tb}\right)^{1/2}
            + \left(\frac{LF_0\tau_CG}{Tb}\right)^{2/3}$
            \\ \midrule

            \makecellnew{Shuffled \\ Asynchronous SGD \\ {\bf [NEW]}} & 
            Alg~\ref{alg:shuffled_asynchronous} &
            Ours &
            Yes & 
            
            $ \frac{LnF_0}{T} 
            + \left(\frac{LF_0\sigma^2}{T}\right)^{1/2}
            + \left(\frac{LF_0\sqrt{n}\zeta}{T}\right)^{2/3}
            + \left(\frac{LF_0Gn}{T}\right)^{2/3} $\\
            \bottomrule 
        
        \end{tabular}
        }
\begin{tablenotes}
      {\scriptsize 
        \item (a) We present the best-known rates under the same set of assumptions as we use in the analysis.
        \item (b) \cite{mishchenko2022asynchronous} uses delay adaptive stepsizes to get rid of the dependency on $\tau_{\max}.$
        \item (c) If we set $\eta_l = \frac{\gamma}{b}, \eta_g=b, Q=1$ in Theorem~$1$ \cite{nguyen2022fedbuff}. The analysis is done under the unrealistic assumption that $\{i_t\}_{t=0}^{T-1}$ are distributed uniformly at random. 
        }
    \end{tablenotes}  
\end{table*}

We consider the classical Empirical Risk Minimization (ERM) problem of the form
\begin{equation}\label{eq:problem}
\squeeze
    \min\limits_{x\in \R^d}\left[f(x) \eqdef \frac{1}{n}\sum\limits_{i=1}^nf_i(x)\right],
\end{equation}
which covers many optimization problems in machine learning. Here, $x \in \R^d$ denotes the parameters of a model we aim to train, $n$ is the number of workers participating in the distributed training, and $f_i(x)$ is the loss associated with data $\cD_i$ available to worker $i\in [n] \eqdef \{1, 2, \dots, n\}.$ For example, the local function $f_i$ can be written as an expectation $f_i(x) \eqdef \EE_{\xi \sim \cD_i}[f_i(x, \xi)]$ over dataset $\cD_i$ locally stored at worker $i$. Besides, the formulation \eqref{eq:problem} also recovers the single-node setting. In this scenario, $f$ represents a loss over a dataset of size $n$, and each $f_i$ is a loss associated with $i$-th data point. We denote the minimum of the problem \eqref{eq:problem} by $f^*,$ and assume that it is finite.

\begin{algorithm*}[t]
\caption{{\tt AsGrad} framework: General Asynchronous SGD}
\label{alg:pseudocode}
\begin{algorithmic}[1]
\State \textbf{Input:} $x_0\in \R^{d}$, stepsize $\gamma > 0$, set of assigned jobs $\cA_0 = \empty$, set of received jobs $\cR_{0} = \empty$
\State \textbf{Initialization:} for all jobs $(i, 0) \in \cA_1$, the server assigns worker $i$ to compute a stochastic gradient $g_i(x_0)$ 
    \For{$t = 0,1,2,\dots, T-1$}
		\State once worker $i_t$ finishes a job $(i_t,\pi_t) \in \cA_{t+1}$, it sends $g_{i_t}(x_{\pi_t})$ to the server
        \State server updates the current model $x_{t+1} = x_t - \gamma g_{i_t}(x_{\pi_t})$ and the set $\cR_{t+1} = \cR_{t} \cup \{(i_t, \pi_t)\}$
        \State server assigns worker $k_{t+1}$ to compute a gradient $g_{k_{t+1}}(x_{\alpha_{t+1}})$
        \State server updates the set $\cA_{t+2} = \cA_{t+1} \cup \{(k_{t+1}, \alpha_{t+1})\}$
    \EndFor
\end{algorithmic}	
\end{algorithm*}

\subsection{{\tt AsGrad}: Proposed Algorithmic Framework}
To solve the problem \eqref{eq:problem}, we apply general asynchronous SGD satisfying {\tt AsGrad} framework (Algorithm~\ref{alg:pseudocode}). The backbone of our algorithmic approach is to analyze updates of the form
\begin{equation}\label{eq:real_iterates_update}
    x_{t+1} = x_t - \gamma g_{i_t}(x_{\pi_t}),
\end{equation}
where $i_t \in [n]$ is an index of a worker whose potentially stale and stochastic gradient $g_{i_t}(x_{\pi_t})$ computed at an outdated point $x_{\pi_t}$ is applied at iteration $t$.

It is crucial to note that the order of applied gradients $g_{i_t}(x_{\pi_t})$ is not controlled by the server. Once an ongoing gradient computation is done by some worker, the gradient is immediately communicated to the server and is used to update the global model without synchronizing with other workers. In fact, the order of received gradients depends on the speeds of the workers which can change during the training. We denote by $\cR_t$ the set of all received jobs\footnote{By job we define a pair $(i,j)$ such that worker $i$ is assigned to compute $\nabla f_i(x_j)$ for a model from iteration $j$.} before iteration $t$. Hence, before training $\cR_0=\empty$, and once a new job $(i_t,\pi_t)$ is completed (i.e., the gradient $g_{i_t}(x_{\pi_t})$ is computed) we update the set $\cR_{t+1} = \cR_{t} \cup \{(i_t, \pi_t)\}$.

After receiving a completed job $(i_t,\pi_t)$ and updating the model parameters $x_{t+1}$, the server proceeds to assign a new job. This is the stage when the server can influence the training by controlling the assigned jobs. Specifically, the server is allowed to assign any worker $k_{t+1}\in[n]$ to compute a stochastic gradient at any point $x_{\alpha_{t+1}}$ in the history of model parameters. Similar to the set of received jobs, denote by $\cA_t$ the set of all assigned jobs before iteration $t$. Thus, before the training no job is assigned, i.e., $\cA_0=\empty$. Then, the training starts by the initial job assignments $\cA_1 = \{(i,0) \colon \text{ for some workers } i\in[n]\}$. Afterward, once a new job $(k_{t+1}, \alpha_{t+1})$ is assigned by the server, we update the set $\cA_{t+2} = \cA_{t+1} \cup \{(k_{t+1}, \alpha_{t+1})\}$.

By definition, $\cR_t \subseteq \cA_t$ since only assigned jobs might be finished. Besides, the set $\cA_{t+1}\setminus \cR_t$ represents the jobs that are ``in flight'' at iteration $t$. In particular, $\cA_{T+1}\setminus \cR_T$ consists of all jobs that are not finished within the optimization process.

The power of our algorithmic framework is its versatility to recover various variants of both asynchronous and synchronous SGD algorithms in one method. The strength of our approach is that the theoretical analysis for all of them is covered by unified theory, as well as the rates obtained by our theory match the best-known results for those cases or improve them. We present convergence guarantees for several synchronous and asynchronous methods covered by our analysis in Table~\ref{tab:table2} and Section~\ref{sec:conv_analysis} in more details.

\subsection{Special Cases}
Below we list some interesting special cases covered by our framework. Detailed description is provided in Sections~\ref{sec:special_cases_real_theorem1} and \ref{sec:special_cases_real_theorem4} while convergence rates are presented in Table~\ref{tab:table2}.

\paragraph{Pure Asynchronous SGD.} In the beginning of the training the server assigns jobs to all workers at $x_0$. Then, a new job (for a freshly updated model) is assigned back to the same worker which completed the previous job, i.e., $(k_{t+1}, \alpha_{t+1}) \equiv (i_t, t+1)$. We derive improved square root dependency on $\tau_{\max}$ in contrast to previous works, and remove this dependency completely with bounded gradients assumption. Our rate matches the best-known result \citep{koloskova2022sharper} in homogeneous regime (i.e., $\zeta^2=0$).

\paragraph{Pure Asynchronous SGD with waiting.} In contrast to the previous case, the server waits for $b \ge 1$ workers to finish their jobs, and then it assigns to them (i.e., $k_{t+1} = i_t$ as before) new jobs for the same updated model with $\alpha_{t+1} = \lfloor \nicefrac{t+1}{b} \rfloor b$. In comparison with previous method, waiting for $b$ leads to faster convergence while keeping the same dependency $\tau_{\max}$.

\paragraph{Random Asynchronous SGD \cite{koloskova2022sharper}.} In this version, a new job is assigned to the worker $k_{t+1}\sim\textrm{Uni}[1,2,\dots,n]$ chosen independently and uniformly at random among all workers for the latest model, i.e., $\alpha_{t+1} = t+1$. Thus, some workers might receive new jobs without completing the current one. We obtain the same rates as in \citep{koloskova2022sharper} using more general theory which indicates the sharpness of our approach.

\paragraph{Random Asynchronous SGD with waiting \cite{nguyen2022fedbuff}.} This method is a special case of the FedBuff algorithm \cite{nguyen2022fedbuff} with $Q=1$ local steps. Besides, it can be seen as a combination of the previous two special cases, namely asynchronous SGD with waiting and random assignments. The server waits for the first $b$ fastest workers and then assigns new jobs with the same model $\alpha_{t+1} = \lfloor \nicefrac{t+1}{b} \rfloor b$ to $b$ randomly chosen workers $k_{t+1}\sim\textrm{Uni}[1,\dots,n]$. Unlike \citep{nguyen2022fedbuff}, we derive convergence under realistic assumptions. Moreover, we derive $\tau_{\max}$-free rate and show the benefit from waiting for few workers as the rate improves with $b.$

\paragraph{Shuffled Asynchronous SGD {\bf [NEW]}.} In this case, we assume that all participating workers $[n]$ are active in the training. Similar to random asynchronous SGD described above, new jobs are always for the latest model, i.e., $\alpha_{t+1}=t+1$. However, new jobs are not assigned to workers independently but rather based on a random permutation of workers that can be re-sampled after each cycle or sampled once and reused throughout the training. More specifically, if $\chi$ is a random permutation of indices $[n]$, then $k_{t+1} = \chi(j)$, where $j-1 = t \;(\textrm{mod } n)$ is the remainder of $t$ when divided by $n$. In Section~\ref{sec:shuffled_acynchronous_SGD} we demonstrate that new method outperforms its random counterpart in highly heterogeneous regime $\zeta \ge \sqrt{n\varepsilon }$ which typically holds in Federated Learning.

\paragraph{Mini-batch SGD.} This is the standard variety of SGD method, and a popular method for a single-node setting. We show that the update rule of mini-batch SGD can be modified to suit update rule \eqref{eq:real_iterates_update} and derive standard convergence rate. In particular, if we treat each data point as a separate client then mini-batch SGD can be viewed as random asynchronous SGD with waiting where initial number of jobs assigned by the server is $b$.

\paragraph{SGD with Random Reshuffling \cite{nedic2001incremental}.} SGD with random reshuffling  is one the most used and sometimes a default algorithm in practice to train neural networks. At the beginning of each epoch, the dataset is randomly shuffled, and gradients are computed following that random order. Similar to the analogy described for mini-batch SGD, we can view SGD with random reshuffling as a special case of shuffled asynchronous SGD. Our rate matches the best known guarantees showing the tightness of our approach.

\section{Convergence Theory}\label{sec:conv_analysis}

\subsection{Theoretical assumptions}\label{sec:assumptions}
Below we list the assumptions we use in the theoretical analysis. All of them are standard in the  distributed non-convex optimization literature.

\begin{assumption}\label{asmp:smoothness} Local functions $f_i$ are differentiable and $L$-smooth for some positive constant $L$, namely,
\begin{equation}\label{eq:smoothness}
    \|\nabla f_i(x) - \nabla f_i(y)\| \le L\|x-y\|, \quad \forall x, y \in \R^d.
\end{equation}
\end{assumption}

For some of our results, we also need a bounded variance assumption on stochastic gradients. 

\begin{assumption}\label{asmp:bound_var} Stochastic gradients $\nabla f_i(x, \xi)$ are unbiased estimators of $\nabla f_i(x)$, i.e.,
\begin{equation}\label{eq:unbias}
    \squeeze
    \EE_{\xi\sim\cD_i}\left[\nabla f_i(x,\xi)\right] = \nabla f_i(x),  \quad \forall x \in \R^d,
\end{equation}
and have bounded variance $\sigma^2\ge0$, namely,
\begin{equation}\label{eq:bound_var}
    \squeeze
    \EE_{\xi\sim \cD_i}\left[\|\nabla f_i(x, \xi) - \nabla f_i(x)\|^2\right] \le \sigma^2, \quad \forall x \in \R^d.
\end{equation}

\end{assumption}
This is a typical assumption in the literature, and it holds, for example, when we have access to the gradients with Gaussian noise. We denote a realization of $\nabla f_i(x, \xi)$ by $g_i(x)$ for shortness. 

Next, we also assume that the bounded function heterogeneity assumption holds since in general case it is not possible to derive any convergence guarantees for asynchronous algorithms. 

\begin{assumption}\label{asmp:grad_sim} 
    Local gradients $\nabla f_i(x)$ satisfy bounded heterogeneity condition for some $\zeta^2\ge0$, i.e.,
    \begin{equation}\label{eq:grad_sim}
        \|\nabla f_i(x) - \nabla f(x)\|^2 \le \zeta^2, \quad \forall x\in \R^d.
    \end{equation}
\end{assumption}

Several results require the Lipschitzness of local loss functions. 

\begin{assumption}\label{asmp:bound_grad} Local functions $f_i(x)$  are $G$-Lipschitz, i.e. for some positive constant $G$ they satisfy
\begin{equation}
    |f_i(x) - f_i(y)| \le G\|x-y\| \quad \forall x, y \in \R^d.
\end{equation}
\end{assumption}
Note that, in the case of differentiable $f_i$, this assumption implies that local gradients are bounded, i.e., for all $x\in \R^d$ $\|\nabla f_i(x)\| \le G$ \cite{bubeck2015convex}. In contrast to \cite{mishchenko2022asynchronous}, we do not assume the boundedness of stochastic gradients. 
Practical implementations frequently resort to using clipping in the presence of Byzantine workers or stragglers. The clipping automatically bounds the norms of applied gradient, forcing the constant $G^2$ to be small.

\subsection{Notation}

Generally, we do not make any assumptions on the delays~--- gradients might be received in any random or deterministic order. We assume that the server can receive and assign jobs with delays, namely, $\pi_t \eqdef t - \tau_t$ and $\alpha_t \eqdef t - \wtau_t,$ where $\tau_t, \wtau_t \ge 0$ are corresponding delays.\footnote{If $\tau_t \equiv \wtau_t\equiv 0$, then there is no delay.} The order might be natural as in the case of pure asynchronous SGD or be pre-set as for mini-batch SGD; see Sections~\ref{sec:special_cases_real_theorem1} and \ref{sec:special_cases_real_theorem4} for more examples of how $\{i_t\}_{t=0}^{T-1}$ might look like. Besides, we introduce the notion of maximum and average delays similar to \cite{koloskova2022sharper}.

\begin{definition}\label{def:max_avg_delays}
    Let $\{\tau_t\}_{t=0}^{T-1}$ be the delays of all applied gradients. The average and maximum delays are defined as follows
    \begin{equation}
        \squeeze \tau_{\avg} \eqdef \squeeze\frac{1}{|\cA_{T+1}|}\left(\sum\limits_{t=0}^{T-1}\tau_t + \sum\limits_{(i,j) \in \cA_{T+1}\setminus \cR_T} T-j\right), \;
        \squeeze \tau_{\max} \eqdef \squeeze\max\left\{\max\limits_{0\le t < T} \tau_t, \max\limits_{(i,j) \in \cA_{T+1}\setminus \cR_T} T-j\right\} \label{eq:max_avg_delays}.
    \end{equation}
\end{definition}

Quantities $\wtau_{\avg}$ and $\wtau_{\max}$ are defined analogously with respect to the delays $\{\wtau_t\}_{t=0}^{T-1}$ related to  assigning process. Moreover, in the analysis, we use the maximum number of active jobs or concurrency $\tau_C.$ This quantity indicates the maximum number of jobs already assigned, but not yet completed (i.e., active jobs) during the optimization process.

\begin{definition}\label{def:max_active_jobs}
    The maximum number of active jobs or concurrency is defined as
    \begin{equation}\label{eq:max_active_jobs}
    \squeeze \tau_C \eqdef \max\limits_{0 \le t \le T} \left|\cA_{t+1}\setminus \cR_t\right|.
    \end{equation}
\end{definition}

To utilize available resources in a more efficient way, in practice, all workers are always busy, i.e., $\tau_C=n$. Nevertheless, it might happen that a fraction of all workers can be unavailable from time to time.

The received $\{i_t\}$ and assigned $\{k_t\}$ orders of functions define the convergence properties of Algorithm~\ref{alg:pseudocode}. Hence, we are interested in the correlation between functions within a certain correlation interval $\tau.$ The final rate depends on how much the functions within the correlation interval differ from the averaged gradient. To mathematically describe the aforesaid, we define the sequence correlation \cite{koloskova2023shuffle} below. 

\begin{definition}\label{def:variance_order}
    For any given correlation period $\tau \ge 1$, we successively split the set of received gradient indices $\{i_t\}_{t=0}^{T-1}$ into $\left\lceil\frac{T}{\tau} \right\rceil$ chunks of size $\tau$. Then, the sequence correlation of received jobs within $k$-th period is defined as
    \begin{equation}\label{eq:received_seq_cor_variance}
        \squeeze \sigma_{k, \tau}^2 \eqdef \max\limits_{0 \le j < \tau}\EE\left[\left\|\sum_{t=k\tau}^{\min\{k\tau+j, T-1\}}\nabla f_i(x_{k\tau}) - \nabla f(x_{k\tau})\right\|^2\right].
    \end{equation}
\end{definition}

Next, the magnitude of delays affects the resulting convergence rate. We measure the effect of the delays by the quantity defined below. Note that it does not involve any correlation period since it is designed to track how $\pi_t$ impacts the rate.

\begin{definition}\label{def:variance_delay} For the sequence of received gradient indices $\{i_t\}_{t=0}^{T-1}$ the delay variance is defined as
\begin{equation*}
        \squeeze \nu^2 \eqdef \sum\limits_{t=0}^{T-1}\EE\left[\left\|\sum_{j=\pi_t}^{t-1} \nabla f_{i_j}(x_{\pi_j}) - \nabla f(x_{\pi_j}) \right\|^2\right]. 
    \end{equation*}
\end{definition}

\subsection{Analysis of Gradient Receiving Process}

In the first theorem, we analyze the gradient receiving process, i.e., how the order of received gradients $\{i_t\}_{t=0}^{T-1}$ influences the convergence. We would like to highlight that this ordering can not be controlled by the server; it is built naturally depending on the speeds of the workers. 

For the analysis, we additionally construct a sequence of virtual iterates \cite{mania2017virtual} with restarts following the approach created in \cite{koloskova2023shuffle}. The sequence is defined as $\wtilde{x}_0 = x_0$ and
\begin{equation}\label{eq:real_virtual_seq}
    \squeeze \wtilde{x}_{t+1} = \begin{cases}
        \wtilde{x}_t - \gamma\nabla f(x_{t}) & \text{if } t+1 \neq 0\mod \tau, \\
        x_{t+1} & \text{if } t+1 = 0 \mod \tau.
\end{cases}
\end{equation}
In contrast to the sequence $\{x_t\}_{t=0}^{T}$, virtual iterates $\{\wtilde{x}_t\}_{t=0}^{T-1}$ are updated using full gradient always evaluated for the last model $x_t.$   Moreover, we restart the virtual iterates once in $\tau$ iterations, so we can track the progress of real iterates within one correlation period. More particularly, we use $\tau = \Theta(\frac{1}{L\gamma})$, where $L$ and $\gamma$ are the smoothness constant and  the stepsize respectively. 

\begin{restatable}{theorem}{theoremthird}
\label{th:theorem3} Let Assumptions \ref{asmp:smoothness} and \ref{asmp:bound_var} hold. Let the stepsize $\gamma$ satisfy inequalities $6L\gamma \le 1$ and $20L\gamma\sqrt{\tau_{\max}\tau_C}\le 1$, the correlation period $\tau = \left\lfloor\frac{1}{20L\gamma}\right\rfloor$, and quantities $\{\sigma_{k,\tau}^2\}_{k=0}^{\lfloor T/\tau\rfloor}$ and $\nu^2$ are finite. Then
\begin{equation}\label{eq:theorem3}
    \squeeze \E{\|\nabla f(\hat{x}_T)\|^2} 
    \le \cO\left(
    \frac{F_0}{\gamma T}
    + L\gamma\sigma^2
    + L^2\gamma^2\Phi \right),
\end{equation}
where $F_0 \eqdef f(x_0) - f^*,\; \Phi \eqdef \frac{1}{\lfloor T/\tau \rfloor}\sum\limits_{k=0}^{\lfloor \frac{T}{\tau} \rfloor} \sigma^2_{k,\tau} + \frac{1}{T}\nu^2$, and $\hat{x}_T$ is chosen uniformly at random from $\{x_1,\dots,x_T\}$.
\end{restatable}

We give a detailed proof of the theorem in Appendix, Section~\ref{sec:proof_theorem3}. We observe that the rate consists of three terms. The first one, $\frac{F_0}{\gamma T}$, is a standard optimization term that always appears for vanilla GD, and is shown to be optimal. The second term, $L\gamma\sigma^2$, appears because of the stochastic nature of the gradients. The third term, $L^2\gamma^2\Phi$, the most intriguing one, consists of two parts~--- the first one represents the effect of function ordering while the second part shows the impact of delays on the convergence. 

Since the first and second terms are standard for gradient-based algorithms, one asynchronous algorithm differs from another one with a bound on $L^2\gamma^2\Phi.$ Intuitively, our goal is to create an algorithm for which this term is as small as possible in order to guarantee better convergence. This can be achieved by properly balancing the workers' contributions. Moreover, the stepsize $\gamma$ is decreased by $\tau_{\max}^{1/2}$ to mitigate the effect of the delays.

Note that all quantities $\{\sigma_{k,\tau}^2\}_{k=0}^{\lfloor \frac{T}{\tau}\rfloor}$ and $\nu^2$ depend on $\tau$, and consequently, on the stepsize $\gamma$ as well. Hence, for the general case, the inequality~\eqref{eq:theorem3} is implicit. However, for some special cases, we are able to compute all quantities and derive convergence guarantees.

\subsection{Analysis of Gradient Assigning Process}

Now we switch to the analysis of the order $\{k_t\}_{t=1}^{T}$ which the server uses to decide the order of assigning new jobs. Recall that the server is able to control this order. That is why we can use various randomization procedures to balance the workers.

The analysis in this case is based on the virtual iterates $\{y_t\}_{t=0}^{T}$ that follow the assigning process. Formally, we define $y_0 = x_0$ and
\begin{equation}
    \squeeze y_{t+1} = \squeeze y_t - \gamma\sum_{(i,j) \in \cA_{t+1} \setminus \cA_t} g_i(x_j)
    \overset{t>0}{=} \squeeze y_t - \gamma g_{k_t}(x_{\alpha_t}) \label{eq:virtual_iterates_update}.
\end{equation}
Here we highlight that the server may decide to send a job at outdated point $x_{\alpha_t}$ with bounded by $\wtau_{\max}$ delay.\footnote{In other words, $\alpha_t = t - \wtau_t$ where $\wtau_t \ge 0.$} This enables us to investigate methods where several workers compute gradients at the same point, e.g., mini-batch SGD. Hence, it brings even more flexibility to our framework. The real and virtual iterates can not be arbitrarily far away from each other. The next Lemma reveals the connection between them.

\begin{lemma}
    Let real $\{x_t\}_{t=0}^T$ and virtual $\{y_t\}_{t=0}^T$ iterates be defined in \eqref{eq:real_iterates_update} and \eqref{eq:virtual_iterates_update} respectively. Then
    \begin{equation*}
        \squeeze x_t - y_t = \gamma\sum_{(i,j) \in \cA_t\setminus \cR_t} g_i(x_j).
    \end{equation*}
\end{lemma}

Based on the sequence $\{y_t\}_{t=0}^T$ we construct the restarting sequence $\{\wtilde{y}_t\}_{t=0}^T$ similarly to \eqref{eq:real_virtual_seq}

\begin{equation}\label{eq:virtual_virtual_seq}
    \wtilde{y}_1 = y_1, \quad
    \wtilde{y}_{t+1} =
    \begin{cases}
        \wtilde{y}_t - \gamma\nabla f(x_{t}) & \text{ if } t  \neq 0\mod \tau, \\
        y_{t + 1} & \text{ if } t = 0\mod \tau.
    \end{cases}\notag
\end{equation}
Analogously, we can define the sequence correlation and the delay variance for the sequence of assigned gradient indices $\{k_t\}_{t=1}^{T}$. We denote them by $\{\wtilde{\sigma}_{k, \tau}^2\}_{t=0}^{\lfloor\frac{T}{\tau}\rfloor}$ and $\wtilde{\nu}^2$ respectively. Based on this, we present our second theorem.

\begin{restatable}{theorem}{theoremfourth}
\label{th:theorem4} Let Assumptions \ref{asmp:smoothness}, \ref{asmp:bound_var}, and \ref{asmp:bound_grad} hold. Let the stepsize $\gamma$ satisfies inequalities $6L\gamma \le 1$ and $30L\gamma\max\{\wtau_{\max},\tau_C\}\le 1$, the correlation period $\tau = \left\lfloor\frac{1}{30L\gamma}\right\rfloor$, quantities $\{\wtilde\sigma_{k,\tau}^2\}_{k=0}^{\lfloor T/\tau\rfloor}$ and $\wtilde\nu^2$ are finite. Then
\begin{equation}\label{eq:theorem4}
    \squeeze\E{\|\nabla f(\hat{x}_{T})\|^2} 
    \le \cO\left(
    \frac{F_1}{\gamma T} 
    + L\gamma\sigma^2 +L^2\gamma^2\wtilde{\Phi}
    \squeeze +\; L^2\gamma^2(\tau_C-1)^2G^2
    \right),
\end{equation}
where $F_1 \eqdef f(y_1) - f^*,\; \wtilde{\Phi} \eqdef \frac{1}{\lfloor\nicefrac{T}{\tau}\rfloor}\sum\limits_{k=0}^{\lfloor \frac{T}{\tau}\rfloor} \wtilde{\sigma}_{k,\tau}^2 + \frac{1}{T}\wtilde{\nu}^2$, and $\hat{x}_T$ is chosen uniformly at random from $\{x_1,\dots,x_T\}$.
\end{restatable}
The proof of Theorem~\ref{th:theorem4} can be found in  Section~\ref{sec:proof_theorem4}. It gives similar convergence behaviour as Theorem~\ref{th:theorem3}, but for the assigning sequence $\{k_t\}_{t=1}^{T}$. The difference is in the last $L^2\gamma^2(\tau_C-1)^2G^2$ term which is not present in \eqref{eq:theorem3}, and it appears because of the bound between real and virtual iterates.

We emphasize that if $\tau_C=1$, i.e., there is only one active job at each iteration, then the fourth term is \emph{zero}. For example, this holds for single-node algorithms such as SGD with random reshuffling or shuffle once. Therefore, we recover the same rates for those methods. Besides, we highlight the fact that the rate \eqref{eq:theorem4} admits the absence of the dependency on $\tau_{\max}$ in particular special cases; the delays affect the rate through averaged delays only. 

Moreover, the term $\wtilde{\Phi}$ depends on sequences $\{k_t\}$ and $\{\alpha_t\}$ which are fully under the control of the server. Therefore, the server may change its job assignment strategy if it observes that the current one leads to poor performance. In Section~\ref{sec:special_cases_real_theorem4} we demonstrate precisely which assignment strategies can be used.

\section{Experiments}

\begin{figure*}[t]
\centering
        \begin{tabular}{cccc}
            \hspace{3mm} {\tiny (a) fixed}  &
            \hspace{3mm}{\tiny (b) normal}&
            \hspace{3mm}{\tiny (c) poisson} &
            \hspace{3mm}{\tiny (d) uniform}
            \\
            \includegraphics[width=0.24\linewidth]{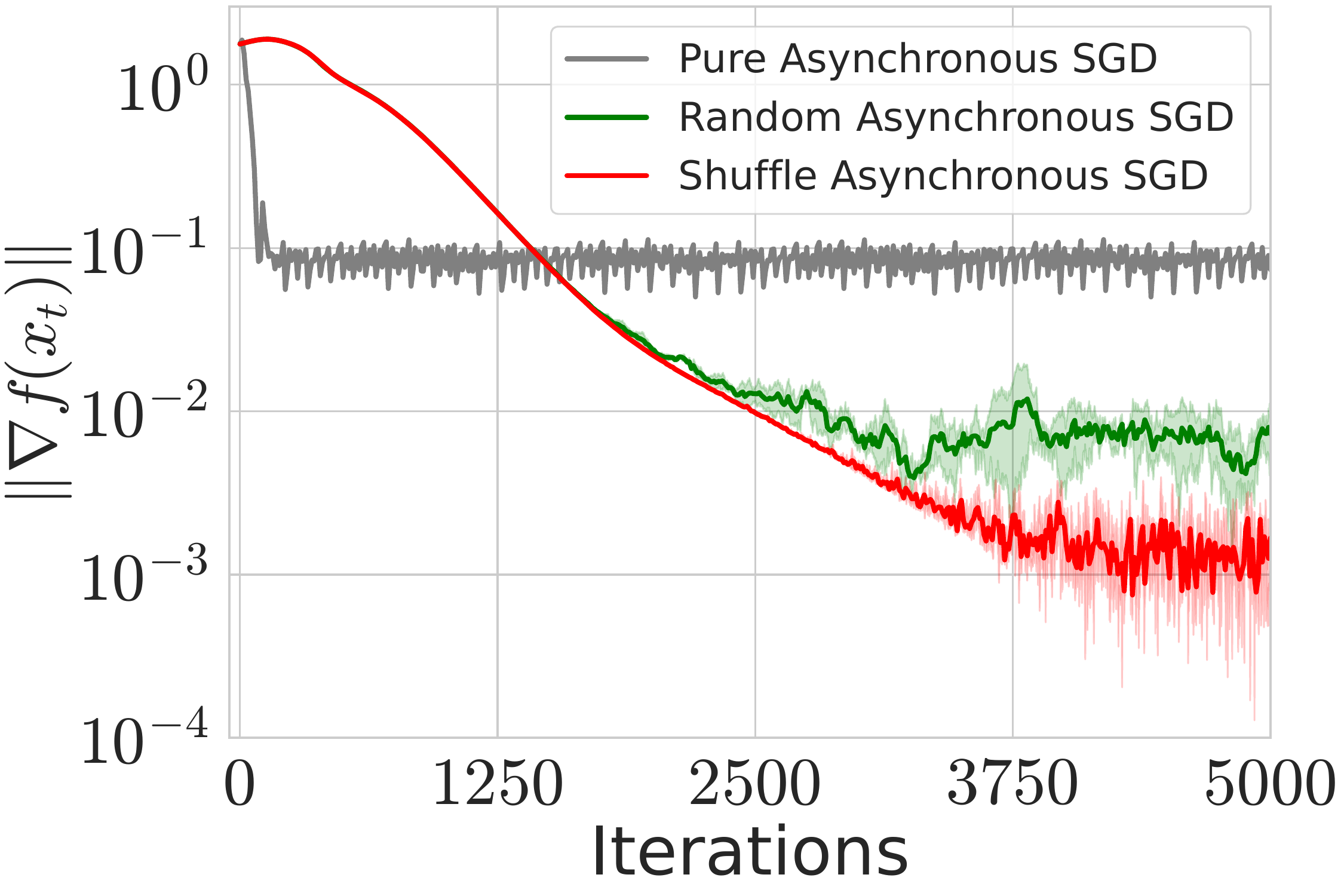} & \hspace{-5mm}
            \includegraphics[width=0.24\linewidth]{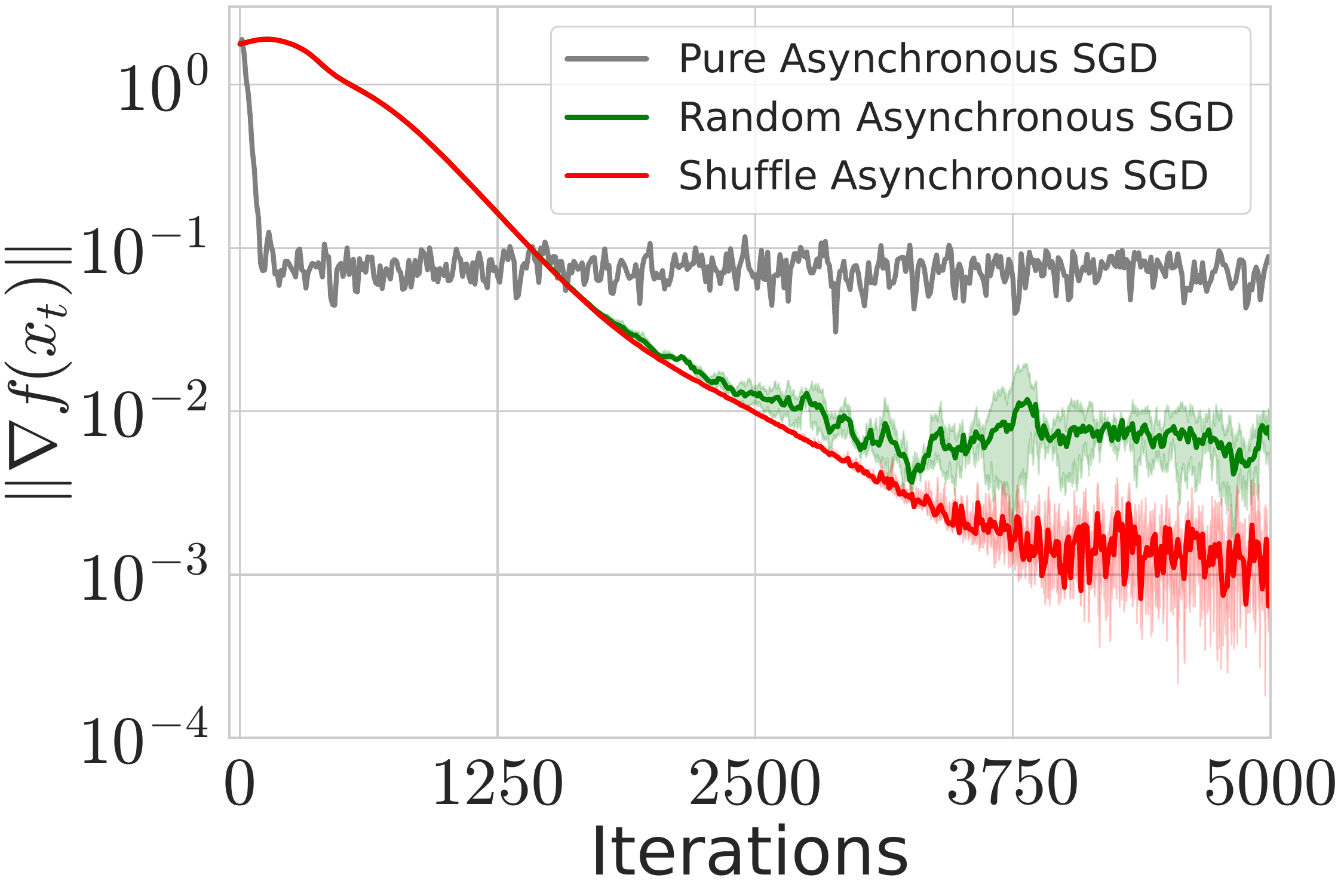} & \hspace{-5mm}
            \includegraphics[width=0.24\linewidth]
            {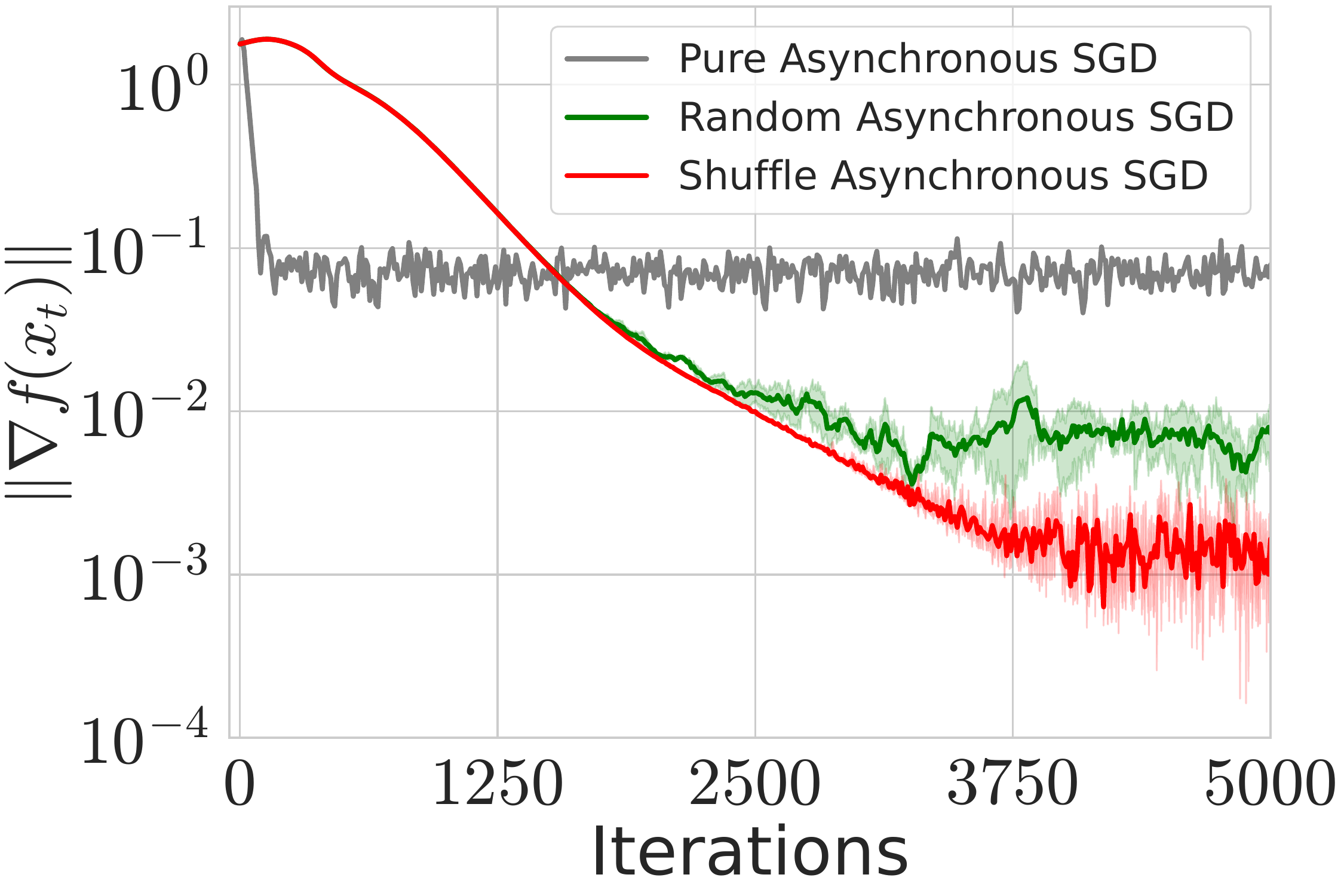} & \hspace{-5mm}
            \includegraphics[width=0.24\linewidth]{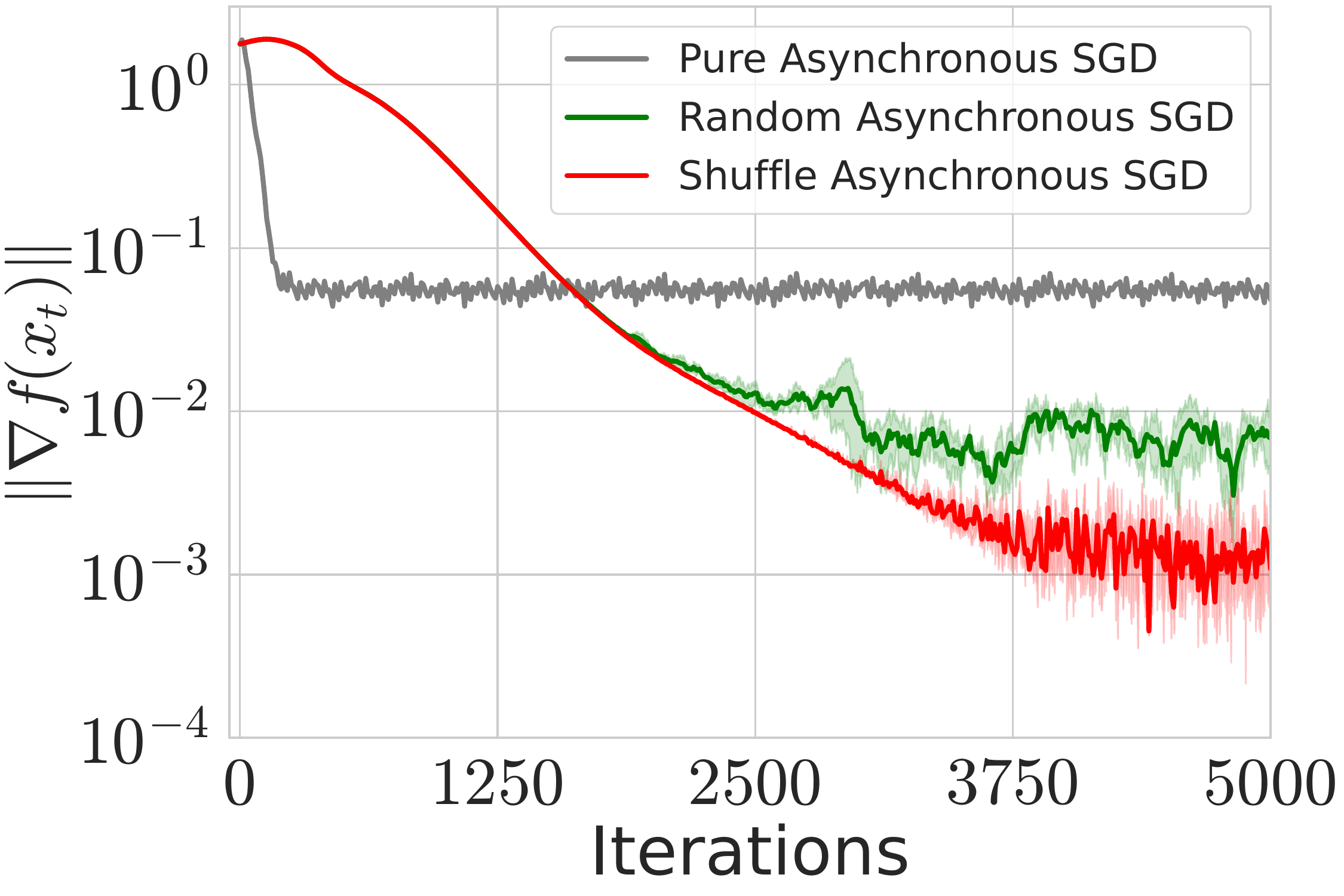} \\
            \hspace{3mm} {\tiny (a) fixed}  &
            \hspace{3mm}{\tiny (b) normal}&
            \hspace{3mm}{\tiny (c) poisson} &
            \hspace{3mm}{\tiny (d) uniform}
            \\
            \includegraphics[width=0.24\linewidth]{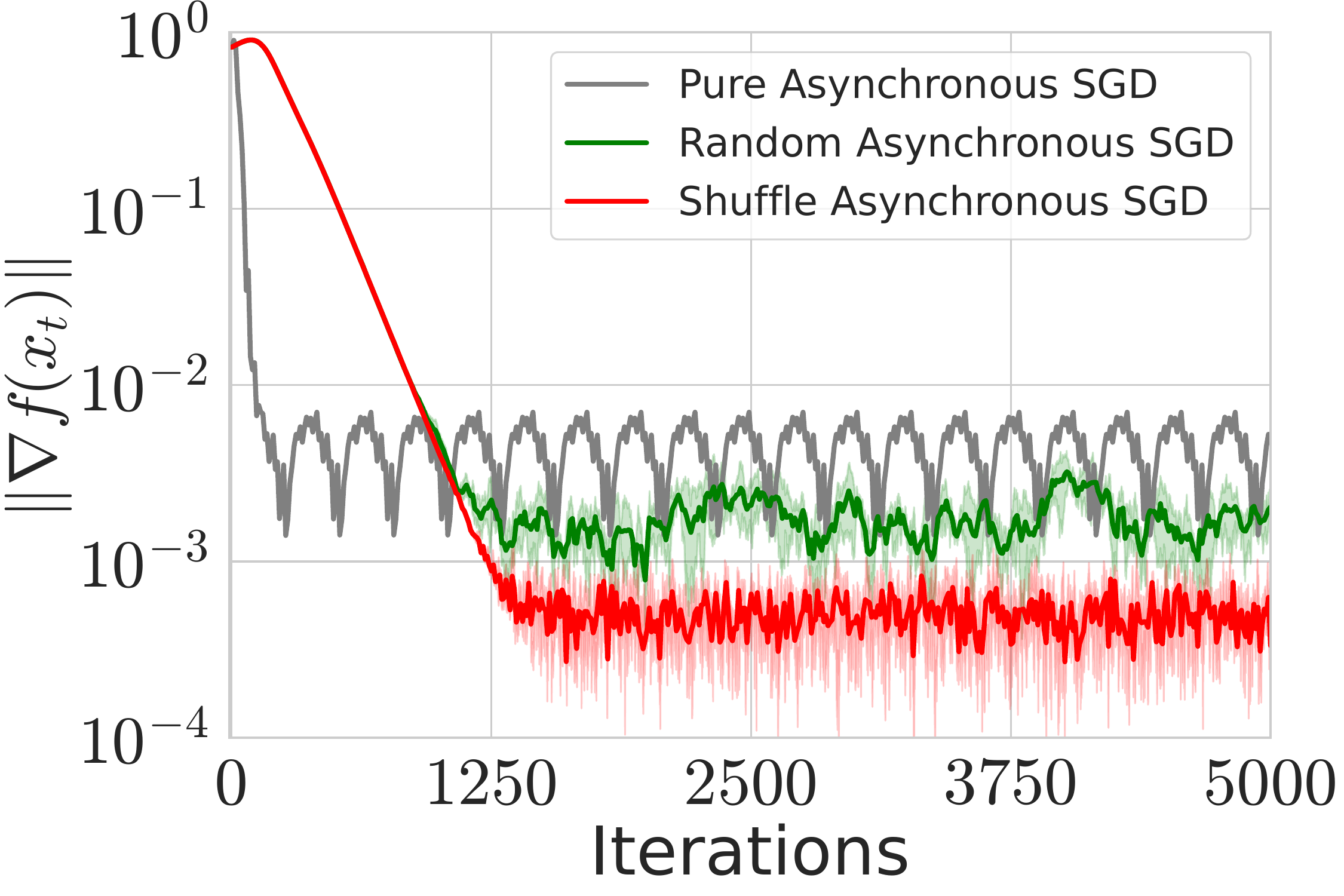} & \hspace{-5mm}
            \includegraphics[width=0.24\linewidth]{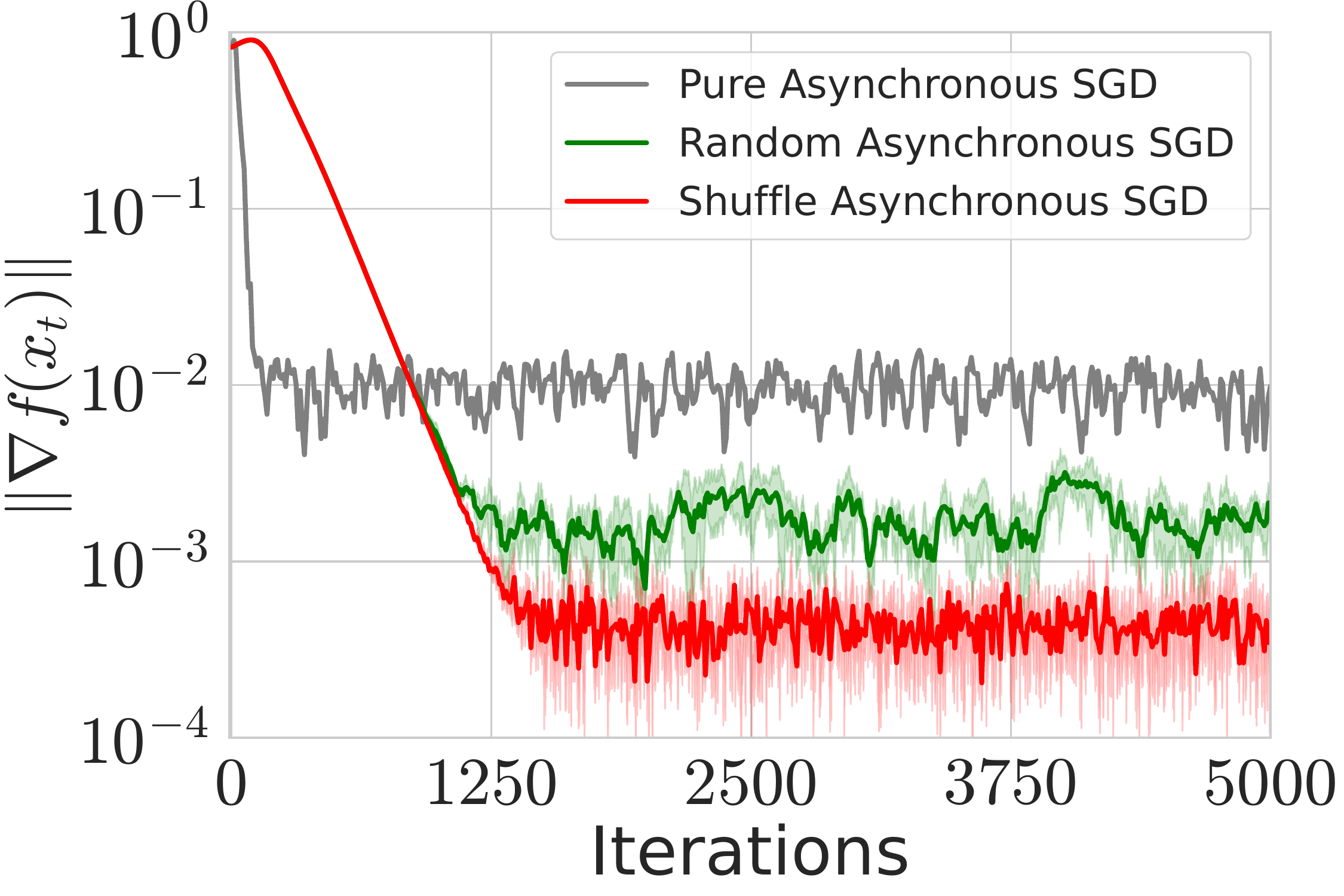} & \hspace{-5mm}
            \includegraphics[width=0.24\linewidth]
            {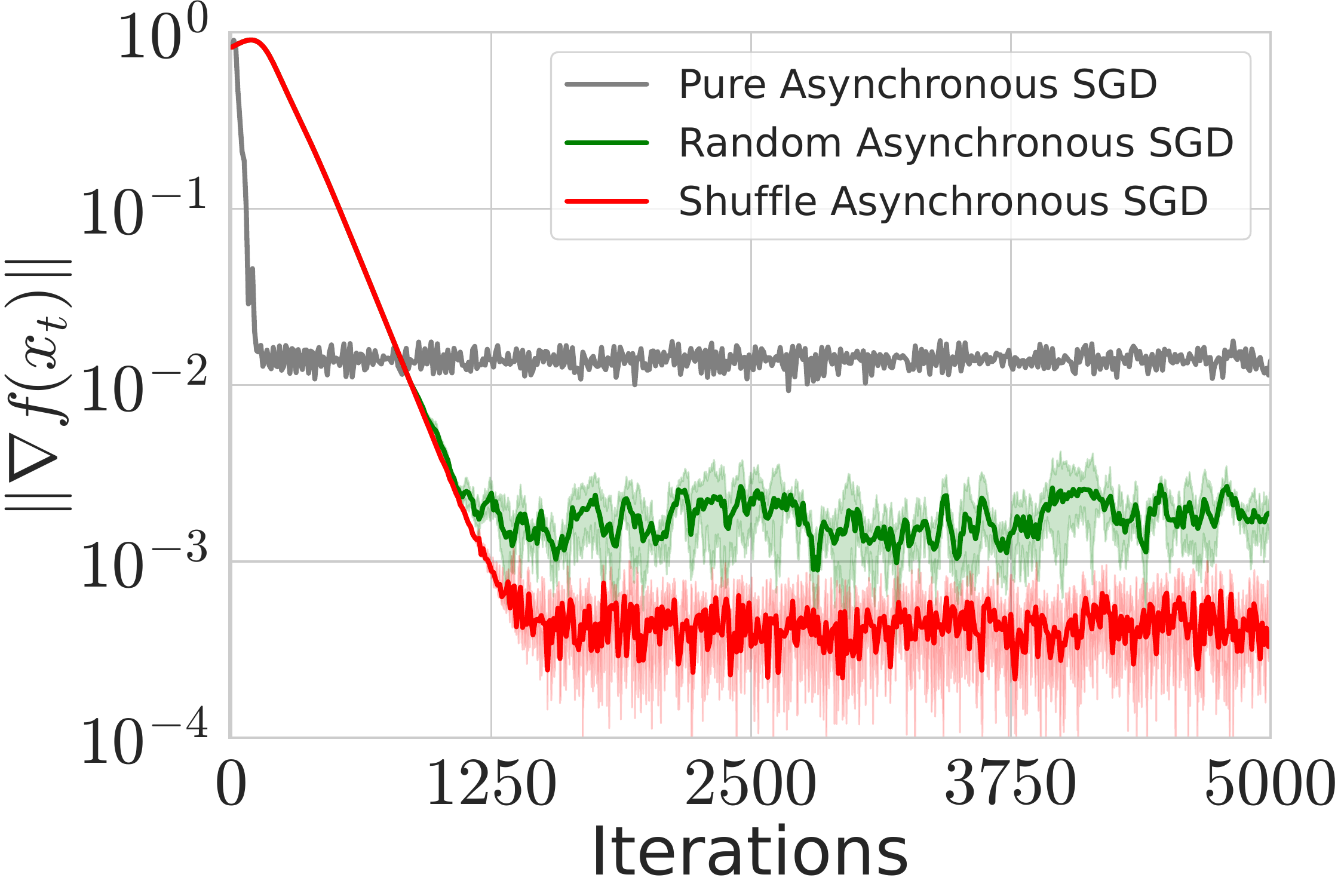} & \hspace{-5mm}
            \includegraphics[width=0.24\linewidth]{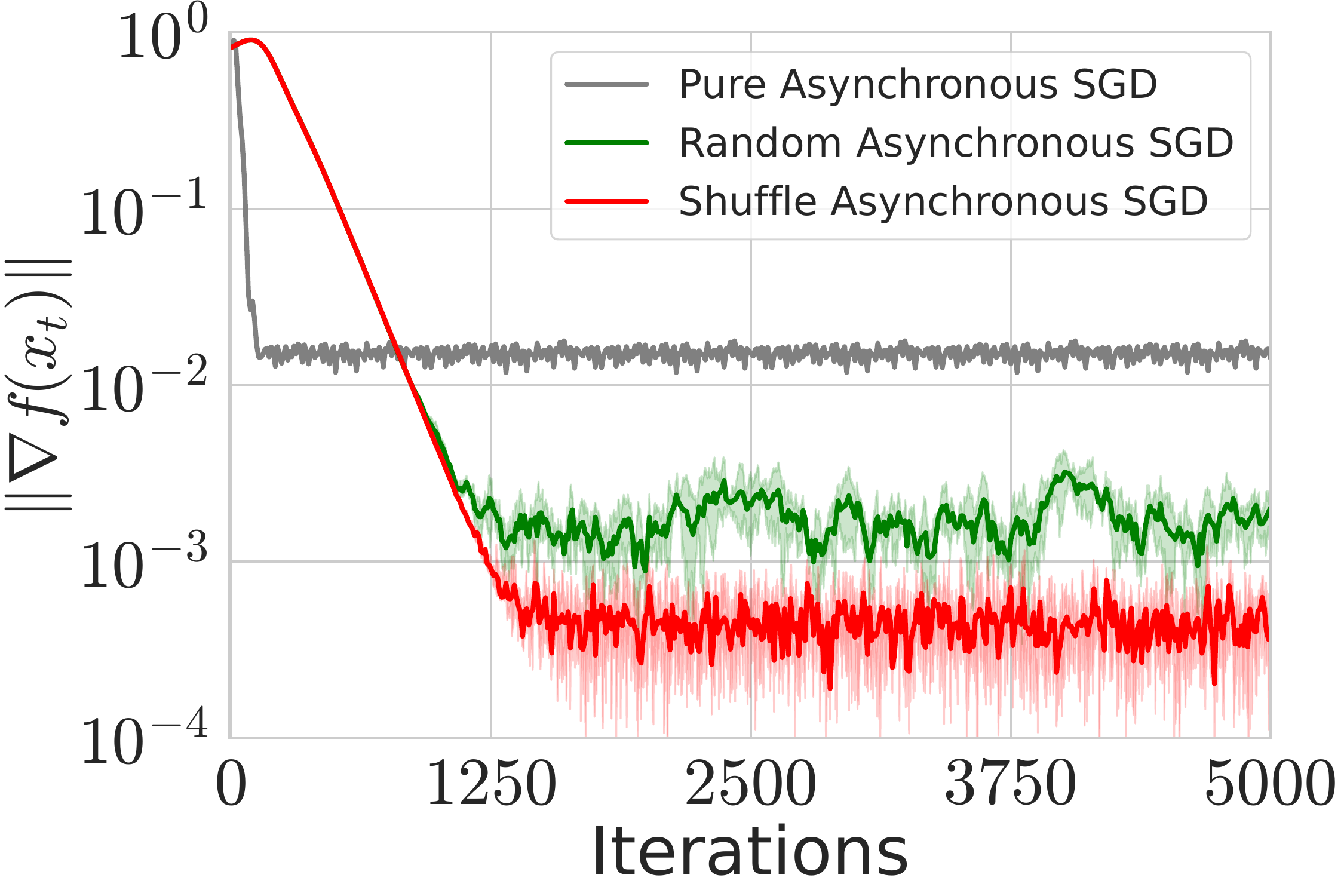} \\
        \end{tabular}             
    \caption{Comparison of pure, random, and shuffled asynchronous SGD with tuned stepsizes and full gradient computation on {\tt w7a} (first row) and {\tt phishing} (second row) datasets with various delay patterns. Here $n=10$, $\lambda=0.1$, $d=300,$ $m=2505$ for {\tt w7a} dataset and $n=10, \lambda=0.1$, $d=68,$ $m=1105$ for {\tt phishing} dataset.}
    \label{fig:w7a}
    \vspace{-10pt}
\end{figure*}

\begin{figure*}[t]
\centering
        \begin{tabular}{ccc}
            \qquad {\tiny fixed, $\text{Syn}(0.5,0.5)$}  &
            {\tiny fixed, $\text{Syn}(1,1)$}&
            {\tiny fixed, $\text{Syn}(1.5,1.5)$}\\
            \includegraphics[width=0.3\linewidth]{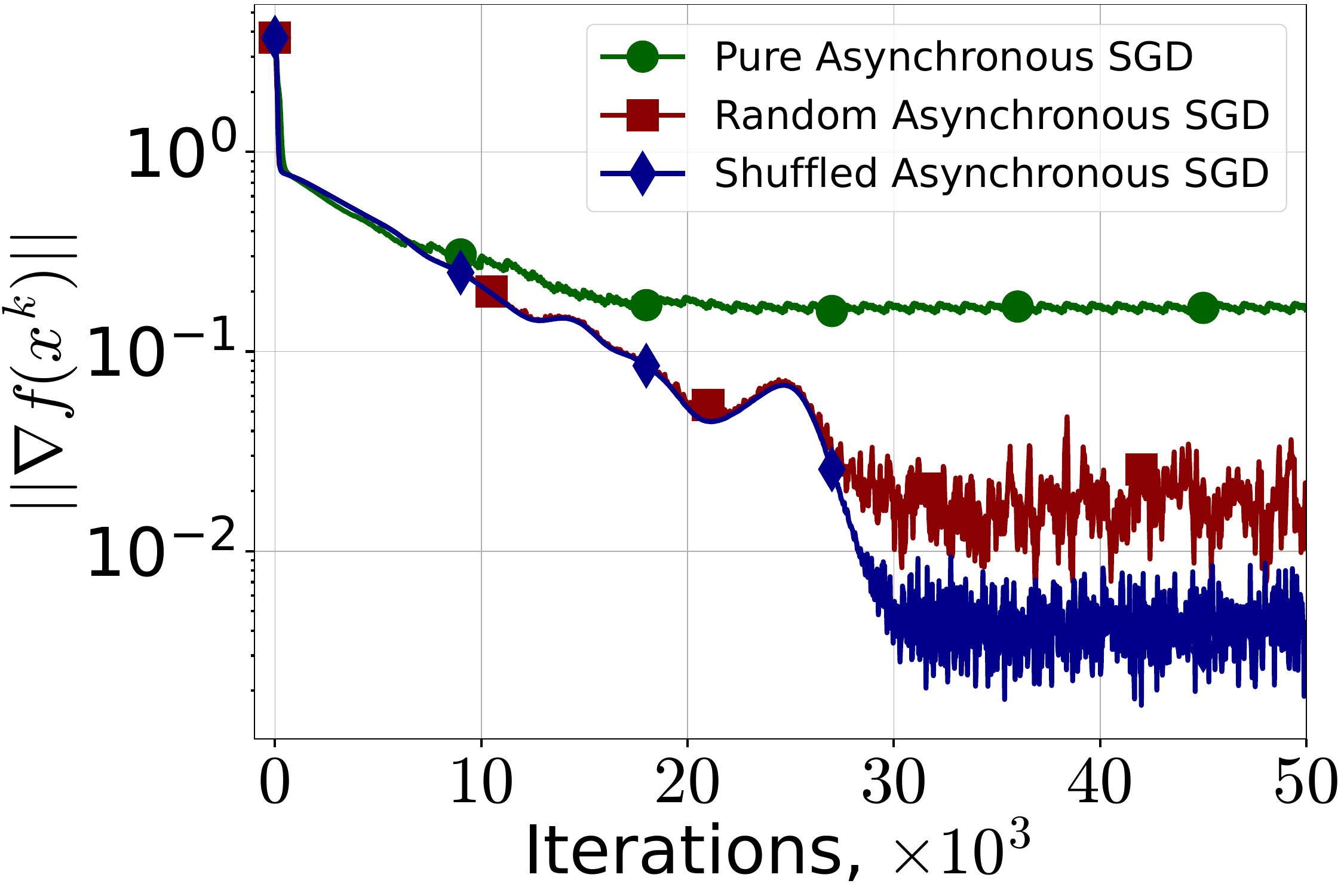} & \hspace{-5mm}
            \includegraphics[width=0.3\linewidth]{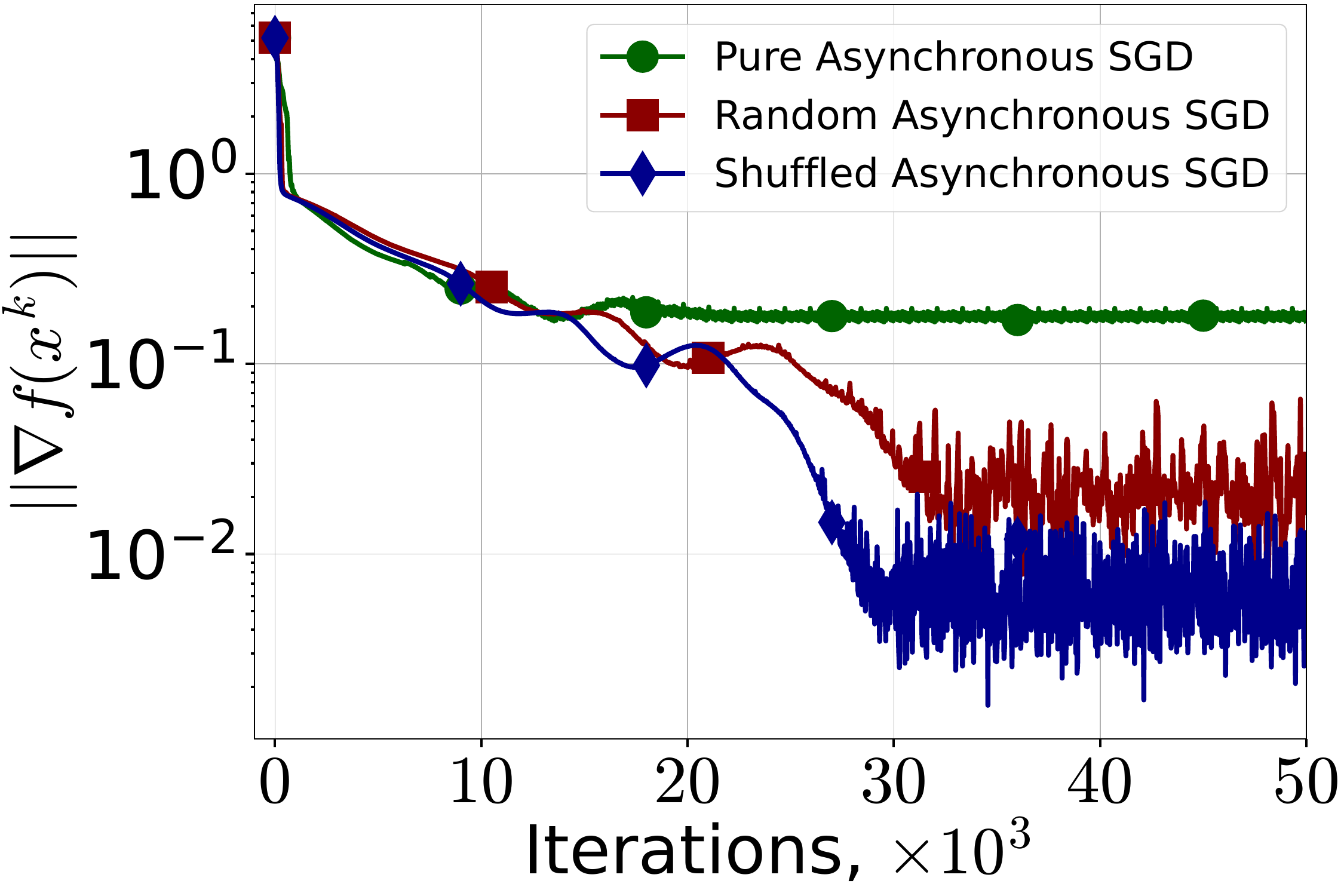} & \hspace{-5mm}
            \includegraphics[width=0.3\linewidth]{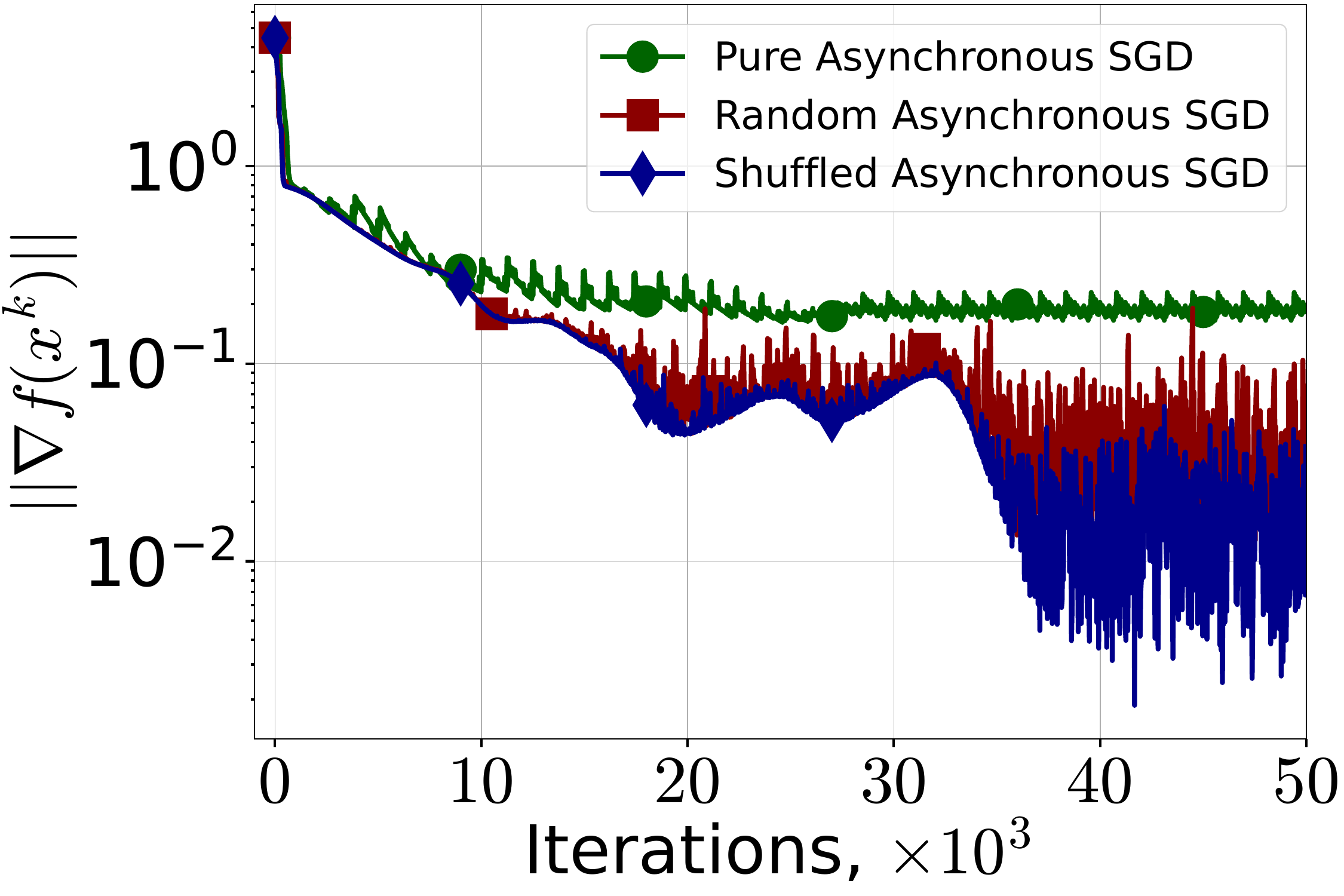}\\
            \qquad {\tiny $\text{Syn}(0.5,0.5)$}  &
            \hspace{5mm} {\tiny $\text{Syn}(1,1)$} &
            \hspace{5mm} {\tiny $\text{Syn}(1.5,1.5)$}\\
            \includegraphics[width=0.3\linewidth]{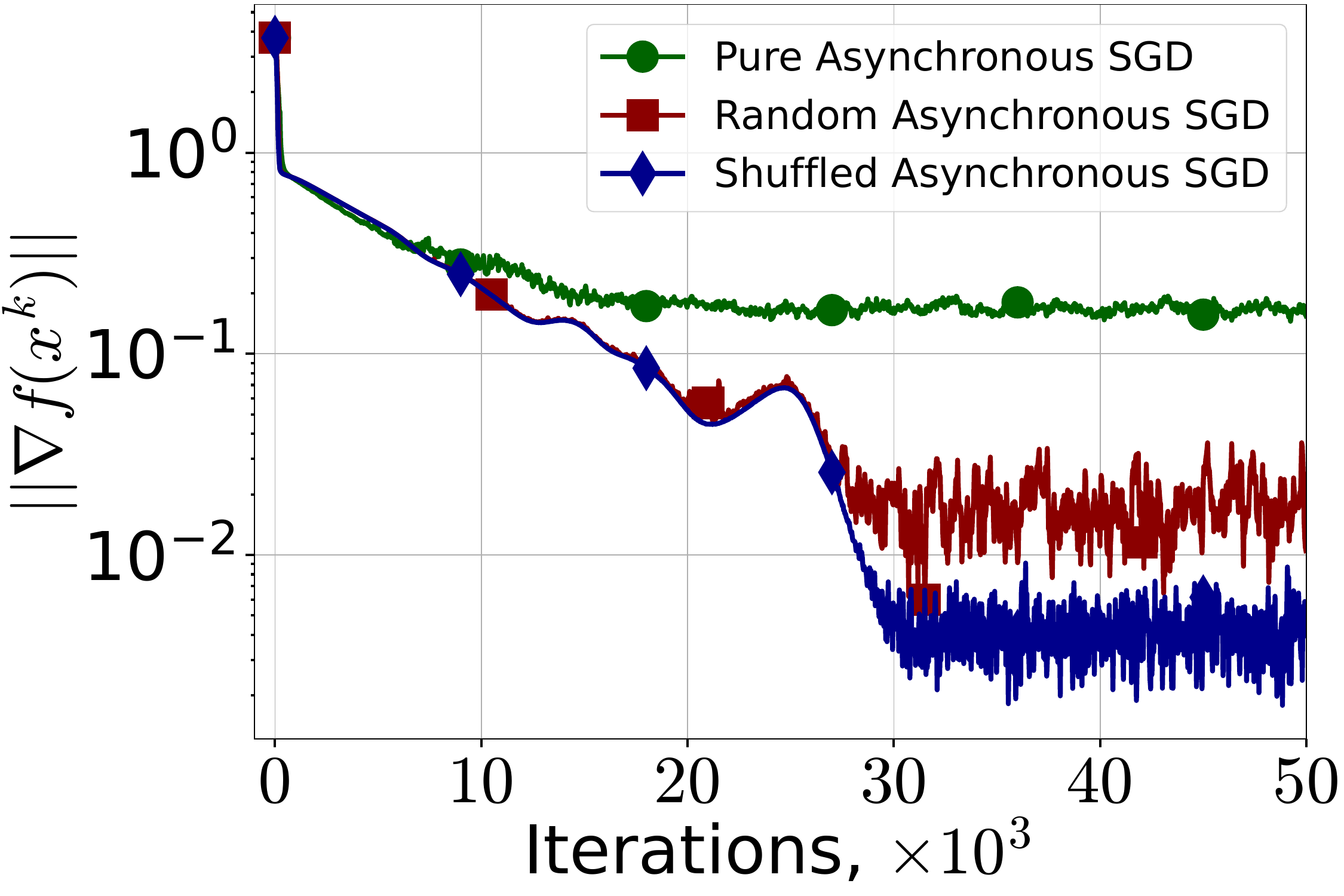} & 
            \includegraphics[width=0.3\linewidth]{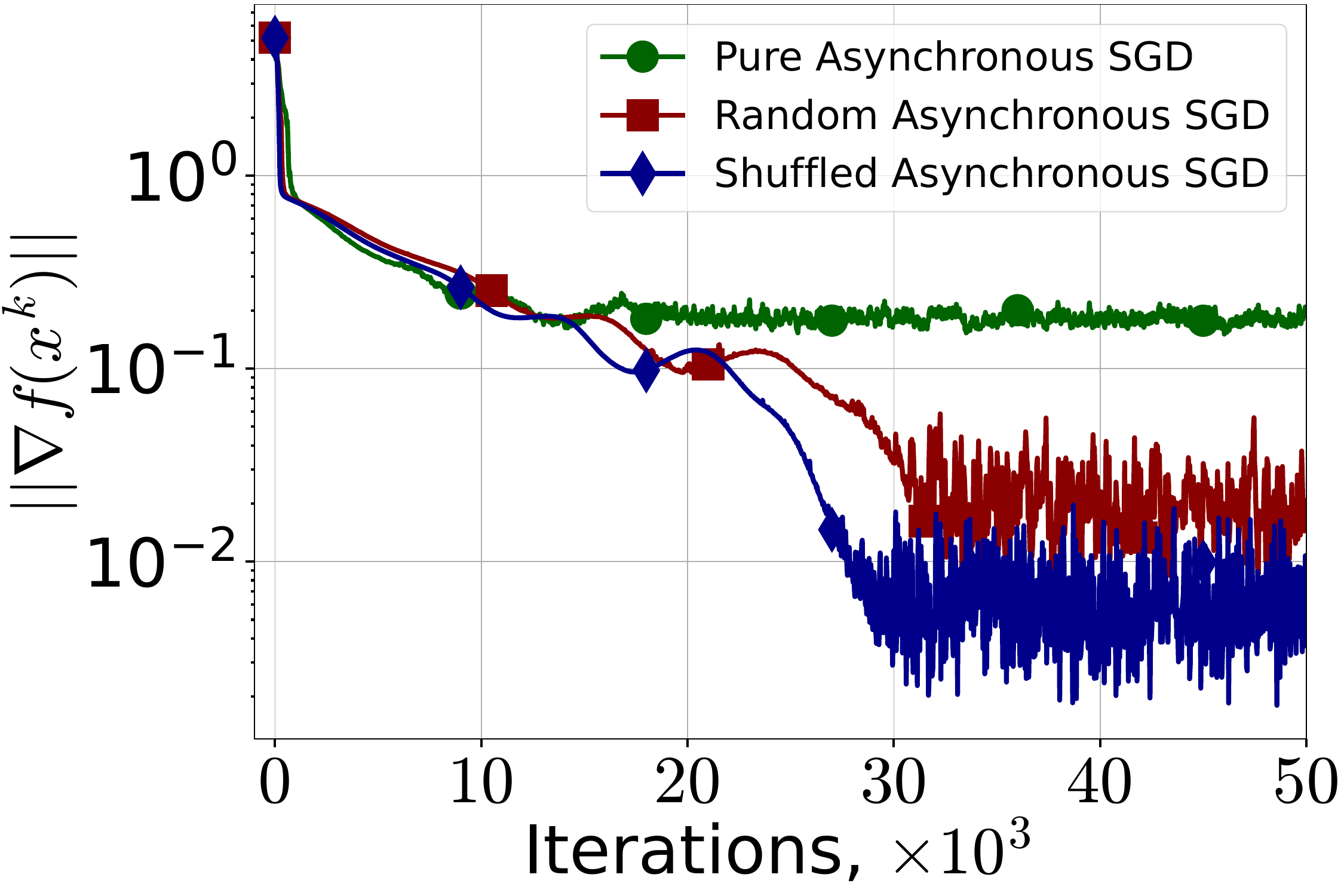} & 
            \includegraphics[width=0.3\linewidth]{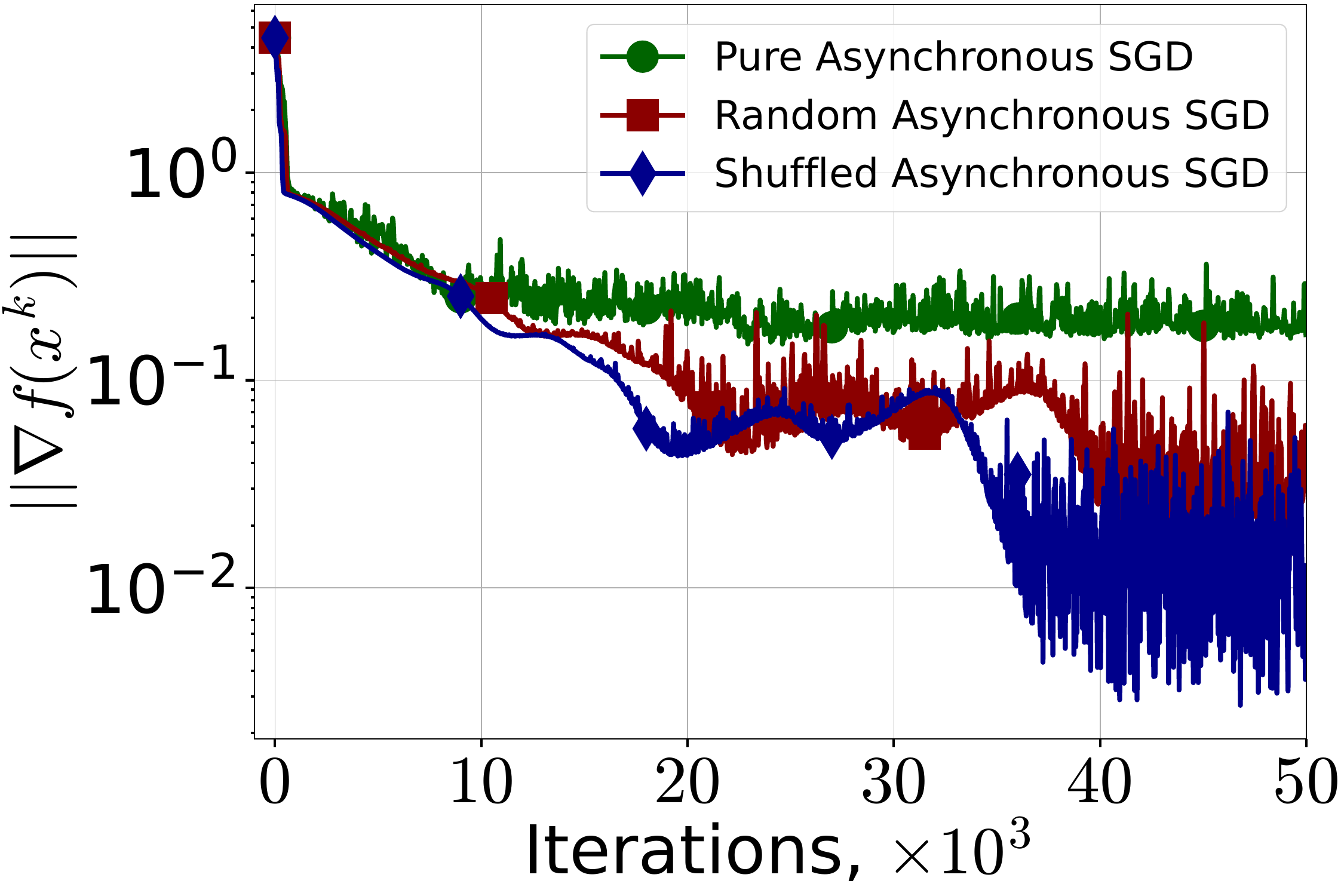}\\
            \hspace{10mm} {\tiny normal, $\text{Syn}(0.5,0.5)$}  &
            \hspace{5mm}{\tiny normal, $\text{Syn}(1,1)$}&
            \hspace{5mm}{\tiny normal, $\text{Syn}(1.5,1.5)$}\\
            \includegraphics[width=0.3\linewidth]{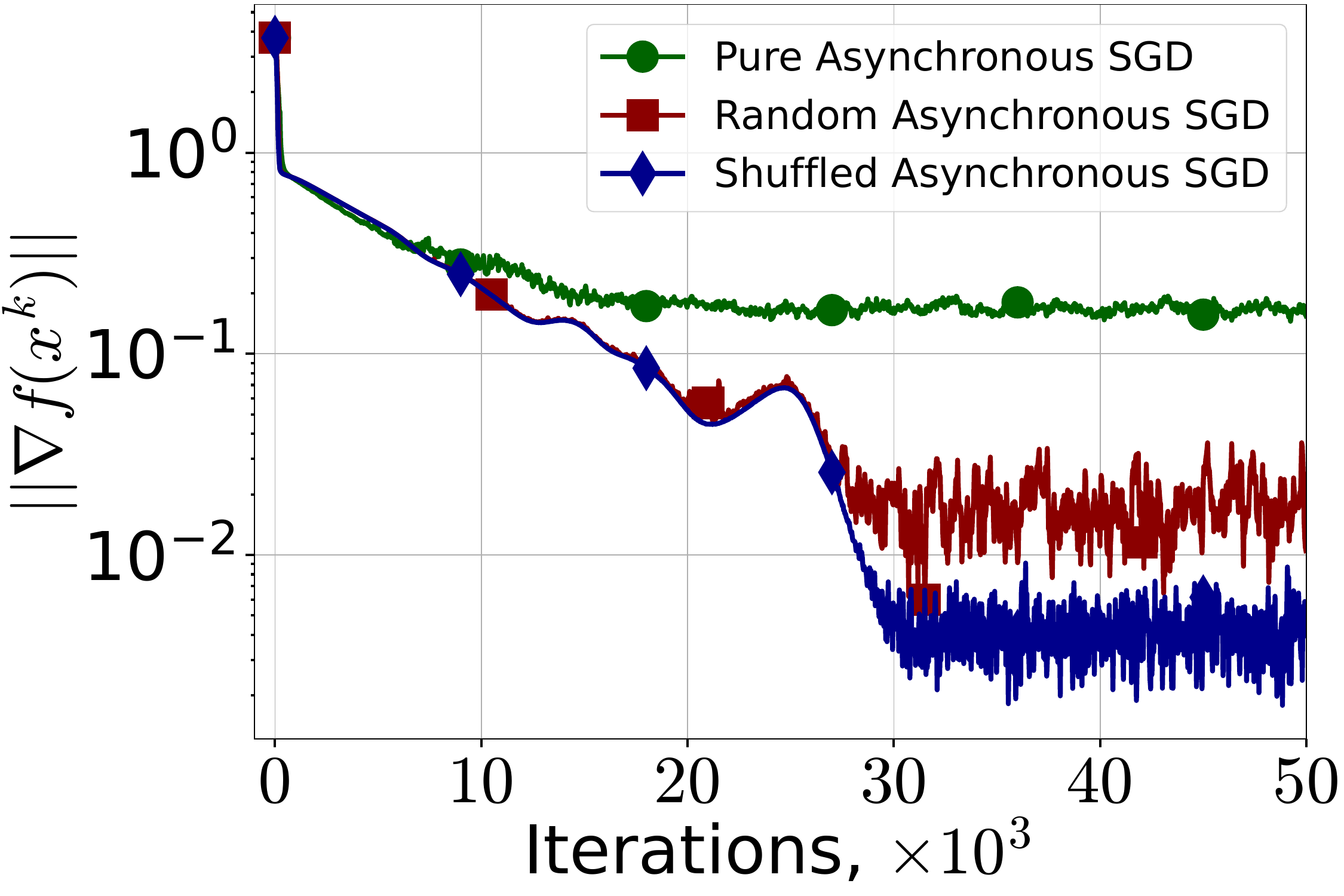} & \hspace{-5mm}
            \includegraphics[width=0.3\linewidth]{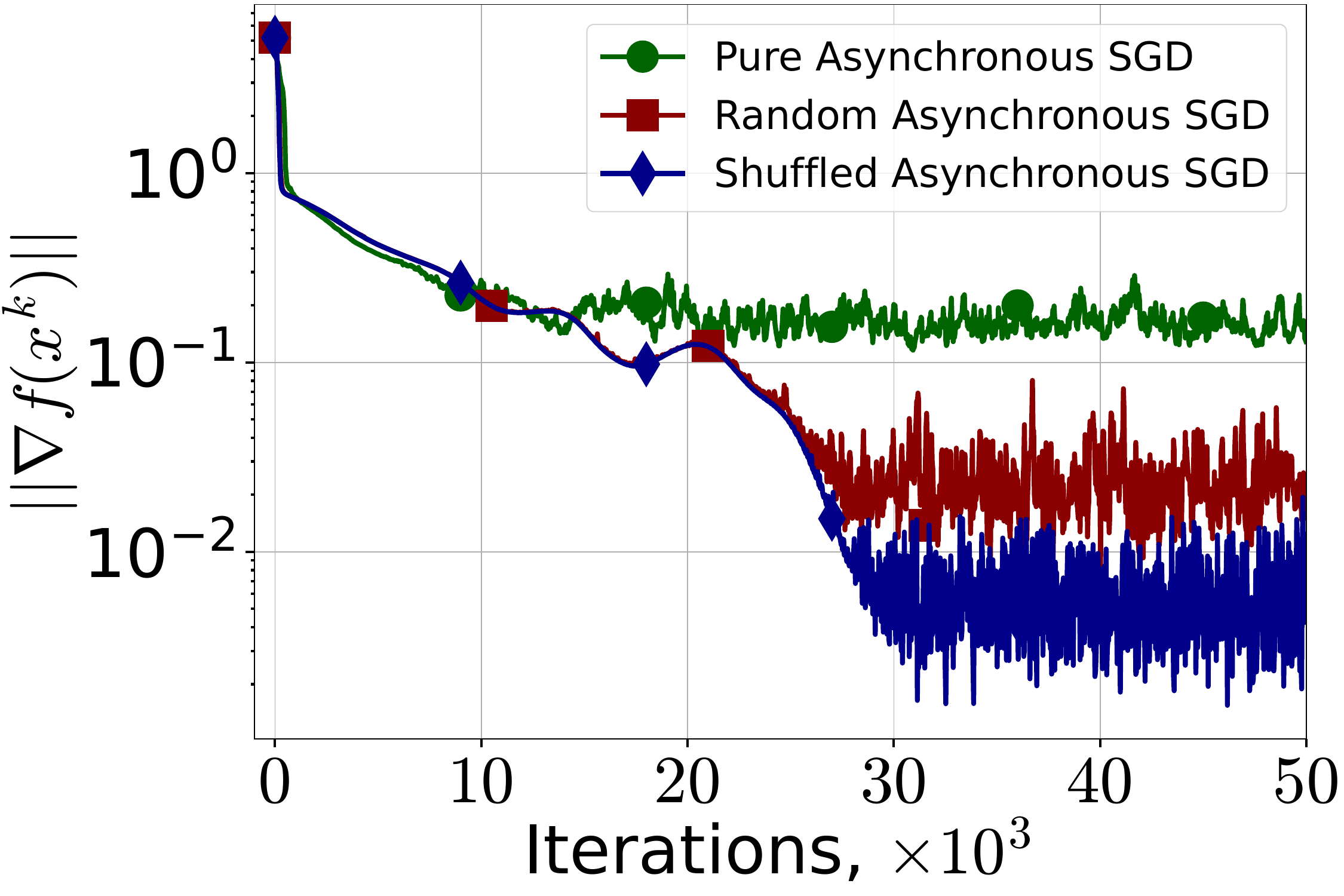} & \hspace{-5mm}
            \includegraphics[width=0.3\linewidth]{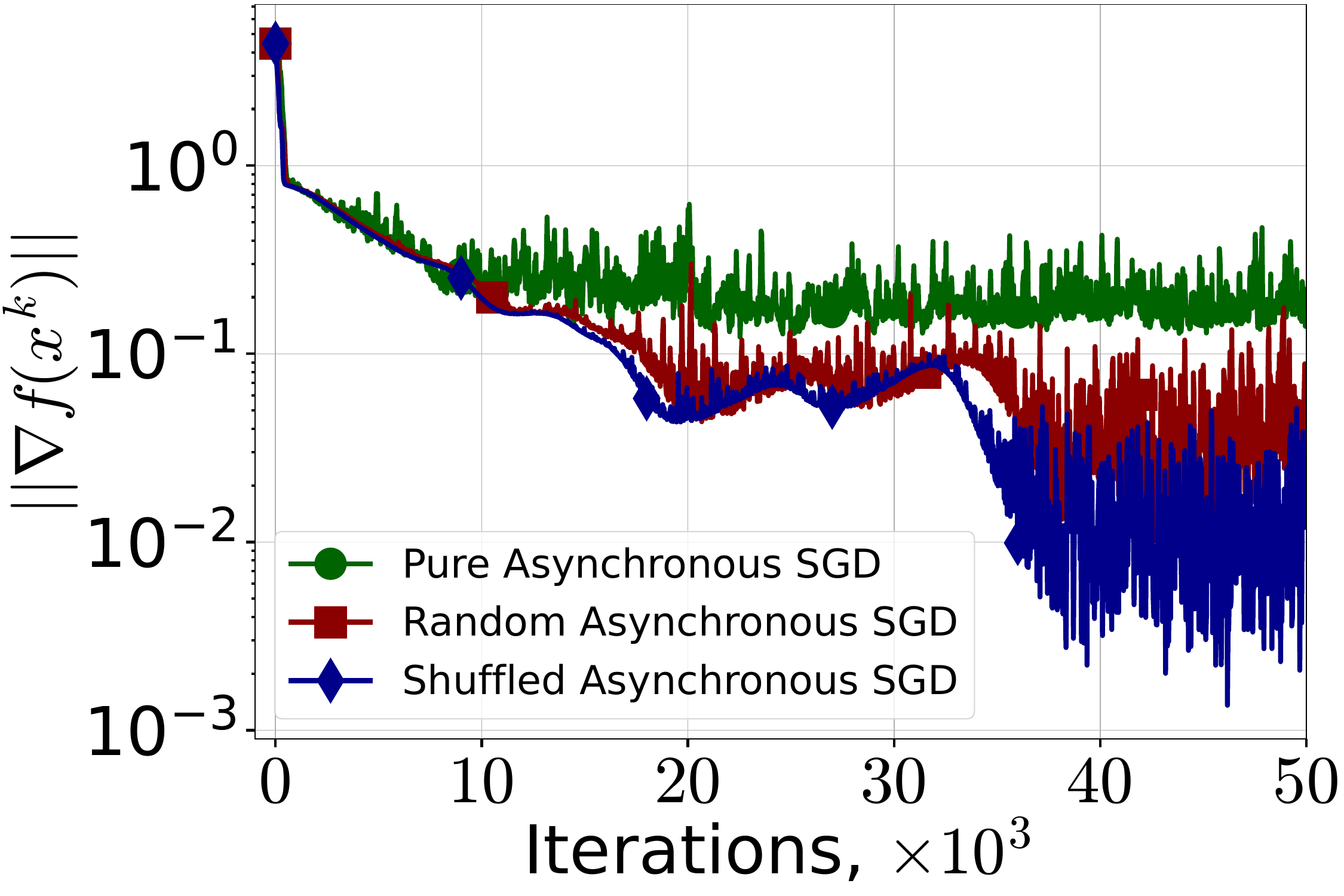}
        \end{tabular}             
    \caption{Comparison of pure, random, and shuffled asynchronous SGD with tuned stepsizes and stochastic gradient computation on synthetic datasets with different levels of heterogeneity and poisson delay pattern. Here $n=10$, $\lambda=0.1$, $d=300,$ $m=200,$ and batch size $\frac{m}{10}.$}
    \label{fig:main_stochastic_grad}
\end{figure*}

We conduct experiments on logistic regression with nonconvex regularization, namely, we consider the problem \eqref{eq:problem} with local functions
\[
    \squeeze f_i(x) = \frac{1}{m}\sum\limits_{j=1}^m \log\left(1+e^{-b_{ij}a_{ij}^\top x}\right) + \lambda\sum\limits_{j=1}^d \frac{x_j^2}{1+x_j^2}.
\]

We test the performance of pure, random, and shuffled versions of asynchronous SGD methods. We use the \texttt{w7a} and \texttt{phishing} datasets from the LibSVM library \cite{chang2011libsvm}, and additionally construct synthetic datasets following \cite{li2018federated, safaryan2022FedNL} with control of statistical heterogeneity through parameters $\alpha$ and $\beta.$\footnote{The corresponding dataset is denoted as $\text{Syn}(\alpha, \beta)$.} The detailed dataset generation and parameters setup are given in Appendix~\ref{apx:experiments}.

We assume that all workers start computing gradients in the beginning, i.e. $\tau_C = n.$ We fine-tune the stepsize and illustrate the performance with the fine-tuned stepsize only choosing the best one that achieves smaller gradient norm and fluctuations.

We conducted experiments in four different settings. In particular, each worker has a positive parameter $s_i$. The way these parameters are used is as follows

\begin{itemize}
\item  {\bf Fixed:} in this case, $r \equiv s_i$. Such fixed timing implies that the delay pattern is fixed as well.
    
\item  {\bf Poisson:} for each worker we sample $r \sim \text{Po}(s_i)$. 

\item {\bf Normal:} similarly, we sample $s \sim \cN(s_i, s_i),$ and then set $r = |s| + 1$.

\item {\bf Uniform:} in this case, we sample $r \sim \text{Uni}(0, s_i).$ 
\end{itemize}
The number $r$  indicates how many seconds worker $i$ spends to compute one gradient. Such timing patterns simulate possible workers' behavior in practice. 

First, we present the practical performance of the three aforesaid methods with full gradient computation in Figure~\ref{fig:w7a}. We observe that pure asynchronous SGD gets stuck when the gradient norm becomes about $10^{-1},$ while two other methods achieve better stationary points with smaller gradient norms. Moreover, shuffled asynchronous SGD starts converging faster than its random counterpart around $10^{-1.5}.$ Further, shuffled asynchronous SGD achieves a stationary point with roughly $10$ times smaller gradient norm. We observe similar results for {\tt phishing} dataset.

Next, we conduct the same set of experiments with stochastic gradients on synthetic datasets changing parameters $\alpha$ and $\beta$; see results in Figure~\ref{fig:main_stochastic_grad}.  We observe similar performance in stochastic setting. Random and shuffled asynchronous SGD are superior to purxe asynchronous SGD. Our proposed method always finds the stationary point with the smallest error, and in several cases, the convergence is faster than that of random asynchronous SGD.

These numerical results support our theoretical findings. First, pure asynchronous SGD does not converge to the optimum even when choosing small stepsize---it always gets stuck at the heterogeneity level. Moreover, random job assigning allows to get rid of this issue, and converge to a better stationary point. Finally, more accurate job assigning via worker shuffling allows for improving further the performance 

\section{Conclusion}

Our framework provides a deeper understanding of asynchronous SGD-type algorithms. In contrast to previous works, we derive a theoretical analysis for few method simultaneously with improved dependency on $\tau_{\max}.$ In addition, our theory provides the intuition of what affects the convergence of asynchronous-type algorithms, and how we can improve the convergence by balancing the workers' jobs. The analysis highlights two terms appearing in the rate due to delays and data ordering caused by asynchronicity. We show that different prior techniques and our new proposed approach, in fact, make those two terms in the rate smaller, and as a consequence, lead to better practical performance.  Moreover, we do not impose any assumption on the delay pattern. 
Finally, experiments support our theoretical findings. We both empirically and theoretically show that pure asynchronous SGD converges up to  heterogeneity level. Besides, we demonstrate that the proposed shuffled job assigning allows to balance more carefully the workers' contributions.

\section*{Acknowledgements}

Mher Safaryan has received funding from the European Union's Horizon 2020 research and innovation program under the Marie Skłodowska-Curie grant agreement No 101034413.

\bibliography{references}
\bibliographystyle{plain}

\clearpage
\appendix
\onecolumn
{
  \tableofcontents
}
\newpage

\section{Additional Experiments}\label{apx:experiments}

\subsection{Experimental Setup}

\begin{figure}[!t]
\centering
        \begin{tabular}{ccc}
            \hspace{8mm} {\tiny fixed, $\text{Syn}(0.5,0.5)$}  &
            \hspace{5mm}{\tiny fixed, $\text{Syn}(1,1)$}&
            \hspace{5mm}{\tiny fixed, $\text{Syn}(1.5,1.5)$}\\
            \includegraphics[width=0.3\linewidth]{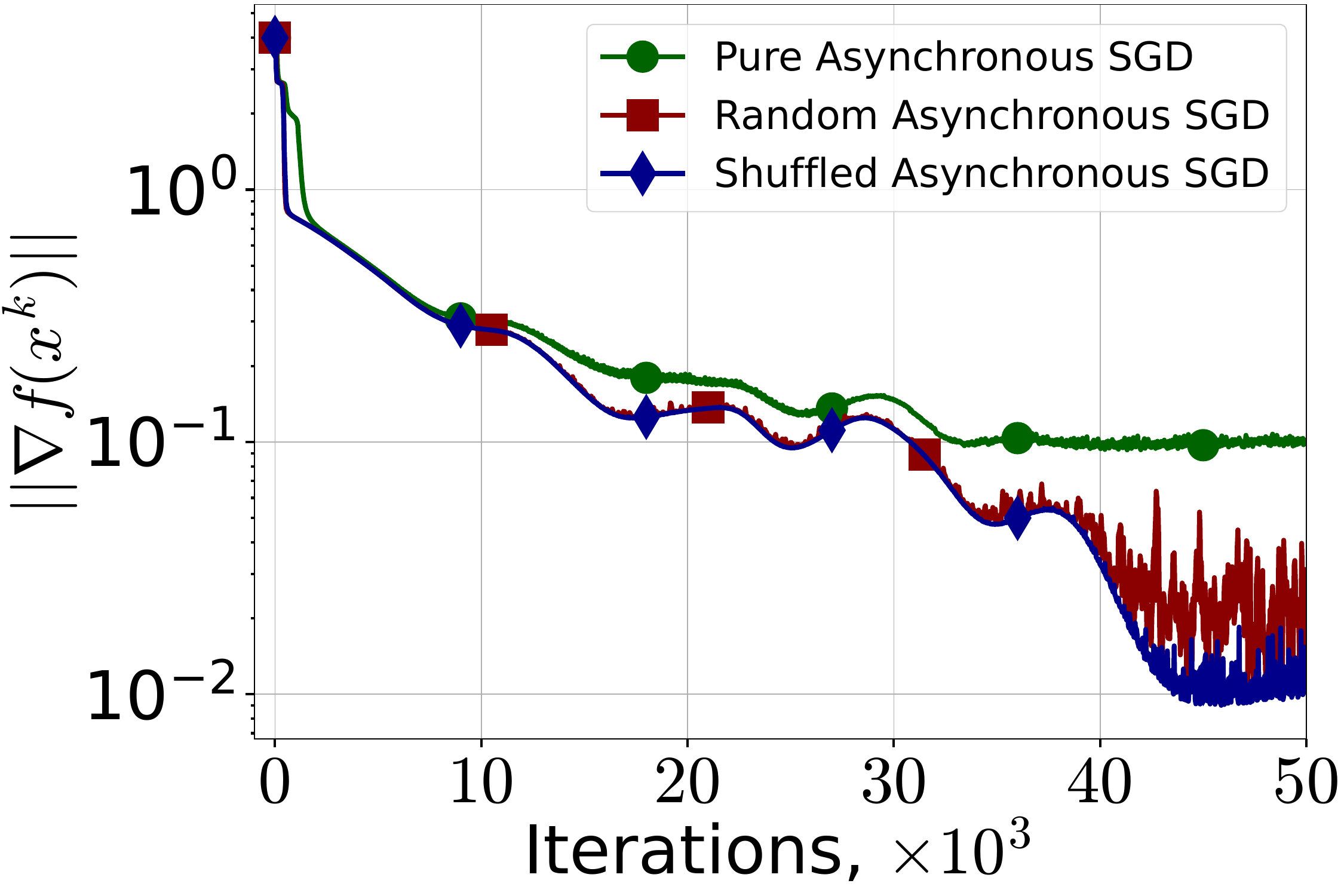} & \hspace{-5mm}
            \includegraphics[width=0.3\linewidth]{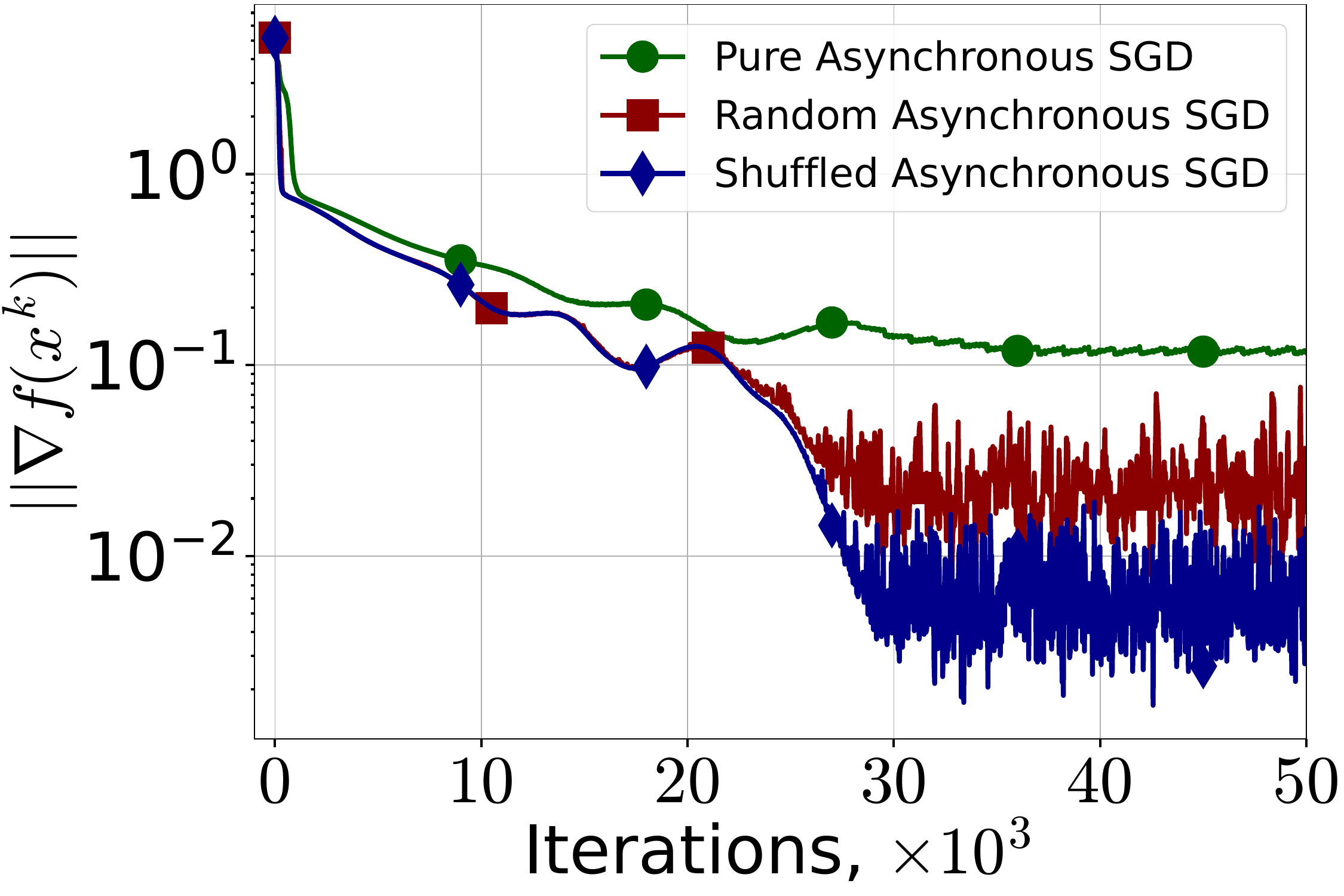} & \hspace{-5mm}
            \includegraphics[width=0.3\linewidth]
            {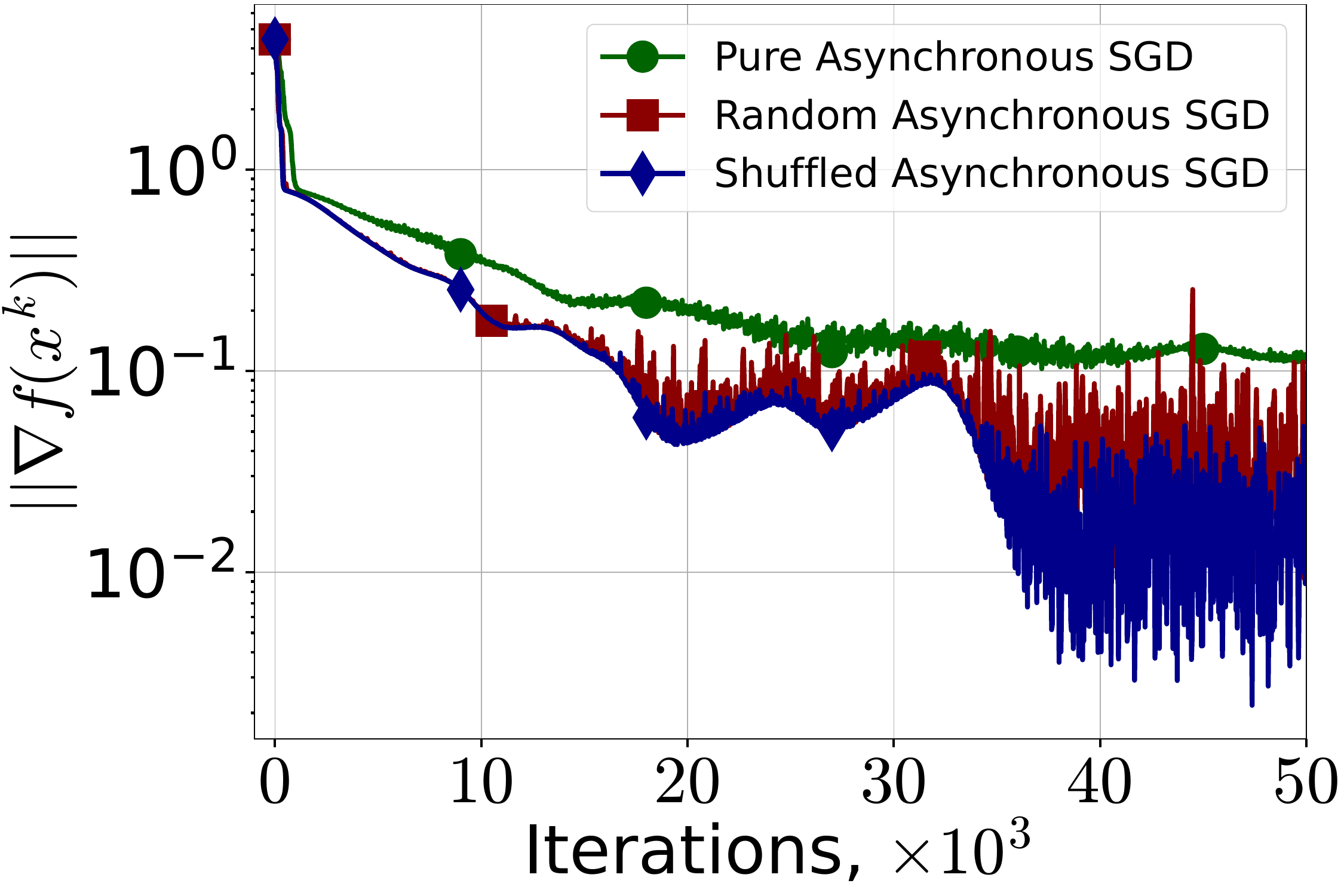}\\
            \hspace{8mm} {\tiny poisson, $\text{Syn}(0.5,0.5)$}  &
            \hspace{5mm}{\tiny poisson, $\text{Syn}(1,1)$}&
            \hspace{5mm}{\tiny poisson, $\text{Syn}(1.5,1.5)$}\\
            \includegraphics[width=0.3\linewidth]{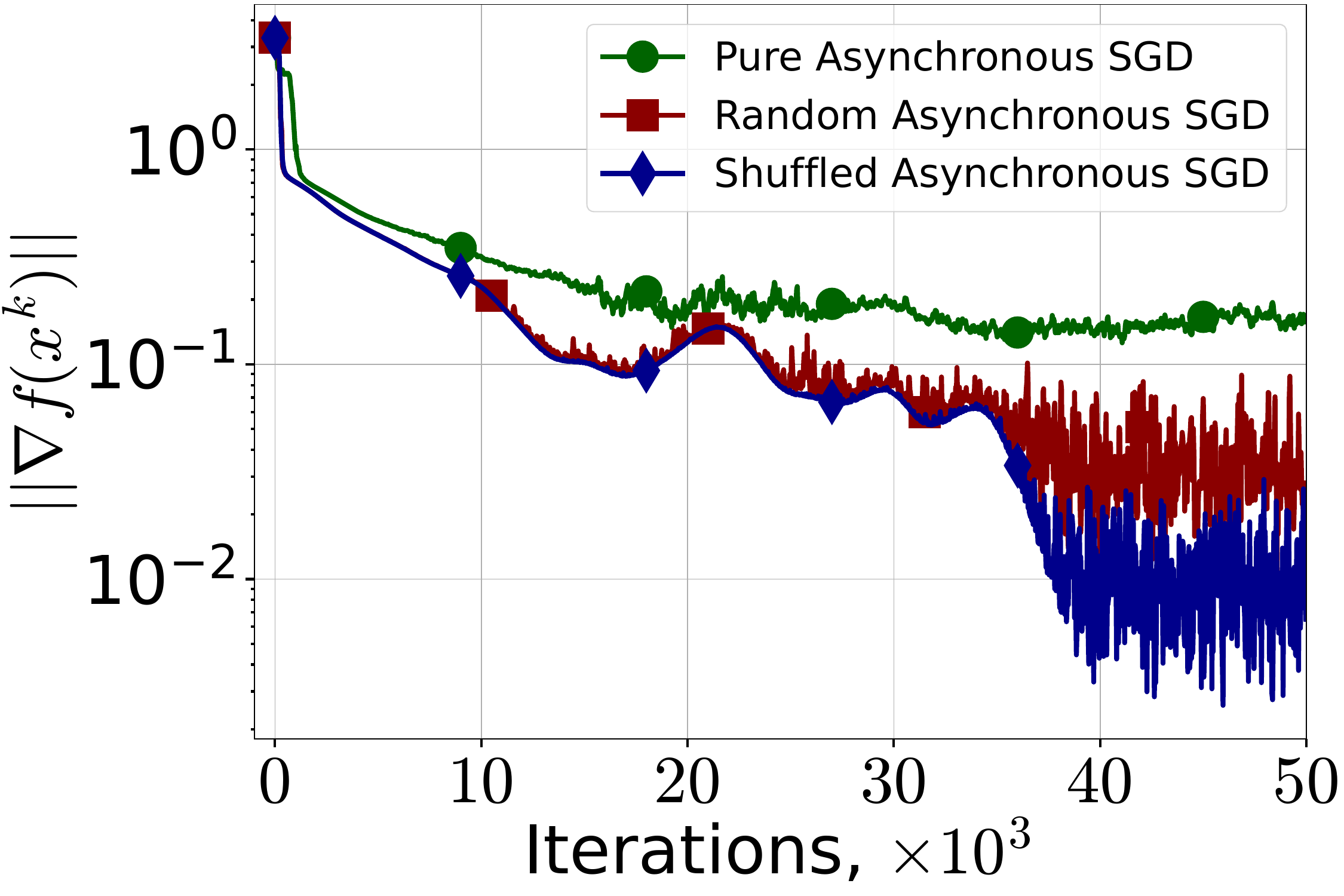} & \hspace{-5mm}
            \includegraphics[width=0.3\linewidth]{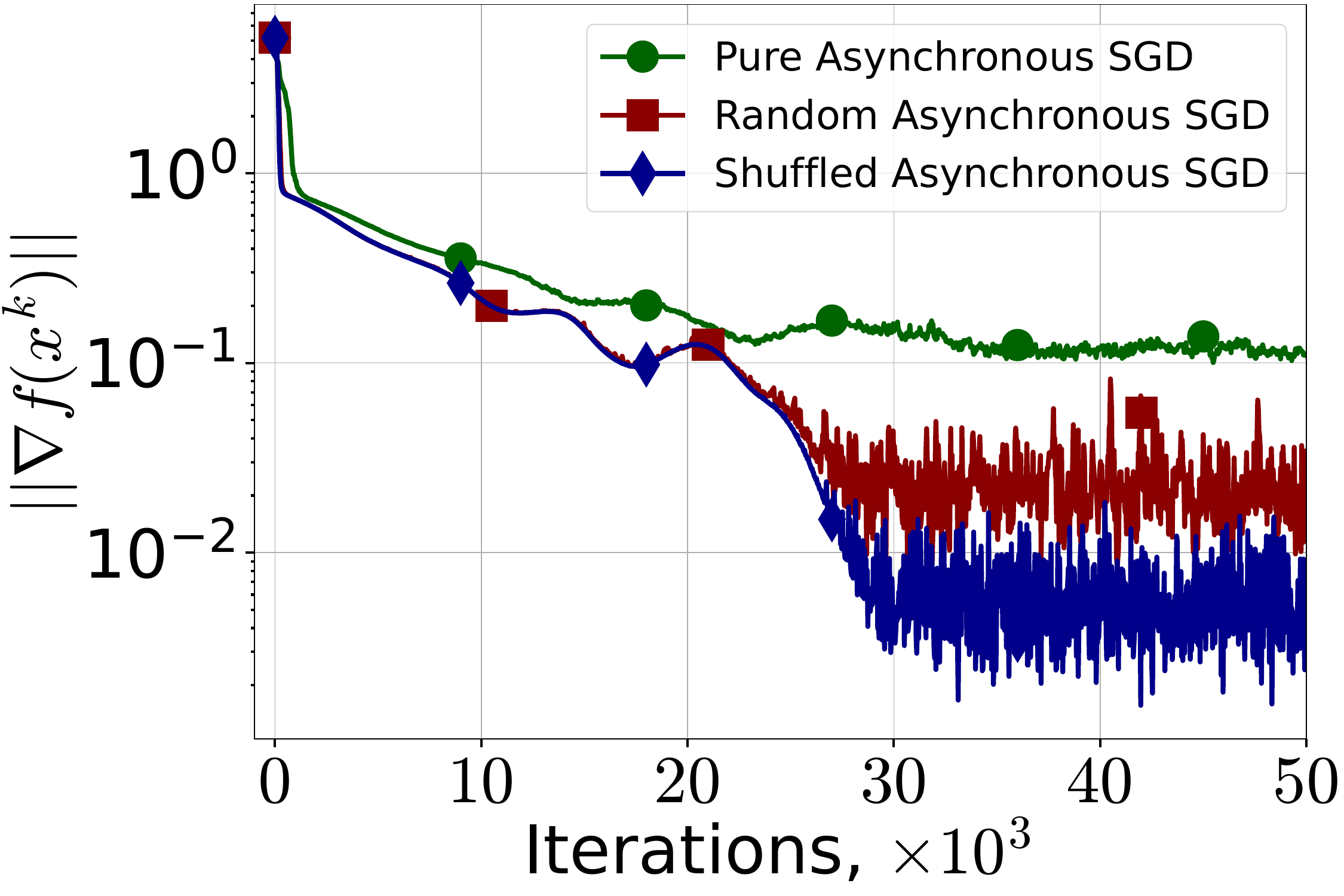} & \hspace{-5mm}
            \includegraphics[width=0.3\linewidth]{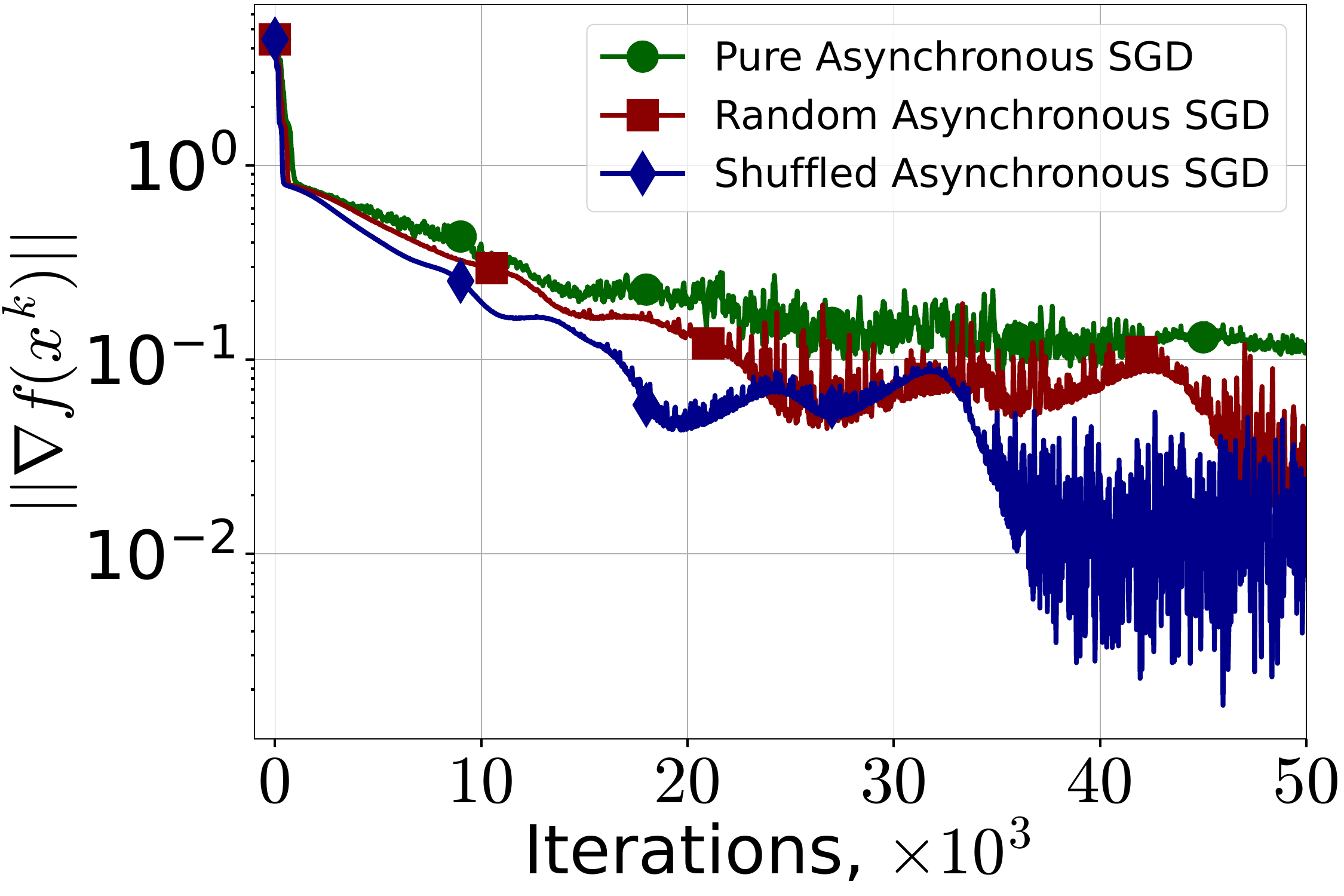}\\
            \hspace{8mm} {\tiny normal, $\text{Syn}(0.5,0.5)$}  &
            \hspace{5mm}{\tiny normal, $\text{Syn}(1,1)$}&
            \hspace{5mm}{\tiny normal, $\text{Syn}(1.5,1.5)$}\\
            \includegraphics[width=0.3\linewidth]{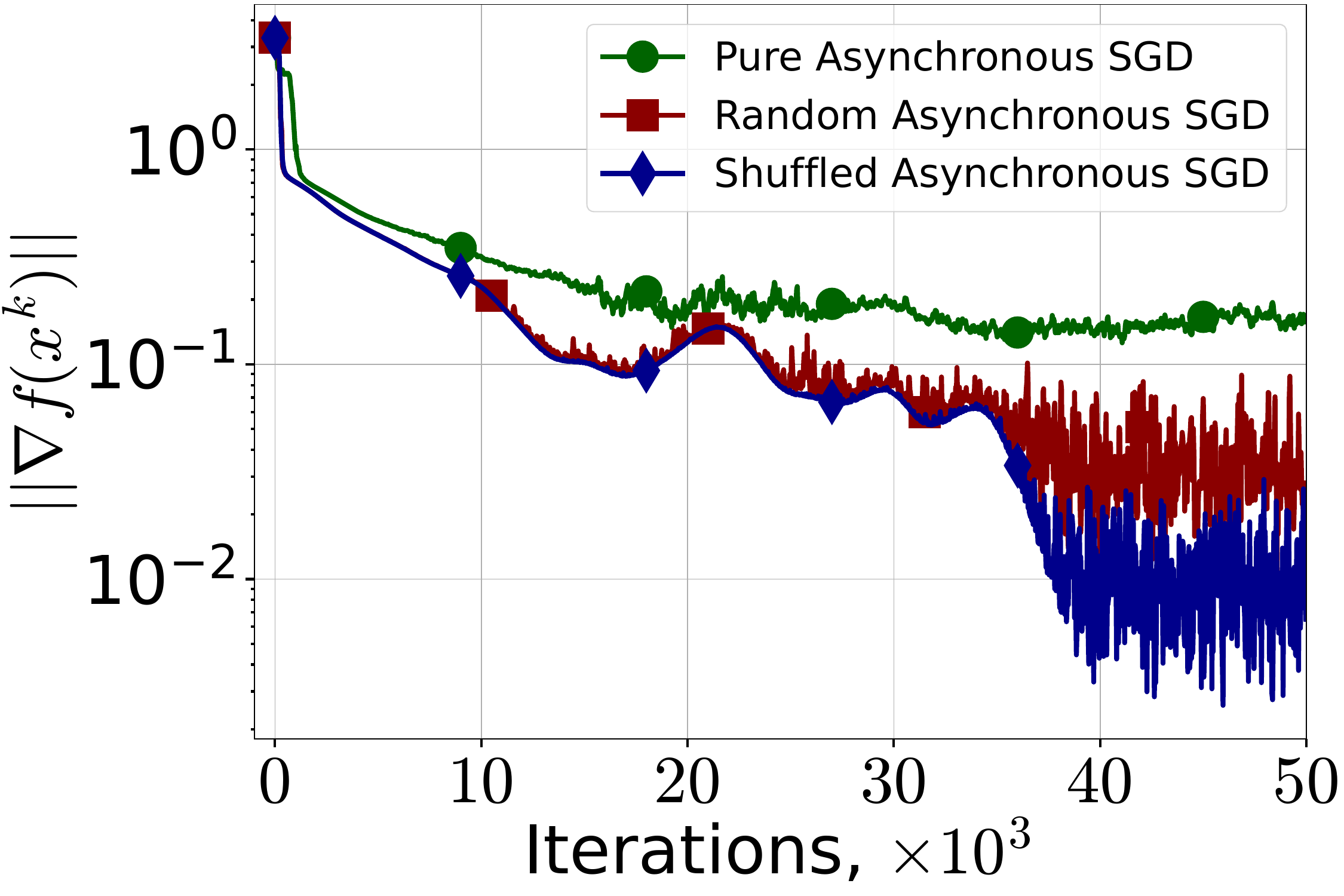} & \hspace{-5mm}
            \includegraphics[width=0.3\linewidth]{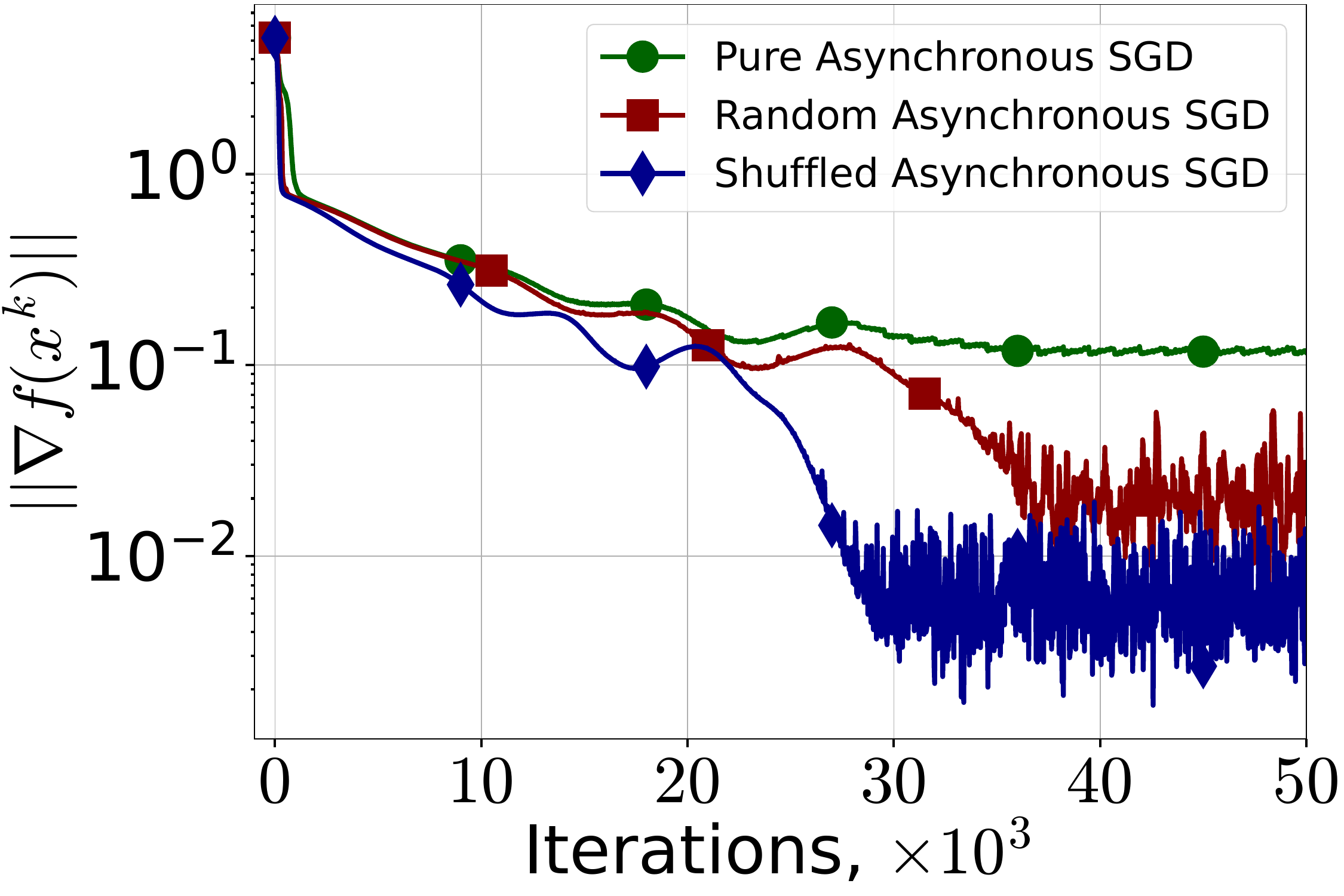} & \hspace{-5mm}
            \includegraphics[width=0.3\linewidth]{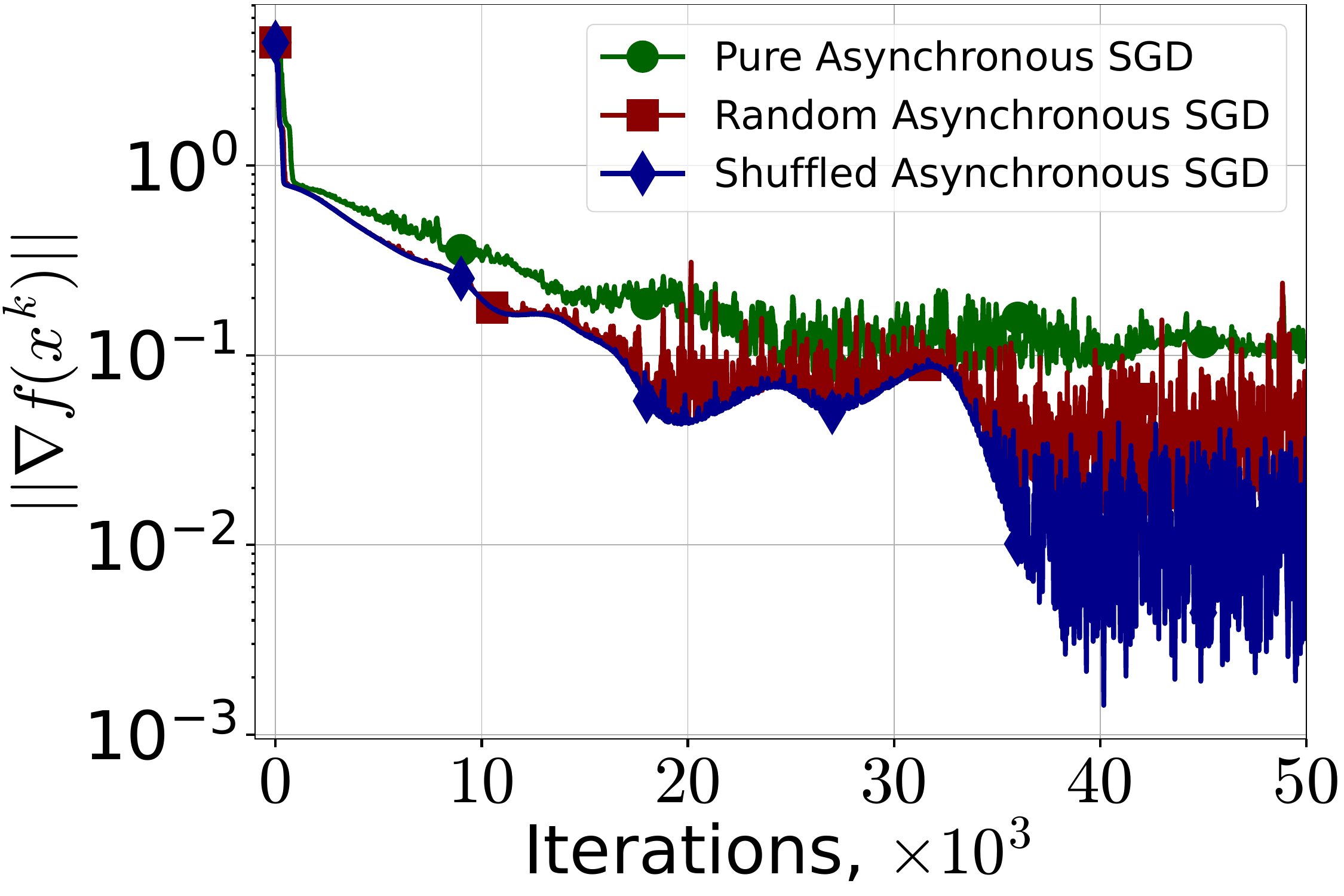}
        \end{tabular}             
    \caption{Comparison of pure, random, and shuffled asynchronous SGD with tuned stepsizes and full gradient computation on synthetic datasets with different levels of heterogeneity and various delay patterns.}
    \label{fig:full_grad}
\end{figure}

In this section, we provide additional numerical results as well as the detailed experiment setup. We consider logistic regression problem with nonconvex regularization, namely,
\begin{equation*}
    f(x) = \frac{1}{n}\sum_{i=1}^n f_i(x), \quad \text{ where } \quad  f_i(x) = \frac{1}{m}\sum\limits_{j=1}^m \log\left(1+e^{-b_{ij}a_{ij}^\top x}\right) + \lambda\sum\limits_{j=1}^d \frac{x_j^2}{1+x_j^2}.
\end{equation*}
For synthetic datasets, we set $n=10, m = 200, d=300$, and $\lambda = 0.1$ in all experiments. The initial point $x_0\in \R^d$ is sampled from standard Gaussian distribution. 

We compare the practical performance of pure asynchronous SGD, random asynchronous SGD, and shuffled asynchronous SGD. We assume that all workers are active from the beginning which implies that $\tau_C = n.$ We perform a grid search for the stepsize over the set $\{0.005, 0.004, 0.003, 0.002,$ $ 0.001, 0.0005, 0.0001\},$ and illustrate the  performance with the fine-tuned stepsize only choosing the best one that achieves smaller gradient norm and fluctuations.

\subsection{Synthetic Dataset Generation}

The generation function of datasets has two parameters $\alpha$ and $\beta.$ Below we detail the rules of dataset generation:
\begin{enumerate}
    \item Sample $n$ numbers $B_i \sim \cN(0, \beta)$.
    \item Sample $n$ vectors $v_i \in \R^d$ such that each component $[v_i]_j \sim \cN(B_i, 1).$
    \item For all $i \in [n]$ sample $m$ vectors $a_{ij}\in \R^d$ such that $a_{ij} \sim \cN(v_i, \Sigma)$ where $\Sigma$ is a diagonal matrix with $\Sigma_{jj} = j^{-1.2}.$ A set $\{a_{ij}\}_{j=1}^m$ will be utilized feature vectors of $i$-th worker.
    \item Sample $n$ pairs of numbers $u_i \sim \cN(0, \alpha)$ and $c_i \sim \cN(u_i, 1).$
    \item Sample $n$ vectors $w_i\in \R^d$ such that each component $[w_i]_j \sim \cN(u_i, 1).$
    \item Let $p_{ij} = \sigma(w_i^\top a_{ij} + c_i)$ where $\sigma(\cdot)$ is a sigmoid function. 
    \item For all $i \in [n]$ sample $m$ numbers $b_{ij}$ such that $b_{ij}$ equals to $-1$ with probability $p_{ij}$ and to $1$ otherwise.
\end{enumerate}
The larger $\alpha$ and $\beta$ are, the more heterogeneous local datasets are. In our experiments, we choose these parameters from the set $\{0.5, 1, 1.5\}$.

\subsection{Comparison on Synthetic Dataset}

First, we test the aforementioned algorithms in the case of full local gradient computation in order to avoid the effect of stochasticity and illustrate the impact of delay patterns only; see Figure~\ref{fig:full_grad}. We change the heterogeneity level controlling parameters $\alpha$ and $\beta$, and delay patterns. We observe that in all cases pure asynchronous SGD gets stuck much earlier than its counterparts regardless of the delay pattern. However, due to that, we can choose a larger stepsize which leads to faster convergence. 

Besides, we notice that the shuffled asynchronous SGD method finds a better stationary point than the random version which happens because of the more careful job assigning. Moreover, the shuffled version sometimes can be even faster (e.g., normal Syn(1,1) and poison Syn(1.5, 1.5)).

\section{Useful Lemmas}

We frequently use several well-known results.

\begin{lemmasec}\label{lem:lemmaA1}
    Let $\{v_i\}_{i=1}^m$ be a finite set of vectors in $\R^d.$ Then we have
    \begin{equation}
        \left\|\sum_{i=1}^mv_i\right\|^2 \le m\sum_{i=1}^m\|v_i\|^2.
    \end{equation}
\end{lemmasec}

\begin{lemmasec}\label{lem:lemmaA2}
    Let $a, b$ be any two vectors in $\R^d$, and $\alpha$ be any positive number. Then we have
    \begin{equation}
        \<a, b> \le \frac{\alpha}{2}\|a\|^2 + \frac{1}{2\alpha}\|b\|^2.
    \end{equation}
\end{lemmasec}

The next lemma is useful when we work we double sums in the proofs. Let $\beta(t, \pi_t)$ be any non-negative function with two iterate-dependent arguments $t$ and $\pi_t$.  For example, $\beta(t, \pi_t) = \|\nabla f_{i_t}(x_{\pi_t})\|^2$ or $\beta(t, \pi_t) = \|x_t - x_{\pi_t}\|^2$. 

\begin{lemmasec}\label{lem:lemmaA3} Let $\beta(t, \pi_t)$ be defined as above. If $\tau_C$ is defined in \eqref{def:max_active_jobs}, then we have
\begin{equation}
    \sum_{t=0}^{T-1}\sum_{j=\pi_t}^{t-1}\beta(j, \pi_j) \le (\tau_C-1) \sum_{t=0}^{T-1}\beta(t, \pi_t).
\end{equation}
\begin{proof}
    Each term $\beta(t, \pi_t)$ appears for all $t^\prime$ such that $t \in [\pi_{t^\prime}, t^\prime).$ We need to understand how many such $t^\prime$ might exist at most.  Indeed, such term might appear only for $t^\prime$ before $t$, thus, that job should have been started somewhere in the past but still not been  applied. The number of such it is bounded by $|\cA_t\setminus \cR_t| = |\cA_{t+1} \setminus \cR_t|-1 \le \tau_C-1.$ 
\end{proof}
    
\end{lemmasec}

\section{Proofs for Analysis of Gradient Receiving Process}\label{sec:proof_theorem3}

For $r(t) \le m < r(t) + \tau$, denote
$$
\phi_t^m(x) \eqdef \E{\left\| \sum_{j=r(t)}^{m} (\nabla f(x) - \nabla f_{i_j}(x)) \right\|^2}.
$$

Moreover, for shortness we use the following notation 
\begin{eqnarray*}
    A  &\eqdef& \sum_{t=0}^{T-1}\E{\|x_t - x_{\pi_t}\|^2}, \quad B \eqdef \sum_{t=0}^{T-1}\E{\|\nabla f(x_t)\|^2}, \quad \Phi \eqdef \sum_{t=0}^{T-1}\phi_t^{t-1}(x_{r(t)})\\
    \Psi &\eqdef& \sum_{t=0}^{T-1}\E{\left\|\sum_{j=\pi_t}^{t-1}\nabla f_{i_j}(x_{\pi_j}) - \nabla f(x_{\pi_j})\right\|^2}.
\end{eqnarray*}

The key part of the proof is bounding the distance between virtual and real iterates. Based on \eqref{eq:real_iterates_update} and \eqref{eq:real_virtual_seq}, we have
\begin{eqnarray*}
    x_t = x_{r(t)} - \gamma\sum_{j=r(t)}^{t-1} g_{i_j}(x_{\pi_j}), \qquad
    \wtilde{x}_t = x_{r(t)} - \gamma\sum_{j=r(t)}^{t-1} \nabla f(x_j).
\end{eqnarray*}

Also, denote
\begin{eqnarray}\label{eq:def_delta}
    \Delta_t^m \eqdef \sum_{j=r(t)}^m (\nabla f(x_j) - g_{i_j}(x_{\pi_j})).
\end{eqnarray}

\subsection{Key Lemmas}

\begin{lemmasec}\label{lem:lemmaD1}
If $20 \gamma L\tau \le 1$, then
\begin{eqnarray*}
    \E{\|\Delta_t^m\|^2}
    &\le& 4\E{\phi_t^m(x_{r(t)})}
    + \frac{1}{6\tau} \sum_{j=r(t)}^{m} \E{\phi_t^{j-1}(x_{r(t)})}
     +\; \frac{25}{6}L^2\tau\sum_{j=r(t)}^{m} \E{\|x_j - x_{\pi_j}\|^2}\\
    && +\;  \frac{\tau}{24} \sum_{j=r(t)}^{m} \E{\left\| \nabla f(x_{j}) \right\|^2}+
     \frac{\tau}{24}\sigma^2.
\end{eqnarray*}
\end{lemmasec}
\begin{proof}
Fix any $m$ such that $r(t) \le m < r(t) + \tau$, so that the iterates are within one block.
\begin{eqnarray}
    &&\E{\|\Delta_t^m\|^2}
    \le  \E{\left\|\sum_{j=r(t)}^{m} \nabla f_{i_j}(x_{\pi_j}) - \nabla f(x_{j}) \right\|^2} + \tau\sigma^2\notag \\
    &\le& 4 \E{\left\|\sum_{j=r(t)}^{m} \nabla f_{i_j}(x_{r(t)}) - \nabla f(x_{r(t)}) \right\|^2} 
    + 4\E{\left\|\sum_{j=r(t)}^{m} \nabla f_{i_j}(x_{r(t)}) - \nabla f_{i_j}(x_{j}) \right\|^2} \notag \\
    &&\quad +\;  4\E{\left\|\sum_{j=r(t)}^{m} \nabla f_{i_j}(x_j) - \nabla f_{i_j}(x_{\pi_j}) \right\|^2}
    + 4 \E{\left\|\sum_{j=r(t)}^{m} \nabla f(x_{r(t)}) - \nabla f(x_{j}) \right\|^2} + \tau\sigma^2 \notag\\
    &\le& 4 \E{\phi_t^m(x_{r(t)})} 
    + 4L^2\tau\sum_{j=r(t)}^{m} \E{\|x_{j} - x_{\pi_j}\|^2}
    + 8L^2\tau\sum_{j=r(t)}^{m} \E{\|x_{j} - x_{r(t)}\|^2} + \tau\sigma^2.\label{eq:lemmaD1_1}
\end{eqnarray}

Next, we bound the third term involving $\E{\|x_j - x_{r(t)}\|^2}$.

\begin{eqnarray*}
    && \sum_{j=r(t)}^{m} \E{\|x_j - x_{r(t)}\|^2}
    = \gamma^2 \sum_{j=r(t)}^{m} \E{\left\| \sum_{l=r(t)}^{j-1}  g_{i_l}(x_{\pi_l}) \right\|^2} \\
    &\le& 2\gamma^2 \sum_{j=r(t)}^{m} \E{\left\| \sum_{l=r(t)}^{j-1} g_{i_l}(x_{\pi_l}) - \nabla f(x_{l}) \right\|^2}
    + 2\gamma^2 \sum_{j=r(t)}^{m} \E{\left\| \sum_{l=r(t)}^{j-1} \nabla f(x_{l}) \right\|^2}
    \\
    &=& 2\gamma^2 \sum_{j=r(t)}^{m} \E{\left\| \Delta_t^{j-1} \right\|^2}
    + 2\gamma^2 \sum_{j=r(t)}^{m} \E{\left\| \sum_{l=r(t)}^{j-1} \nabla f(x_{l}) \right\|^2}
    \\
    &\overset{\eqref{eq:lemmaD1_1}}{\le}&
    8\gamma^2 \sum_{j=r(t)}^{m} \E{\phi_t^{j-1}(x_{r(t)})}
    + 16\gamma^2L^2\tau \sum_{j=r(t)}^{m}\sum_{l=r(t)}^{j-1} \E{\|x_l - x_{r(t)}\|^2}\\
    && +\; 8\gamma^2L^2\tau \sum_{j=r(t)}^{m}\sum_{l=r(t)}^{j-1} \E{\|x_l - x_{\pi_l}\|^2} + 2\gamma^2\tau \sum_{j=r(t)}^{m}\sum_{l=r(t)}^{j-1} \E{\left\| \nabla f(x_{l}) \right\|^2} + 2\gamma^2\tau^2\sigma^2
    \\
    &\le&
    8\gamma^2 \sum_{j=r(t)}^{m} \E{\phi_t^{j-1}(x_{r(t)})}
    + 16\gamma^2L^2\tau^2 \sum_{j=r(t)}^{m} \E{\|x_j - x_{r(t)}\|^2}
    + 8\gamma^2L^2\tau^2 \sum_{j=r(t)}^{m} \E{\|x_j - x_{\pi_j}\|^2} \\
    && +\; 2\gamma^2\tau^2 \sum_{j=r(t)}^{m} \E{\left\| \nabla f(x_{j}) \right\|^2}
    +2\gamma^2\tau^2\sigma^2,
\end{eqnarray*}
where we use the fact that $\tau_C \ge 1$. Choosing $\gamma L\tau \le \frac{1}{20}$, we cancel the term $\|x_j - x_{r(t)}\|^2$ from the right hand side:
\begin{eqnarray*}
    \sum_{j=r(t)}^{m} \E{\|x_j - x_{r(t)}\|^2}
    &\le&
    \frac{25}{3}\gamma^2 \sum_{j=r(t)}^{m} \E{\phi_t^{j-1}(x_{r(t)})}
    + \frac{25}{3}\gamma^2L^2\tau^2 \sum_{j=r(t)}^{m} \E{\|x_j - x_{\pi_j}\|^2}\\
    && +\;  \frac{25}{12}\gamma^2\tau^2 \sum_{j=r(t)}^{m} \E{\left\| \nabla f(x_{j}) \right\|^2} + \frac{25}{12}\gamma^2\tau^2\sigma^2.
\end{eqnarray*}

Plugging this inequality back to \eqref{eq:lemmaD1_1}, we get
\begin{eqnarray*}
    \E{\|\Delta_t^m\|^2}
    &\le& 4 \E{\phi_t^m(x_{r(t)})} + 4L^2\tau\sum_{j=r(t)}^{m} \E{\|x_{j} - x_{\pi_j}\|^2}
    + 8L^2\tau\sum_{j=r(t)}^{m} \E{\|x_{j} - x_{r(t)}\|^2}\\
    &\le& 4\E{\phi_t^m(x_{r(t)})}
    + \frac{200}{3}\gamma^2L^2\tau \sum_{j=r(t)}^{m} \E{\phi_t^{j-1}(x_{r(t)})} \\
    && +\; \frac{25}{6}L^2\tau\sum_{j=r(t)}^{m} \E{\|x_j - x_{\pi_j}\|^2}
    + \frac{50}{3}\gamma^2L^2\tau^3 \sum_{j=r(t)}^{m} \E{\left\| \nabla f(x_{j}) \right\|^2} 
    + \frac{50}{3}\gamma^2L^2\tau^3\sigma^2.
\end{eqnarray*}
Using the relation $\gamma^2L^2\tau^2\le\frac{1}{400}$ further, we can simplify the above bound as
\begin{eqnarray*}
    \E{\|\Delta_t^m\|^2}
    &\le& 4\E{\phi_t^m(x_{r(t)})}
    + \frac{1}{6\tau} \sum_{j=r(t)}^{m} \E{\phi_t^{j-1}(x_{r(t)})}
     +\; \frac{25}{6}L^2\tau\sum_{j=r(t)}^{m} \E{\|x_j - x_{\pi_j}\|^2}\\
    && +\;  \frac{\tau}{24} \sum_{j=r(t)}^{m} \E{\left\| \nabla f(x_{j}) \right\|^2}+
     \frac{\tau}{24}\sigma^2.
\end{eqnarray*}
Since $\frac{1}{16} > \frac{1}{24}$ we conclude the proof.
\end{proof}

\begin{lemmasec}\label{lem:lemmaD2}
If $20 \gamma L\tau \le 1$, then
\begin{eqnarray}\label{eq:lemmaD2_1}
\sum_{t=0}^{T-1} \E{\|x_t - \wtilde{x}_t\|^2}
&\le&
    \frac{25}{6}\gamma^2 \underbrace{\sum_{t=0}^{T-1} \E{\phi_t^{t-1}(x_{r(t)})}}_{\eqdef \Phi}
    + \frac{1}{96} \underbrace{\sum_{t=0}^{T-1} \E{\|x_t - x_{\pi_t}\|^2}}_{\eqdef A}\notag\\
    &&+\;\frac{1}{9600L^2} \underbrace{\sum_{t=0}^{T-1} \E{\left\| \nabla f(x_{t}) \right\|^2}}_{\eqdef B} 
    +  \frac{\gamma^2\tau}{24}\sigma^2T.
\end{eqnarray}
\end{lemmasec}
\begin{proof}
Notice that for $r(t) < j \le t$ it holds $r(t)=r(j)$. Hence, we can replace $\phi_t^{j-1}(x_{r(t)})$ by $\phi_j^{j-1}(x_{r(j)})$ in the summation below. If $j\le r(t)$, then $\phi_t^{j-1}(x_{r(t)})=0$.
\begin{eqnarray*}
    \sum_{t=0}^{T-1} \E{\|x_t - \wtilde{x}_t\|^2}
    &=& \gamma^2 \sum_{t=0}^{T-1}\E{\left\|\sum_{j=r(t)}^{t-1}  g_{i_j}(x_{\pi_j}) - \nabla f(x_{j}) \right\|^2}\\
    &=&  \gamma^2 \sum_{t=0}^{T-1}\E{\|\Delta_t^{t-1}\|^2}\\
    &\overset{\textrm{Lemma}~\ref{lem:lemmaD1}}{\le}& 4\gamma^2 \sum_{t=0}^{T-1} \E{\phi_t^{t-1}(x_{r(t)})}
    + \frac{\gamma^2}{6\tau} \sum_{t=0}^{T-1}\sum_{j=r(t)}^{t-1} \E{\phi_t^{j-1}(x_{r(t)})} \\
    && +\; \frac{1}{96\tau} \sum_{t=0}^{T-1}\sum_{j=r(t)}^{t-1} \E{\|x_j - x_{\pi_j}\|^2}
    + \frac{\gamma^2\tau}{24}\sum_{t=0}^{T-1}\sum_{j=r(t)}^{t-1} \E{\left\| \nabla f(x_{j}) \right\|^2} \\
    && +\; \frac{\gamma^2\tau}{24}\sigma^2T\\
    &\le& \frac{25}{6}\gamma^2 \sum_{t=0}^{T-1} \E{\phi_t^{t-1}(x_{r(t)})}
    + \frac{1}{96} \sum_{t=0}^{T-1} \E{\|x_t - x_{\pi_t}\|^2}\\
    && +\;\frac{\gamma^2\tau^2}{24} \sum_{t=0}^{T-1} \E{\left\| \nabla f(x_{t}) \right\|^2} 
    + \frac{\gamma^2\tau}{24}\sigma^2T.
\end{eqnarray*}
Using the bound on the learning rate we conclude the proof.
\end{proof}

\begin{lemmasec}\label{lem:lemmaD3}
    If $20\gamma L\sqrt{\tau_{\max}\tau_C}\le 1$, then
\begin{eqnarray*}
\underbrace{\sum_{t=0}^{T-1} \E{\|x_t - x_{\pi_t}\|^2}}_{= A}
&\le&
\frac{1}{132L^2} \underbrace{\sum_{t=0}^{T-1} \E{\left\| \nabla f(x_{t}) \right\|^2}}_{= B}
+ \frac{\gamma}{5L}T\sigma^2\\
&& \;+\; \frac{100\gamma^2}{33}\underbrace{\sum_{t=0}^{T-1}\E{\left\|\sum_{j=\pi_t}^{t-1} \nabla f_{i_j}(x_{\pi_j}) - \nabla f(x_{\pi_j})\right\|^2}}_{= \Psi}.
\end{eqnarray*}
\end{lemmasec}
\begin{proof}
Recall that $\pi_t = t - \tau_t$.
\begin{eqnarray*}
    && \E{\|x_t - x_{\pi_t}\|^2}
    = \gamma^2 \E{\left\|\sum_{j=\pi_t}^{t-1} g_{i_j}(x_{\pi_j})\right\|^2} \\
    &\le& \gamma^2 \E{\left\|\sum_{j=\pi_t}^{t-1} \nabla f_{i_j}(x_{\pi_j})\right\|^2} +\tau_t\gamma^2\sigma^2\\
    &\le& 3\gamma^2 \E{\left\|\sum_{j=\pi_t}^{t-1} \nabla f_{i_j}(x_{\pi_j}) - \nabla f(x_{\pi_j})\right\|^2}
    + 3\gamma^2 \E{\left\|\sum_{j=\pi_t}^{t-1} \nabla f(x_{\pi_j}) - \nabla f(x_j)\right\|^2}\\
    && + \; 3\gamma^2 \E{\left\|\sum_{j=\pi_t}^{t-1} \nabla f(x_{j})\right\|^2} +\tau_t\gamma^2\sigma^2 \\
    &\le& 3\gamma^2 L^2 \tau_t \sum_{j=\pi_t}^{t-1} \E{\|x_{\pi_j} - x_j\|^2}
    + 3\gamma^2 \E{\left\|\sum_{j=\pi_t}^{t-1} \nabla f_{i_j}(x_{\pi_j}) - \nabla f(x_{\pi_j})\right\|^2}\\
    && + \;3\gamma^2\tau_t \sum_{j=\pi_t}^{t-1} \left\| \nabla f(x_{j})\right\|^2 + \tau_t\gamma^2\sigma^2.
\end{eqnarray*}

Then we add summation over the entire iterates and count the number of times (which is $\le\tau_C$) each term appears in the sum:
\begin{eqnarray*}
    \sum_{t=0}^{T-1} \E{\|x_{\pi_t}- x_t\|^2}
    &\le& 3\gamma^2 L^2 \tau_{\max} \sum_{t=0}^{T-1}\sum_{j=\pi_t}^{t-1} \E{\|x_{\pi_j} - x_j\|^2} 
    + 3\gamma^2\tau_{\max}\sum_{t=0}^{T-1}\sum_{j=\pi_t}^{t-1} \E{\left\| \nabla f(x_{j})\right\|^2}\\
    && +\;3\gamma^2 \sum_{t=0}^{T-1}\E{\left\|\sum_{j=\pi_t}^{t-1} \nabla f_{i_j}(x_{\pi_j}) - \nabla f(x_{\pi_j})\right\|^2} 
    + \tau_{\text{avg}}T\gamma^2\sigma^2 \\
    &\le& 
    3\gamma^2 L^2 \tau_{\max}\tau_C \sum_{t=0}^{T-1}\E{\|x_{\pi_t} - x_t\|^2}
    + 3\gamma^2\tau_{\max}\tau_C \sum_{t=0}^{T-1} \E{\left\| \nabla f(x_{t})\right\|^2} \\
    && +\; 3\gamma^2 \sum_{t=0}^{T-1}\E{\left\|\sum_{j=\pi_t}^{t-1} \nabla f_{i_j}(x_j) - \nabla f(x_j)\right\|^2} +\tau_{\text{avg}}T\gamma^2\sigma^2\\
    &\le& \frac{3}{400} \sum_{t=0}^{T-1} \E{\|x_{\pi_t} - x_t\|^2}
    + \frac{3}{400L^2} \sum_{t=0}^{T-1} \E{\left\| \nabla f(x_{t})\right\|^2}\\
    &&+\; 3\gamma^2 \sum_{t=0}^{T-1}\E{\left\|\sum_{j=\pi_t}^{t-1} \nabla f_{i_j}(x_{\pi_j}) - \nabla f(x_{\pi_j})\right\|^2} +\tau_{\text{avg}}T\gamma^2\sigma^2,
\end{eqnarray*}
provided that $\gamma^2 L^2 \tau_{\max}\tau_C \le \frac{1}{400}$. After cancellation, we get
\begin{eqnarray*}
\sum_{t=0}^{T-1} \E{\|x_t - x_{\pi_t}\|^2}
&\le&
    \frac{1}{132L^2} \sum_{t=0}^{T-1} \E{\left\| \nabla f(x_{t})\right\|^2}
    +  2\tau_{\text{avg}}T\gamma^2\sigma^2 \\
    && +\;\frac{100\gamma^2}{33}\sum_{t=0}^{T-1}\E{\left\|\sum_{j=\pi_t}^{t-1} \nabla f_{i_j}(x_{\pi_j}) - \nabla f(x_{\pi_j})\right\|^2}.
\end{eqnarray*}
Since $\tau_{\avg} \le \tau_{\max}$ and $\tau_{\avg} \le 2\tau_C$ (the latter holds because of the Remark $5$ in \cite{koloskova2022sharper}), then we finally derive 
\begin{eqnarray*}
\sum_{t=0}^{T-1} \E{\|x_t - x_{\pi_t}\|^2}
    &\le&
    \frac{1}{132L^2} \sum_{t=0}^{T-1} \E{\left\| \nabla f(x_{t})\right\|^2}
    +  \frac{2\tau_{\text{avg}}}{20L\sqrt{\tau_{\max}\tau_C}}T\gamma\sigma^2 \\
    && +\;\frac{100\gamma^2}{33}\sum_{t=0}^{T-1}\E{\left\|\sum_{j=\pi_t}^{t-1} \nabla f_{i_j}(x_{\pi_j}) - \nabla f(x_{\pi_j})\right\|^2}\\
    &\le& \frac{1}{132L^2} \sum_{t=0}^{T-1} \E{\left\| \nabla f(x_{t})\right\|^2}
    +  \frac{\gamma}{5L}T\sigma^2\\
    && +\;\frac{100\gamma^2}{33}\sum_{t=0}^{T-1}\E{\left\|\sum_{j=\pi_t}^{t-1} \nabla f_{i_j}(x_{\pi_j}) - \nabla f(x_{\pi_j})\right\|^2}.
\end{eqnarray*}
\end{proof}

We can combine all the previous lemmas and bound $\E{\|x_t - \wtilde{x}_t\|^2}$ in the following lemma:

\begin{lemmasec}\label{lem:lemmaD5}
	If $20 \gamma L\tau \le 1$ and $20\gamma L \sqrt{\tau_{\max}\tau_C} \le 1$ hold, then
	\begin{eqnarray}
		\sum_{t=0}^{T-1} \E{\|x_t - \wtilde{x}_t\|^2}
		&\le&
		5\gamma^2 \Phi 
        + \gamma^2 \Psi
        + \frac{1}{ 5460 L^2} B 
        + \frac{1}{240L}\gamma\sigma^2T.\label{eq:lemmaD5_1}
	\end{eqnarray}
\end{lemmasec}
\begin{proof}
	Summary of obtained inequalities 
	\begin{eqnarray*}
		\sum_{t=0}^{T-1} \E{\|x_t - \wtilde{x}_t\|^2}
		&\overset{{\rm Lemma}~\ref{lem:lemmaD2}}{\le}&
		\frac{25}{6}\gamma^2 \Phi 
        + \frac{1}{96}A 
		+ \frac{1}{9600L^2} B 
		+ \frac{\gamma^2\tau}{24}\sigma^2T\\
		A &\overset{{\rm Lemma}~\ref{lem:lemmaD3}}{\le}& 
        \frac{1}{132L^2} B
        + \frac{100}{33}\gamma^2 \Psi 
        + \frac{\gamma}{5L}T\sigma^2.
	\end{eqnarray*}

Hence, using the restriction $20\gamma L\tau \le 1$ we get 
\begin{eqnarray*}
	\sum_{t=0}^{T-1} \E{\|x_t - \wtilde{x}_t\|^2}
	&\le& \frac{25}{6}\gamma^2 \Phi 
    + \frac{1}{96}A 
    + \frac{1}{9600L^2} B 
    + \frac{\gamma }{480L}\sigma^2T\\
	&\le&  \frac{25}{6}\gamma^2 \Phi
    + \frac{1}{96\cdot 132L^2} B 
    + \frac{1}{ 9600L^2} B 
    + \frac{1}{30}\gamma^2 \Psi 
    + \frac{11\gamma T \sigma^2}{100L}\\
    &=& 5\gamma^2 \Phi
    + \frac{1}{5460L^2} B 
    + \gamma^2 \Psi 
    + \frac{\gamma}{240L}T \sigma^2.
\end{eqnarray*}
\end{proof}

\subsection{Proof of Theorem~\ref{th:theorem3}}

\theoremthird*
\begin{proof}
The analysis starts from deriving a descent inequality for the virtual iterates $\wtilde{x}_t$ defined in \eqref{eq:real_virtual_seq}. Consider two cases for virtual iterates: depending on whether restarts happen or not.

{\bf Iterations without restart:} If restarts do not happen, namely $(t+1)\mod\tau\ne 0$ then from the smoothness assumption of $f$ and \eqref{eq:real_virtual_seq}, we have
\begin{eqnarray*}
	\E{f(\wtilde{x}_{t+1})}
	&\le& \E{f(\wtilde{x}_t)}
	- \gamma\E{\langle \nabla f(\wtilde{x}_t), \nabla f(x_t) \rangle}
	+ \frac{L\gamma^2}{2}\E{\|\nabla f(x_t)\|^2} \\
	&=&   \E{f(\wtilde{x}_t)}
	- \frac{\gamma}{2}\E{\|\nabla f(\wtilde{x}_t)\|^2} - \E{\frac{\gamma}{2}\|\nabla f(x_t)\|^2} + \frac{\gamma}{2}\E{\|\nabla f(\wtilde{x}_t) - \nabla f(x_t)\|^2}\\
	&& + \; \frac{L\gamma^2}{2}\E{\|\nabla f(x_t)\|^2} \\
	&\le& \E{f(\wtilde{x}_t)}
	- \frac{\gamma}{2}\E{\|\nabla f(\wtilde{x}_t)\|^2 }- \frac{\gamma}{2}\E{\|\nabla f(x_t)\|^2} + \frac{L^2\gamma}{2}\E{\|\wtilde{x}_t - x_t\|^2}\\
	&& + \; \frac{L\gamma^2}{2}\E{\|\nabla f(x_t)\|^2} \\
	&\le& \E{f(\wtilde{x}_t)}
	- \frac{\gamma}{2}\E{\|\nabla f(\wtilde{x}_t)\|^2} - \frac{\gamma}{3}\E{\|\nabla f(x_t)\|^2} + \frac{L^2\gamma}{2}\E{\|\wtilde{x}_t - x_t\|^2}.
\end{eqnarray*}

{\bf Iterations with restart:} If a restart happens, namely $(t+1)\mod\tau = 0$ then
\begin{eqnarray*}
	\wtilde{x}_{t+1}
	&=& x_{t+1} = x_t - \gamma g_{i_t}(x_{\pi_t}) \\
	&=& \wtilde{x}_t + (x_t - \wtilde{x}_t) -\gamma\nabla f(x_t) + (\gamma\nabla f(x_t) - \gamma g_{i_t}(x_{\pi_t})) \\
	&=& \wtilde{x}_t -\gamma\nabla f(x_t) + \gamma \underbrace{\sum_{j=r(t)}^t \nabla f(x_j) -  g_{i_j}(x_{\pi_j})}_{= \Delta_t^t}.
\end{eqnarray*}

Then we use smoothness of $f$ to get
\begin{eqnarray*}
	&& \E{f(\wtilde{x}_{t+1})} \\
	&\le& \E{f(\wtilde{x}_t)} - \gamma\E{\langle \nabla f(\wtilde{x}_t), \nabla f(x_t) - \wtilde{\Delta}_t^t \rangle} + \frac{L\gamma^2}{2}\E{\|\nabla f(x_t) - \wtilde\Delta_t^t\|^2} \\
	&\le& \E{f(\wtilde{x}_t)}
	- \gamma\E{\langle \nabla f(\wtilde{x}_t), \nabla f(x_t) \rangle}
	+ \gamma\E{\langle \nabla f(\wtilde{x}_t), \wtilde{\Delta}_t^t \rangle}
	+ L\gamma^2\E{\|\nabla f(x_t)\|^2}
	+ L\gamma^2\E{\|\wtilde{\Delta}_t^t\|^2}\\
	&\le& \E{f(\wtilde{x}_t)}
	-\frac{\gamma}{2}\E{\|\nabla f(\wtilde{x}_t)\|^2}
	-\frac{\gamma}{2}\E{\|\nabla f(x_t)\|^2}
	+\frac{\gamma}{2}\E{\|\nabla f(\wtilde{x}_t) - \nabla f(x_t)\|^2} \\
	&& \;+\; \frac{1}{160L}\E{\|\nabla f(\wtilde{x}_t)\|^2} +  40L\gamma^2\E{\|\wtilde{\Delta}_t^t\|^2}
	+ L\gamma^2\E{\|\nabla f(x_t)\|^2}
	+ L\gamma^2\E{\|\wtilde{\Delta}_t^t\|^2}\\
	&\le& \E{f(\wtilde{x}_t)}
	- \frac{\gamma}{2}\E{\|\nabla f(\wtilde{x}_t)\|^2}
	- \frac{\gamma}{3}\E{\|\nabla f(x_t)\|^2}
	+ \frac{L^2\gamma}{2}\E{\|\wtilde{x}_t - x_t\|^2}
	+ \frac{1}{ 160L}\E{\|\nabla f(\wtilde{x}_t)\|^2},
\end{eqnarray*}
where in the third inequality Young's inequality is used. 

If we denote by $\xi_t$ the indicator function of restart event at $t+1$, namely $\xi_{k\tau-1}=1$ for all $k\ge1$ and is $0$ otherwise, then we can unify the descent inequality for both cases as follows:
\begin{eqnarray}
	\E{f(\wtilde{x}_{t+1})}
	&\le& \E{f(\wtilde{x}_t)}
	- \frac{\gamma}{2}\E{\|\nabla f(\wtilde{x}_t)\|^2}
	- \frac{\gamma}{3}\E{\|\nabla f(x_t)\|^2}
	+ \frac{L^2\gamma}{2}\E{\|\wtilde{x}_t - x_t\|^2} \notag\\
	&& \;+\; \left(
	\frac{1}{ 160L}\E{\|\nabla f(\wtilde{x}_t)\|^2}
	+ 41L\gamma^2\E{\|\Delta_t^t\|^2}\right)\xi_t,
	\quad \forall~t\ge0. \label{eq:theorem3_1}
\end{eqnarray}

Next, we apply summation over the entire iterates and bound the terms that appear only in every $\tau$ iteration.

\begin{eqnarray*}
	\frac{1}{L}\E{\|\nabla f(\wtilde{x}_t)\|^2}
	&=& \frac{1}{L\tau}\sum_{j=0}^{\tau-1} \E{\|\nabla f(\wtilde{x}_t)\|^2}\\
	&\le& \frac{2}{L\tau}\sum_{j=0}^{\tau-1} \E{\|\nabla f(\wtilde{x}_t) - \nabla f(\wtilde{x}_{t-j})\|^2}
	+ \frac{2}{L\tau}\sum_{j=0}^{\tau-1} \E{\|\nabla f(\wtilde{x}_{t-j})\|^2} \\
	&\le& \frac{2L}{\tau}\sum_{j=0}^{\tau-1} \E{\|\wtilde{x}_t - \wtilde{x}_{t-j}\|^2}
	+ \frac{2}{L\tau}\sum_{j=0}^{\tau-1} \E{\|\nabla f(\wtilde{x}_{t-j})\|^2 }\\
	&\le& \frac{2L\gamma^2}{\tau}\sum_{j=0}^{\tau-1} \E{\left\| \sum_{l=t-j}^{t-1} \nabla f(x_l) \right\|^2}
	+ \frac{2}{L\tau}\sum_{j=0}^{\tau-1} \E{\|\nabla f(\wtilde{x}_{t-j})\|^2} \\
	&\le& 2L\gamma^2 \sum_{j=0}^{\tau-1}\sum_{l=t-j}^{t-1} \E{\left\| \nabla f(x_l) \right\|^2}
	+ \frac{2}{L\tau}\sum_{j=0}^{\tau-1} \E{\|\nabla f(\wtilde{x}_{t-j})\|^2} \\
	&\le& 2L\gamma^2\tau \sum_{j=0}^{\tau-1} \E{\left\| \nabla f(x_{t-j}) \right\|^2}
	+ \frac{2}{L\tau}\sum_{j=0}^{\tau-1}\E{\|\nabla f(\wtilde{x}_{t-j})\|^2} \\
	&\le& \frac{\gamma}{10} \sum_{j=0}^{\tau-1} \E{\left\| \nabla f(x_{t-j}) \right\|^2}
	+ 80\gamma \sum_{j=0}^{\tau-1} \E{\|\nabla f(\wtilde{x}_{t-j})\|^2}
\end{eqnarray*}
provided that $\frac{1}{40} \le L\gamma\tau \le \frac{1}{20}$ (e.g., $\tau=\lfloor \frac{1}{20L\gamma}\rfloor$). Then we can use this bound to derive
\begin{eqnarray}\label{eq:theorem3_2}
	\sum_{t=0}^{T-1} \frac{1}{160L}\E{\|\nabla f(\wtilde{x}_t)\|^2 }\xi_t
	\le \frac{\gamma}{1600} \sum_{t=0}^{T-1}\E{\left\| \nabla f(x_{t}) \right\|^2}
	+ \frac{\gamma}{2} \sum_{t=0}^{T-1}\E{\|\nabla f(\wtilde{x}_{t})\|^2}.
\end{eqnarray}

Then we use Lemma \ref{lem:lemmaD1} to bound $\Delta_t^t$:
\begin{eqnarray*}
	L\gamma^2\sum_{t=0}^{T-1} \E{\|\Delta_t^t\|^2}\xi_t
	&\le& 4L\gamma^2 \sum_{t=0}^{T-1}\E{\phi_t^t(x_{r(t)})}\xi_t
	+ \frac{L\gamma^2}{6\tau} \sum_{t=0}^{T-1}\sum_{j=r(t)}^{t} \E{\phi_t^{j-1}(x_{r(t)})}\xi_t \\
	&& \;+\; \frac{25}{6}L^3\gamma^2\tau \sum_{t=0}^{T-1}\sum_{j=r(t)}^{t} \E{\|x_j - x_{\pi_j}\|^2}\xi_t\\
	&& +\; \frac{L\gamma^2\tau}{24}\sum_{t=0}^{T-1}\sum_{j=r(t)}^{t} \E{\left\| \nabla f(x_{j}) \right\|^2}\xi_t 
    + \frac{L\gamma^2\tau}{24}\sigma^2\sum_{t=0}^{T-1}\xi_t.
\end{eqnarray*}
Notice that summation over the entire iterates with weights $\xi_t$ is equivalent to division by $\tau$.
$$
\sum_{t=0}^{T-1}\sum_{j=r(t)}^{t} \phi_t^{j-1}(x_{r(t)})\xi_t
= \sum_{t=0}^{T-1}\sum_{j=r(t)}^{t-1} \phi_t^{j}(x_{r(t)})\xi_t
= \sum_{t=0}^{T-1}\sum_{j=r(t)}^{t-1} \phi_j^{j}(x_{r(j)})\xi_t
\le \sum_{t=0}^{T-1} \phi_t^{t}(x_{r(t)}).
$$
Hence
\begin{eqnarray*}
	41L\gamma^2\sum_{t=0}^{T-1} \E{\|\wtilde{\Delta}_t^t\|^2}\xi_t
	&\le& 164L\gamma^2 \sum_{t=0}^{T-1}\E{\phi_t^t(x_{r(t)})}\xi_t
	+ \frac{41L\gamma^2}{6\tau} \sum_{t=0}^{T-1} \E{\phi_t^{t}(x_{r(t)})} \notag \\
	&& \;+\; 171L^3\gamma^2\tau \underbrace{\sum_{t=0}^{T-1} \E{\|x_t - x_{\pi_t}\|^2}}_{= A}\notag \\
	&& \; + \; \frac{41}{24}L\gamma^2\tau \sum_{t=0}^{T-1} \E{\left\| \nabla f(x_{t}) \right\|^2} 
    + \frac{41}{24}L\gamma^2\sigma^2T.
\end{eqnarray*}
where $A \le \frac{1}{132L^2}B 
+ \frac{100}{33}\gamma^2\Psi 
+ \frac{\gamma T\sigma^2}{5L} $ due to Lemma \ref{lem:lemmaD3}. Therefore,
\begin{eqnarray}
    &&41L\gamma^2\sum_{t=0}^{T-1} \E{\|\Delta_t^t\|^2}\xi_t
	\le 164L\gamma^2 \sum_{t=0}^{T-1}\E{\phi_t^t(x_{r(t)})}\xi_t
	+ \frac{41L\gamma^2}{6\tau} \sum_{t=0}^{T-1} \E{\phi_t^{t}(x_{r(t)})} \notag \\
	&& \;+\; \frac{171}{132}L\gamma^2\tau B 
    + 520L^3\gamma^4\tau \Psi 
    + \frac{171}{100}L\gamma^2\sigma^2 T \notag \\
	&& 
    \; + \; \frac{41}{24}L\gamma^2\tau \sum_{t=0}^{T-1} \E{\left\| \nabla f(x_{t}) \right\|^2} 
    + \frac{41}{24}L\gamma^2\sigma^2T\notag\\
    && \le 
    164L\gamma^2 \sum_{t=0}^{T-1}\E{\phi_t^t(x_{r(t)})}\xi_t
	+ \frac{41L\gamma^2}{6\tau} \sum_{t=0}^{T-1} \E{\phi_t^{t}(x_{r(t)})} 
    + \frac{171}{132}L\gamma^2\tau B 
    + 26L^2\gamma^3 \Psi\notag \\
    && 
    \;+\; \frac{171}{100}L\gamma^2\sigma^2 T 
    +  \frac{41}{24}L\gamma^2\tau B
    + \frac{41}{16}L\gamma^2\sigma^2T\notag\\
    &&\le  6950L^2\gamma^3 \sum_{t=0}^{T-1}\E{\phi_t^t(x_{r(t)})} 
    + \frac{793}{5280}\gamma B 
    + 26L^2\gamma^3 \Psi 
    + 4L\gamma^2\sigma^2T.\label{eq:theorem3_3}
\end{eqnarray}

Plugging \eqref{eq:theorem3_2} and \eqref{eq:theorem3_3} (and lemmas as well) into \eqref{eq:theorem3_1} and adding summation, we have 
\begin{eqnarray*}
&&\E{f(\wtilde{x}_{T}) - f(\wtilde{x}_0)}
\le
    -\frac{\gamma}{2} \sum_{t=0}^{T-1}\E{\|\nabla f(\wtilde{x}_t)\|^2}
    - \frac{\gamma}{3} \sum_{t=0}^{T-1}\E{\|\nabla f(x_t)\|^2}
    + \frac{L^2\gamma}{2} \sum_{t=0}^{T-1}\E{\|\wtilde{x}_t - x_t\|^2 }\notag\\
    && \;+\;
        \frac{1}{160L} \sum_{t=0}^{T-1}\xi_t\E{\|\nabla f(\wtilde{x}_t)\|^2}
        + 41L\gamma^2 \sum_{t=0}^{T-1}\xi_t\E{\|\Delta_t^t\|^2}\\
    &\le&
    - \frac{\gamma}{2} \sum_{t=0}^{T-1}\E{\|\nabla f(\wtilde{x}_t)\|^2}
    - \frac{\gamma}{3} B\\
    &&\;+\; \frac{L^2\gamma}{2} \left(
        5\gamma^2 \sum_{t=0}^{T-1} \E{\phi_t^{t-1}(x_{r(t)})}
        + \gamma^2 \Psi
        + \frac{1}{ 5460 L^2} B 
        + \frac{1}{240L}\gamma\sigma^2T\right) \notag\\
    && 
    \;+\; \frac{\gamma}{1600} B
    + \frac{\gamma}{2} \sum_{t=0}^{T-1}\E{\|\nabla f(\wtilde{x}_t)\|^2} \\
    && \;+\; 6950L^2\gamma^3 \sum_{t=0}^{T-1}\E{\phi_t^t(x_{r(t)})} 
    + \frac{793}{5280}\gamma B 
    + 26L^2\gamma^3 \Psi 
    + 4L\gamma^2\sigma^2T\\
    &\le&
    \left( -\frac{\gamma}{2} + \frac{\gamma}{2} \right) \sum_{t=0}^{T-1}\E{\|\nabla f(\wtilde{x}_t)\|^2} + \left( -\frac{\gamma}{3} 
    + \frac{\gamma}{10920} 
    + \frac{\gamma}{1600} 
    + \frac{793}{5280}\gamma \right) B \\
    && \;+\; \frac{5}{2}L^2\gamma^3 \sum_{t=0}^{T-1} \E{\phi_t^{t-1}(x_{r(t)})}
    + 6950L^2\gamma^3 \sum_{t=0}^{T-1}\E{\phi_t^t(x_{r(t)})}
    + 27L^2\gamma^3\Psi\\
    && \;+\;   5L\gamma^2\sigma^2T\\
    &\le& -\frac{\gamma}{5}B + \frac{5}{2}L^2\gamma^3 \sum_{t=0}^{T-1} \E{\phi_t^{t-1}(x_{r(t)})}
    + 6950L^2\gamma^3 \sum_{t=0}^{T-1}\E{\phi_t^t(x_{r(t)})}
    + 27L^2\gamma^3\Psi\\
    && \;+\;   5L\gamma^2\sigma^2T
\end{eqnarray*}
According to the statement of the theorem, we assume that
$$\sigma_{k,\tau}^2 \eqdef \sup\limits_{x\in\R^d}\max_{0 \le j < \tau}\E{\left\|\sum_{t=k\tau}^{\min\{k\tau+j, T\}}(\nabla f_{i_t}(x_{k\tau})-\nabla f(x_{k\tau}))\right\|^2},$$
and
$$\nu^2 \eqdef \sum_{t=0}^{T-1}\E{\left\|\sum_{j=\pi_t}^{t-1}(\nabla f_{i_j}(x_{\pi_j})-\nabla f(x_{\pi_j}))\right\|^2}$$
are bounded. So $\nu^2$ upper bounds $\Psi$ term in the above inequality while $\sigma_{k,\tau}^2$ upper bounds $\phi_t^{t-1}$ and $\phi_t^t$ terms. Then, by averaging we get

\begin{equation}\label{eq:theorem3_4}
\frac{1}{T}\sum_{t=0}^{T-1} \E{\|\nabla f(x_t)\|^2}
\le
\frac{5(f(x_0) - f^*)}{\gamma T}
+ PL^2\gamma^2 \left( \frac{1}{\lfloor T/\tau \rfloor}\sum_{k=0}^{\lfloor \frac{T}{\tau} \rfloor-1} \sigma^2_{k,\tau}
+ \frac{1}{T}\nu^2\right) 
+ 25L\gamma\sigma^2,
\end{equation}
where $P\eqdef 35000,$ and we use the fact that $\wtilde x_0 = x_0$ by the definition~\eqref{eq:real_virtual_seq}. 

\end{proof}

\subsection{Convergence Guarantees in Special Cases}\label{sec:special_cases_real_theorem1}

We emphasize that in the choice of stepsizes, we omit numerical constants for simplicity, and show the dependency on the parameters of the problem only.

\subsubsection{Pure Asynchronous SGD}

\begin{algorithm*}[t]
\caption{Pure Asynchronous SGD}
\label{alg:pure_asynchronous}
\begin{algorithmic}[1]
\State \textbf{Input:} $x_0\in \R^{d}$, stepsize $\gamma > 0$, set of assigned jobs $\cA_0 = \empty$, set of received jobs $\cR_{0} = \empty$
\State \textbf{Initialization:} for all jobs $(i, 0) \in \cA_1$, server assigns worker $i$ to compute a stochastic gradient $g_i(x_0)$ 
    \For{$t = 0,1,2,\dots, T-1$}
		\State once worker $i_t$ finishes a job $(i_t,\pi_t) \in \cA_{t+1}$, it sends $g_{i_t}(x_{\pi_t})$ to the server
        \State server updates the current model $x_{t+1} = x_t - \gamma g_{i_t}(x_{\pi_t})$ and the set $\cR_{t+1} = \cR_{t} \cup \{(i_t, \pi_t)\}$
        \State server assigns worker $i_t$ to compute a gradient $g_{i_t}(x_{t+1})$
        \State server updates the set $\cA_{t+2} = \cA_{t+1} \cup \{(i_{t}, t+1)\}$
    \EndFor
\end{algorithmic}	
\end{algorithm*}

At the beginning of the algorithm, the server assigns each worker to compute a gradient at $x_0$, i.e. $\cA_1 = \{(i,0)\colon i\in[n]\}$. As soon as one of them (denoted by $i_t$) finishes its job, the server assigns him back new job at the updated point, i.e. $k_t \equiv i_t$. We present the convergence for pure asynchronous SGD.

\begin{propositionsec}\label{prop:pure_async_sgd}
    Let Assumptions~\ref{asmp:smoothness}, \ref{asmp:bound_var}, and \ref{asmp:grad_sim} hold. Let the stepsize $\gamma$ satisfy inequalities $20L\gamma\sqrt{\tau_{\max}\tau_C} \le 1$ and $6L\gamma \le 1$. Let $\tau = \lfloor\frac{1}{20L\gamma} \rfloor$. Then the iterates of Algorithm~\ref{alg:pure_asynchronous} satisfy
    \begin{equation}\label{eq:propB1_1}
        \E{\|\nabla f(\hat{x}_t)\|^2} \le \cO\left(\frac{F_0}{\gamma T} + L\gamma\sigma^2+ \zeta^2\right),
    \end{equation}
    where $\hat{x}_t$ is chosen uniformly at random from $\{x_0, \dots, x_{T-1}\}$ and $F_0 \eqdef f(x_0) - f^*.$ Moreover, if we tune the stepsize, then the iterates of pure asynchronous SGD satisfy
    \begin{equation}
        \E{\|\nabla f(\hat{x}_t)\|^2} \le \cO\left(\frac{LF_0\sqrt{\tau_{\max}\tau_C}}{T} + \left(\frac{LF_0\sigma^2}{T}\right)^{1/2}+ \zeta^2\right),
    \end{equation}
\end{propositionsec}
\begin{proof}
    We need to bound quantities $\sigma_{k,\tau}^2$ and $\nu^2.$ If within interval $[k\tau, k\tau+j]$ all indices from $[n]$ appear, then we do not have to consider them since
    \begin{align*}
        \sum_{i=1}^{n} (\nabla f_{i}(x) - \nabla f(x)) = 0.
    \end{align*}
    Thus, we have 
    \begin{align*}
        \sum_{t=k\tau}^{k\tau+j} (\nabla f_{i_t}(x) - \nabla f(x)) &=  \sum_{i \in \cS} (\nabla f_{i}(x) - \nabla f(x)),
    \end{align*}

    where $|S|\le \tau,$ and hence, we continue as follows
    \begin{eqnarray*}
        \left\|\sum_{t=k\tau}^{k\tau+j} (\nabla f_{i_t}(x) - \nabla f(x)) \right\|^2 &=&  \left\|\sum_{i \in \cS} (\nabla f_{i}(x) - \nabla f(x))\right\|^2\\
        &\le& |\cS|\sum_{i \in \cS} \left\|\nabla f_{i}(x) - \nabla f(x)\right\|^2\\
        &\overset{\text{As.} \ref{asmp:grad_sim}}{\le} &|S|^2\zeta^2 \le \tau^2\zeta^2.
    \end{eqnarray*}
    Next, we bound the term with $\nu^2$. We have
    \begin{eqnarray*}
        \nu^2 &=& \sum_{t=0}^{T-1}\left\|\sum_{j=\pi_t}^{t-1}(\nabla f_{i_j}(x_{\pi_t}) - \nabla f(x_{\pi_t}))\right\|^2\\
        &\overset{\text{ Lemma } \ref{lem:lemmaA1}}{\le}& \sum_{t=0}^{T-1}\tau_t\sum_{j=\pi_t}^{t-1}\left\|\nabla f_{i_j}(x_{\pi_t}) - \nabla f(x_{\pi_t})\right\|^2\\
        &\overset{\text{Lemma } \ref{lem:lemmaA3}}{\le}& 
        \tau_{C}\tau_{\max}\zeta^2T.
    \end{eqnarray*}

Thus, the final rate is  
\begin{eqnarray*}
        \E{\|\nabla f(\hat{x}_t)\|^2} &\le& \cO\left(\frac{F_0}{\gamma T} 
        + L\gamma\sigma^2
        + L^2\gamma^2\left(\frac{1}{\lfloor T/\tau\rfloor}\sum_{k=1}^{\lfloor \frac{T}{\tau}\rfloor}\tau^2\zeta^2 
        + \frac{1}{T}\tau_C\tau_{\max}\zeta^2T\right)\right)\\
        &=& \cO\left(\frac{F_0}{\gamma T} 
        + L\gamma\sigma^2
        + L^2\gamma^2\tau^2\zeta^2 
        + L^2\gamma^2\tau_C\tau_{\max}\zeta^2\right).
    \end{eqnarray*}
    Using the stepsize restriction and the value of $\tau$, we get
    \begin{eqnarray*}
         \E{\|\nabla f(\hat{x}_t)\|^2} &\le& \cO\left(\frac{F_0}{\gamma T} 
         + L\gamma\sigma^2
         + \zeta^2\right).
    \end{eqnarray*}
    Now we need to tune the stepsize. If we choose $\gamma = \Theta\left(\min\left\{\frac{1}{L\sqrt{\tau_{\max}\tau_C}}, \left(\frac{F_0}{L\sigma^2 T}\right)^{1/2}\right\}\right)$, then we have two cases
    \begin{itemize}
        \item if $\gamma = \Theta\left(\frac{1}{L\sqrt{\tau_{\max}\tau_C}}\right)$, then 
        \begin{eqnarray*}
             \E{\|\nabla f(\hat{x}_t)\|^2} &\le& \cO\left(\frac{F_0}{T} \sqrt{L\tau_{\max}\tau_C}
         + L\sigma^2\left(\frac{F_0}{L\sigma^2T}\right)^{1/2}
         + \zeta^2\right)\\
         &=& \cO\left(\frac{LF_0\sqrt{\tau_{\max}\tau_C}}{T} 
         + \left(\frac{LF_0\sigma^2}{T}\right)^{1/2}
         + \zeta^2\right).
        \end{eqnarray*}
        
        \item $\gamma = \Theta\left(\left(\frac{F_0}{L\sigma^2 T}\right)^{1/2}\right),$ then 
        \begin{eqnarray*}
             \E{\|\nabla f(\hat{x}_t)\|^2} &\le& \cO\left(\frac{F_0}{T} \left(\frac{L\sigma^2T}{F_0}\right)^{1/2}
         + L\sigma^2\left(\frac{F_0}{L\sigma^2T}\right)^{1/2}
         + \zeta^2\right)\\
         &=& \cO\left(\left(\frac{LF_0\sigma^2}{T}\right)^{1/2}
         + \zeta^2\right).
        \end{eqnarray*}
    \end{itemize}
    It is left to choose the stepsize to be the minimum over two cases.
\end{proof}

\subsubsection{Mini-batch SGD}
First, we show that the iterates of mini-batch SGD suit the update rule \eqref{eq:real_iterates_update}. The standard iteration of mini-batch SGD has the following form
\begin{equation*}
    z_{q+1} = z_q - \frac{\gamma}{b}\sum_{i\in B_q}\nabla f_i(z_q).\footnote{Here we denote iterates by $z_q$ in order not to confuse them with the notation of the paper}
\end{equation*}
Let $B_q = \{i_{q,1}, \dots, i_{q,b}\}$ be the batch at iteration $q,$ i.e. $i_{q,l} \in [n]$ for all $q$ and $l$, and $B_q$ is sampled from $[n]$ uniformly at random without replacement. We initialize $x_0 = z_0,$ and the stepsize $\wtilde{\gamma} = \frac{\gamma}{b}.$ Then let us consider the following chain of updates
\begin{align*}
    x_{qb+l} &= x_{qb+l-1} - \wtilde{\gamma} \nabla f_{i_{q,l}}(x_{qb}) \quad \text{for } l = 1, \dots, b.
\end{align*}
Then,
\begin{eqnarray}
    x_{qb+b} &=& x_{qb + b-1} - \wtilde{\gamma}\nabla f_{i_{k,b-1}}(x_{qb})\notag\\
    &\vdots&\notag\\
    &=& x_{qb} - \wtilde{\gamma}\sum_{i \in B_q}\nabla f_i(x_{qb}),\label{eq:mini-batch-SGD-update}
\end{eqnarray}
which is exactly the mini-batch update. Besides, we have that $x_{qb} \equiv z_q.$ 

Let us give now the values for $\tau_{\max}$ and $\tau_C.$ At the beginning of each step of mini-batch SGD the server selects uniformly at random $b$ workers and sends them current model $x_{qb}$. Then, workers send one-by-one evaluated gradients while the server does not assign new jobs until the last worker finishes its job. Observing \eqref{eq:mini-batch-SGD-update} we conclude that the maximum delay is $\tau_{\max} = b - 1 <b$ for the slowest worker while $\tau_C = b-1 < b$ which is attained at the beginning of the batch. 

Now we apply Theorem~\ref{th:theorem3}. For that, we need $\tau$ to be a multiple of $b$. This is achieved by carefully choosing the stepsize and $\tau$ restrictions.

\begin{propositionsec}\label{prop:minibatch_sgd}
Let Assumptions~\ref{asmp:smoothness} and \ref{asmp:grad_sim} hold. Let the stepsize $\gamma$ satisfy $20L\gamma \le 1.$ Let $\tau = b\lfloor\frac{1}{20L\gamma}\rfloor$. Then the iterates of mini-batch SGD with batch size $b$ satisfy
\begin{equation}\label{prop:prop-mini-batch}
    \E{\|\nabla f(\hat{x}_{t})\|^2} \le \cO\left(\frac{F_0}{\gamma T} + \frac{L\gamma\zeta^2}{b}\right),
\end{equation}
where $\hat{x}_t$ is chosen uniformly at random from $\{x_0, \dots, x_{T-1}\}$ and $F_0 \eqdef f(x_0) - f^*.$ Moreover, if we tune the stepsize, then the iterates of mini-batch SGD with batch size $b$ satisfy
\begin{equation}\label{prop:prop-mini-batch-tuned}
    \E{\|\nabla f(\hat{x}_{t})\|^2} \le \cO\left(\frac{LF_0}{T} + \sqrt{\frac{LF_0\zeta^2}{Tb}}\right).
\end{equation}
\end{propositionsec}
\begin{proof}

    We run mini-bathc SGD from Algorithm~\ref{alg:pseudocode} point of view with stepsize $\wtilde{\gamma} = \frac{\gamma}{b}.$ Then we should choose $\tau$ such that $20L\wtilde{\gamma}\tau \le 1.$ So we can choose $\tau = b\lfloor\frac{1}{20L\gamma}\rfloor.$ Now we have two restricitons on $\wtilde{\gamma}$
    \begin{eqnarray*}
        \wtilde{\gamma} &\le& \frac{1}{6L} \Rightarrow \gamma \le \frac{b}{6L},\\
        \wtilde{\gamma} &\le& \frac{1}{20L\sqrt{\tau_{\max}\tau_C}} \Rightarrow \gamma \le \frac{1}{20L}.
    \end{eqnarray*}
    That is why we should assume that $\gamma \le \frac{1}{20L}$ with previously chosen $\tau.$

    With the choice $\tau$ we have, every chunk of size $\tau$ consists of several full batches (each chunk has the same number of batches). This choice of $\tau$ is needed in order to apply conditional expectation correctly. Indeed, let us take the conditional expectation depending on all the events before iteration $(k-1)\tau.$ This means that $x_{(k-1)\tau}$ is fixed as well since it is computed using the gradients before iteration $(k-1)\tau.$

    Let us assume that each period of size $\tau$ consists of $m \ge 1$ batches. Then we can compute the conditional expectation from the definition of $\sigma_{k,\tau}^2$ splitting it into $m$ independent conditional expectations for each batch. Note that in this case, the maximum over $j$ in \eqref{eq:received_seq_cor_variance} is always attained for $j=\tau-1$. Thus, we have
    \begin{eqnarray*}
        \sigma_{k,\tau}^2 &=& \EE_{(k-1)\tau}\left\|\sum_{j=k\tau}^{k\tau+\tau-1}(\nabla f_{i_j}(x_{(k-1)\tau}) - \nabla f(x_{(k-1)\tau})\right\|^2\\
        &=&m\cdot b^2\frac{\zeta^2}{b} = mb\zeta^2 = \tau\zeta^2.
    \end{eqnarray*}
    where in the last step we use the variance bound of the mini-batch estimator. Note that in one node setting  $\zeta^2$ plays the role of standard variance bound of stochastic gradient estimator.

    Taking conditional expectations one by one we can bound each $\sigma_{k,\tau}^2$ term by $\tau\zeta^2.$ 

    Now we switch to the $\nu^2$ term. Note that in this case, the delayed iterate for gradients within one batch is the same, i.e. for $j\in [qb, (q+1)b)$ we have $\pi_j = qb,$ thus we can compute expectations correctly. We split $\nu^2$ is split into $\frac{T}{b}$ terms $\nu^2_q$ of the form
    \begin{eqnarray*}
    \nu^2_q &\eqdef& \sum_{t=qb}^{(q+1)b-1}\EE_{qb}\left[\left\|\sum_{j=\pi_t}^{t-1}(\nabla f_{i_j}(x_{\pi_j}) - \nabla f(x_{\pi_j}))\right\|^2\right]\\
    &=& \sum_{t=qb}^{(q+1)b-1}\EE_{qb}\left[\left\|\sum_{j=qb}^{t-1}(\nabla f_{i_j}(x_{qb}) - \nabla f(x_{qb}))\right\|^2\right]\\
    &\le& \sum_{t=qb}^{(q+1)b-1}(t-qb)^2\frac{\zeta^2}{(t-qb)}\\
    &=& \sum_{t=qb}^{(q+1)b-1}(t-qb)\zeta^2\\
    &\le& \frac{b^2}{2}\zeta^2.
    \end{eqnarray*}
    Here we again need to take conditional expectation w.r.t $x_{qb}$, and afterwards compute bounds for $\nu_q^2$ one by one. Saying that, we have $\nu^2 \le \frac{T}{b}\frac{b^2}{2}\zeta^2 \le Tb\zeta^2.$
    Now we apply Theorem~\ref{th:theorem3}.
    \begin{eqnarray}
        \E{\|\nabla f(\hat{x}_{t})\|^2} &\le& \cO\left(\frac{F_0}{\wtilde{\gamma} T} + L^2\wtilde{\gamma}^2\tau\zeta^2 + L^2\wtilde{\gamma}^2b\zeta^2\right)\notag\\
        &\le& \cO\left(\frac{F_0}{\gamma \frac{T}{b}} + \frac{L\gamma\zeta^2}{b}\right),\label{prop:propB2_1}
    \end{eqnarray}
    where we use $L\wtilde{\gamma}\tau = \Theta(1),$ $\wtilde{\gamma} = \frac{\gamma}{b},$ and $\gamma\le \frac{1}{20L}.$ Note that $T$ iterations of Algorithm~\ref{alg:pseudocode} are equivalent to $\frac{T}{b}$ iterations of mini-batch SGD. That is why the right-hand side of \eqref{prop:propB2_1} is a standard rate of mini-batch SGD. However, we observe that the left-hand side of \eqref{prop:propB2_1} is slightly different from what we expect; there we get convergence for all intermediate iterates as well. In order to get the standard rate we need to modify restarting virtual iterates in the following way
    \[
    \wtilde{x}_{t+1} = 
    \begin{cases}
        \wtilde{x}_t - \gamma\nabla f(x_{\pi_t}) & \text{ if } t+1 \ne \tau k \text{ for any } k \ge 1, \\
        x_{t+1} & \text{ if } t+1 = \tau k \text{ for some } k \ge 1,
\end{cases}
    \]
    then the left-hand side of \eqref{prop:propB2_1} will be transformed to 
    \[
    \frac{1}{T}\sum_{q=0}^{T/b-1}\sum_{l=0}^{b-1}\E{\|\nabla f(x_{\pi_{qb+l})}\|^2} = \frac{1}{T}\sum_{q=0}^{T/b-1}\sum_{l=0}^{b-1}\E{\|\nabla f(x_{qb})\|^2} = \frac{1}{T/b}\sum_{q=0}^{T/b-1}\E{\|\nabla f(x_{qb})\|^2},
    \]
    which is exactly what we want. We only need to change $T$ to $Tb$ if we want to present the rate w.r.t. the number of mini-batch SGD steps.

    Now, if we choose stepsize $\gamma = \Theta\left(\min\left\{\frac{1}{L}, \sqrt{\frac{F_0b}{LT\zeta^2}}\right\}\right)$, then we have two possible options
    \begin{itemize}
        \item if $\gamma = \Theta(\frac{1}{L})$ we have 
        \begin{eqnarray*}
            \frac{1}{T}\sum_{t=0}^{T-1}\E{\|\nabla f(x_t)\|^2} &\le& \cO\left(\frac{F_0}{\frac{1}{L}T} + \frac{L\zeta^2}{b}\sqrt{\frac{F_0b}{LT\zeta^2}}\right)\\
            &\le& \cO\left(\frac{LF_0}{T} + \sqrt{\frac{LF_0\zeta^2}{Tb}}\right).
        \end{eqnarray*}
        \item if $\gamma = \Theta\left(\sqrt{\frac{F_0b}{LT\zeta^2}}\right)$ we have
        \begin{eqnarray*}
            \frac{1}{T}\sum_{t=0}^{T-1}\E{\|\nabla f(x_t)\|^2} &\le& \cO\left(\frac{F_0}{T}\sqrt{\frac{LT\zeta^2}{F_0b}} + \frac{L\zeta^2}{b}\sqrt{\frac{F_0b}{LT\zeta^2}}\right)\\
            &\le& \cO\left(\sqrt{\frac{LF_0\zeta^2}{Tb}}\right).
        \end{eqnarray*}
    \end{itemize}
    It is left to set the stepsize to be the minimum over two cases.
\end{proof}

\subsubsection{Pure Asynchronous SGD with waiting}

\begin{algorithm*}[t]
\caption{Pure Asynchronous SGD with waiting}
\label{alg:pure_asynchronous_waiting}
\begin{algorithmic}[1]
\State \textbf{Input:} $x_0\in \R^{d}$, stepsize $\gamma > 0$, set of assigned jobs $\cA_0 = \empty$, set of received jobs $\cR_{0} = \empty$, batch size $b \ge 1$, gradient estimator $g = 0$ and number of received gradients $r = 0$
\State \textbf{Initialization:} for all jobs $(i, 0) \in \cA_1$, server assigns worker $i$ to compute a stochastic gradient $g_i(x_0)$ 
    \For{$t = 0,1,2,\dots, T-1$}
        \State server sets $g_{t, 0} = 0$ and $\cR_{t,0} = \cR_t$
        \For{$j = 1, \dots, b$}
            \State once worker $i_{t,j}$ finishes a job $(i_{t,j},\pi_{t,j}) \in \cA_{t+1}$, it sends $g_{i_{t,j}}(x_{\pi_{t,j}})$ to the server
            \State server updates $g_{t,j} = g_{t,j-1} + g_{i_{t,j}}(x_{\pi_{t,j}})$
            \State server updates the set $\cR_{t, j} = \cR_{t, j-1} \cup \{(i_{t,j}, \pi_{t,j})\}$
        \EndFor
        \State server updates the current model $x_{t+1} = x_t - \frac{\gamma}{b} g_{t, b}$  and set $\cR_{t+1} = \cR_{t,b}$
        \State server assigns worker $i_{t, j}$ to compute a gradient $g_{i_{t,j}}(x_{t+1})$ for all $j\in[b]$
        \State server updates the set $\cA_{t+2} = \cA_{t+1} \cup \{(i_{t,1}, t+1)\} \cup \dots \cup \{(i_{t,b}, t+1)\}$
    \EndFor
\end{algorithmic}	
\end{algorithm*}

This case is similar to the previous one; the only change is that the server waits for the first $b$ fastest workers and assigns them new jobs back. Formally, the update has the following form
\begin{equation}\label{eq:pure_async_sgd_waiting}
z_{q+1} = z_q - \gamma\sum_{i \in B_q}g_i(z_{\pi_{q,i}}),
\end{equation}
where $\pi_{q,i}$ is the iteration counter where worker $i$ evaluated its gradient (possible delayed).\footnote{Here we again use $z_t$ iterates notation in order not to confound with the updates of Algorithm~\ref{alg:pseudocode}}. Let $B_q \eqdef \{i_{q,1}, \dots, i_{q,b}\}$ be the set of $b$ fastest workers at iteration $q$. Then, we can rewrite this update in the following form with $x_0 = z_0$.

\begin{eqnarray*}
    x_{qb+b} &=& x_{qb+b-1} - \gamma g_{i_{q,b-1}}(x_{qb-\tau_{q,b-1}} )\\
    &\vdots&\\
    &=& x_{qb} - \gamma\sum_{l=0}^{b-1} g_{i_{q,l}}(x_{qb-\tau_{q,l}}).
\end{eqnarray*}
We extend one iteration of \eqref{eq:pure_async_sgd_waiting} to $b$ intermediate iterations of Algorithm~\ref{alg:pseudocode}, i.e. the total number of iterations increases by a factor $b$. We also highlight the fact that the server does not assign new jobs before all the workers from the batch send the gradients. This means that workers always send gradients at points $x_{qb}$. Note that $\tau_C$ remains the same while $\tau_{\max}$ might increase $b$ times in the worst case. We are ready to apply Theorem~\ref{th:theorem3}.

\begin{propositionsec}\label{prop:pure_async_sgd_waiting}
    Let Assumptions~\ref{asmp:smoothness}, \ref{asmp:bound_var}, and  \ref{asmp:grad_sim} hold. Let the stepsize $\gamma$ satisfies $20L\sqrt{b\tau_{\max}\tau_C} \gamma \le b$ and $6L\gamma\le 1$. Then the iterates of Algorithm~\ref{alg:pure_asynchronous_waiting} satisfy 
    \begin{align}
        \E{\|\nabla f(\hat{x}_{t})\|^2} 
        \le \cO\left(\frac{F_0}{\gamma T} 
        + L\gamma\frac{\sigma^2}{b}
        + \zeta^2\right),
    \end{align}
    where $\hat{x}_t$ is chosen uniformly at random from $\{x_0, \dots, x_{b-1}, x_b, \dots, x_{2b-1}, \dots, x_{Tb-1}\}$ and $F_0 = f(x_0) - f^*$. Moreover, if we tune the stepsize, then the iterates of pure asynchronous SGD with waiting satisfy 
    \begin{align}
        \E{\|\nabla f(\hat{x}_{t})\|^2} 
        \le \cO\left(\frac{F_0\sqrt{\tau_{\max}\tau_C}}{ T\sqrt{b}}
        + \left(\frac{LF_0\sigma^2}{Tb}\right)^{1/2}
        + \zeta^2\right).
    \end{align}
\end{propositionsec}
\begin{proof}
    We again need to find restrictions on the stepsize. We have 
    \begin{eqnarray*}
        6L\wtilde{\gamma} \le 1 &\Rightarrow &\gamma \le \frac{b}{6L},\\
        20L\wtilde{\gamma}\sqrt{b\tau_{\max}\tau_C} \le 1 &\Rightarrow& \gamma \le \frac{b}{20L\sqrt{b\tau_{\max}\tau_C}}.
    \end{eqnarray*}
    Then we choose $\tau = \lfloor\frac{1}{20L\gamma} \rfloor,$ and similarly to the case of pure asynchronous SGD we bound $\sigma^2_{k,\tau} \le \tau^2\zeta^2$ while $\nu^2 \le b\tau_C\tau_{\max}\zeta^2T.$ Then the rate is
    \begin{align*}
        \frac{1}{Tb}\sum_{q=0}^{T-1}\sum_{l=0}^{b-1}\E{\|\nabla f(x_{qb+l})\|^2} &\le \cO\left(\frac{F_0}{\wtilde{\gamma} Tb} 
        + L\wtilde{\gamma}\sigma^2
        + L^2\wtilde{\gamma}^2\tau^2\zeta^2 
        + L^2\wtilde{\gamma}^2b\tau_C\tau_{\max}\zeta^2\right)\\
        &\le \cO\left(\frac{F_0}{\gamma T} 
        + L\gamma\frac{\sigma^2}{b}
        + \zeta^2\right).
    \end{align*}
    It is left to set the stepsize choice in a similar way as for pure asynchronous SGD. 
\end{proof}

\subsubsection{SGD with Random Reshuffling}

In the beginning of each epoch the server sample a random permutation $\chi_q$ of $[n]$, and then gradients come following that order. We initialize $x_0 = z_0$. The chain of updates in this case has the form as follows
\begin{eqnarray*}
    x_{qn+n} &=& x_{qn+n-1} - \gamma\nabla f_{\chi_q(n-1)} (x_{qn+n-1})\\
    &\vdots& \\
    &=& x_{qn} - \gamma \sum_{l=0}^{n-1}\nabla f_{\chi_q(l)}(x_{qn+l}).
\end{eqnarray*}
Note that in this case, the server receives gradients from the workers one by one without delays, i.e. $\tau_{\max} = 0$. This means that we have the only stepsize restriction $\gamma \le \frac{1}{6L}$ since the other one $20L\gamma\sqrt{\tau_{\max}\tau_C} \le 1$ holds for any choice of $\gamma$. 

Moreover, since we take expectation only once at the end of the proof, we cover the case of SGD shuffle as well without any additional changes. Now we apply Theorem~\ref{th:theorem3} to derive convergence guarantees in this case.

\begin{propositionsec}\label{prop:random_reshuffle} Let Assumptions \ref{asmp:smoothness}, \ref{asmp:grad_sim} hold. Let the the stepsize $\gamma$ satisfy $20nL\gamma \le 1.$ Let $\tau = n\lfloor\frac{1}{20Ln\gamma} \rfloor.$ Then the iterates of SGD with random reshuffling satisfy
\begin{equation}
    \E{\|\nabla f(\hat{x}_{t})\|^2 } \le \cO\left(\frac{F_0}{\gamma T} + L^2n\gamma^2\zeta^2\right),
\end{equation}
where $\hat{x}_t$ is chosen uniformly at random from $\{x_0, \dots, x_{n-1}, x_n, \dots, x_{2n-1}, \dots, x_{T-1}\}$ and $F_0 \eqdef f(x_0) - f^*,$ and $T$ is the total number of gradient evaluations. If we tune the stepsize, then the iterates of SGD with random reshuffling satisfy
\begin{eqnarray}
            \E{\|\nabla f(\hat{x}_t)\|^2}
            \le \cO\left(\frac{LF_0n}{T} + \left(\frac{LF_0\sqrt{n}\zeta}{T}\right)^{2/3}\right).
\end{eqnarray}
\begin{proof}
    We need $\tau$ to be a multiple of $n$ in order to be able to correctly apply conditional expectations (the same trick as for mini-batch SGD). We force the stepsize $\gamma \le \frac{1}{20Ln}.$ Then we automatically satisfy the restriction $6L\gamma \le 1.$ Now we need to choose $\tau$. We need $20L\gamma \tau \le 1,$ so let us choose $\tau = n\lfloor\frac{1}{20Ln\gamma}\rfloor,$ then $20L\gamma\tau \le 20L\gamma n\frac{1}{20Ln\gamma} \le 1.$

    We use the computations from \cite{koloskova2022sharper}, Section C.$2$. Thus, we have $\sigma_{k,\tau}^2 \le \min\{\tau, n\}\zeta^2.$ The fact that $\tau \ge n$ implies $\sigma_{k,\tau}^2 \le n\zeta^2.$ Since we do not have delays in this case, then $\nu^2 = 0.$ This gives us the following result applying Theorem~\ref{th:theorem3}.
    \[
    \frac{1}{T}\sum_{q=0}^{T/n-1}\sum_{l=0}^{n-1}\E{\|\nabla f(x_{qn+l})\|^2} \le \cO\left(\frac{F_0}{\gamma T} + L^2n\gamma^2\zeta^2\right).
    \]
    Now, if we need to tune the stepsize. Similarly to mini-batch SGD we should consider two cases.
    \begin{itemize}
        \item if $\gamma = \Theta(\frac{1}{Ln})$ we have 
        \begin{eqnarray*}
            \frac{1}{T}\sum_{t=0}^{T-1}\E{\|\nabla f(x_t)\|^2} &\le& \cO\left(\frac{F_0}{\frac{1}{Ln}T} + L^2n\zeta^2\left(\frac{F_0}{L^2nT\zeta^2}\right)^{2/3}\right)\\
            &\le& \cO\left(\frac{LF_0n}{T} + \left(\frac{L^2F_0^2n\zeta^2}{T^2}\right)^{1/3}\right).
        \end{eqnarray*}
        \item if $\gamma = \Theta\left(\left(\frac{F_0}{L^2nT\zeta^2}\right)^{1/3}\right)$ we have
        \begin{eqnarray*}
            \frac{1}{T}\sum_{t=0}^{T-1}\E{\|\nabla f(x_t)\|^2} &\le& \cO\left(\frac{F_0}{T}\left(\frac{L^2nT\zeta^2}{F_0}\right)^{1/3} + L^2n\zeta^2\left(\frac{F_0}{L^2nT\zeta^2}\right)^{2/3}\right)\\
            &\le& \cO\left(\left(\frac{L^2F_0^2n\zeta^2}{T^2}\right)^{1/3}\right).
        \end{eqnarray*}
    \end{itemize}
    To get the final rate after tunning we need to set the stepsize as the minimum over two cases.
\end{proof}

\end{propositionsec}

\section{Proofs for Analysis of Gradient Assigning Process}\label{sec:proof_theorem4}

In this section the iteration counter starts with $1$, i.e. we consider $y_1$ as an initial point. Following~\eqref{eq:virtual_iterates_update} and \eqref{eq:virtual_virtual_seq}, we have
\[
y_t = y_{r(t)} - \gamma \sum_{j=r(t)}^{t-1}g_{k_j}(x_{\alpha_j}), \quad \wtilde y_t = y_{r(t)} - \gamma\sum_{j=r(t)}^{t-1} \nabla f(x_j),
\]
where $g_{k_j}(x_{\alpha_j})$ is an unbiased stochastic estimator of $\nabla f_{k_j}(x_{\alpha_j})$ with variance bounded by $\sigma^2$. 

Analogously to the proofs of Theorem~\ref{th:theorem4} we define
\begin{align*}
\wtilde{\Delta}_t^m &= \sum_{j=r(t)}^{m}\nabla f(x_j) - g_{k_j}(x_{\alpha_j}), \quad \wtilde{\phi}_t^m(x) = \left\|\sum_{j=r(t)}^m\nabla f_{k_j}(x)-\nabla f(x)\right\|^2.
\end{align*}
In this section we use shortcuts 
\begin{eqnarray*}
\wtilde{A}  &\eqdef& \sum_{t=0}^{T-1}
 \E{\|\wtilde{y}_t - y_{\alpha_t}\|^2}, \quad B \eqdef \sum_{t=0}^{T-1}\E{\|\nabla f(x_t)\|^2}, \quad 
\wtilde{\Phi} \eqdef \sum_{t=0}^{T-1}\E{\wtilde{\phi}_t^{t-1}(x_{r(t)})}\\
\wtilde{\Psi} &\eqdef& \sum_{t=0}^{T-1}\E{\left\|\sum_{j=\pi_t}^{t-1}\nabla f_{k_j}(x_{\alpha_j}) - \nabla f(x_{\alpha_j})\right\|^2}.
\end{eqnarray*}

\subsection{Key Lemmas}

\begin{lemmasec}\label{lem:x_t_y_t}
    Let real $\{x_t\}_{t=0}^T$ and virtual $\{y_t\}_{t=0}^T$ iterates be defined in \eqref{eq:real_iterates_update} and \eqref{eq:virtual_iterates_update} respectively. Then
    \begin{equation*}
        x_t - y_t = \gamma\sum_{(i,j) \in \cA_t\setminus \cR_t} g_i(x_j).
    \end{equation*}
\end{lemmasec}
\begin{proof}
    We prove the statement by induction. The base is trivial $x_0 - y_0 = 0$ by definition while $\cA_0 = \cR_0  = \empty$ as well. Let the statement hold at iteration $t,$ now we show that it is true at iteration $t+1$ as well. Indeed,
    \begin{eqnarray}
        x_{t+1} - y_{t+1} &=& (x_t - \nabla f_{i_t}(x_{\pi_t})) - (y_t - \gamma \nabla f_{k_t}(x_{\alpha_t}))\notag \\
        &\overset{\text{Base}}{=}& \gamma\sum_{(i,j)\in \cA_t\setminus\cR_t}\nabla f_i(x_j) + \gamma (\nabla f_{k_t}(x_{\alpha_t}) -  f_{i_t}(x_{\pi_t}).\label{eq:x_t_y_t}
    \end{eqnarray}
    Now we note that the connection between sets is updating as follows 
    \begin{eqnarray*}
        \cA_{t+1} \setminus \cR_{t+1} &=& (\cA_t \cup \{(k_t, \alpha_t)\}) \setminus (\cR_t \cup \{(i_t, \pi_t)\})\\
        &=& (\cA_t \setminus \cR_t) \cup \{(k_t, \alpha_t)\} \setminus  \{(i_t, \pi_t)\}.
    \end{eqnarray*}
    Plugging this into \eqref{eq:x_t_y_t}  we get $x_{t+1} - y_{t+1} = \gamma\sum_{(i,j)\in \cA_{t+1}\setminus\cR_{t+1}}\nabla f_i(x_j).$ 
\end{proof}

\begin{lemmasec}\label{lem:lemmaE1} At time step $t$ the following inequality holds
    \begin{align}\label{eq:lemmaE1_1}
        \E{\|y_t-x_t\|^2} \le \gamma^2(\tau_C-1)^2G^2 + (\tau_C-1)\gamma^2\sigma^2.
    \end{align}
\end{lemmasec}
\begin{proof}
    Loaded with Lemma~\ref{lem:x_t_y_t}
    \begin{eqnarray*}
        \E{\|y_t-x_t\|^2} &=& \gamma^2\E{\left\|\sum_{(i,j)\in \cA_t\setminus\cR_t}  g_i(x_{j})\right\|^2} \le \gamma^2\E{\left\|\sum_{(i,j)\in \cA_t\setminus\cR_t}  \nabla f_i(x_{j})\right\|^2} + \gamma^2\sigma^2(\tau_C-1)\\
        &\le& \gamma^2(\tau_C-1)\sum_{(i,j)\in \cA_t\setminus\cR_t} \E{\| \nabla f_i(x_{j})\|^2} + \gamma^2\sigma^2(\tau_C-1)\\
        &\le& \gamma^2(\tau_C-1)^2G^2 + \gamma^2\sigma^2(\tau_C-1),
    \end{eqnarray*}
    since $|\cA_{t+1}| = |\cA_t| + 1.$

\end{proof}

\begin{lemmasec}\label{lem:lemmaE2} If $30L\gamma\tau \le 1$, then
    \begin{align}
        \E{\|\wtilde{\Delta}_t^m\|^2} &\le 
        4\E{\wtilde{\phi}_t^m(x_{r(t)})}
        + \frac{48}{7}L\gamma\sum_{j=r(t)}^{m}\E{\wtilde{\phi}_t^{j-1}(x_{r(t)})}
        + \frac{6}{71}(\tau_C-1)^2G^2
        + \frac{6}{71}(\tau_C-1)\sigma^2\notag\\
        &\quad
        + \frac{90}{7}L^2\tau\sum_{j=r(t)}^m\E{\|y_{j}-y_{\alpha_j}\|^2}
        + \frac{4}{71}\tau\sum_{j=r(t)}^{m}\E{\|\nabla f(x_j)\|^2}
        + \frac{2}{35}\tau\sigma^2
        . \label{eq:lemmaE2_1}
    \end{align}

\end{lemmasec}
\begin{proof}
Using Young's inequality we have
    \begin{align*}
        &\E{\|\wtilde{\Delta}_t^m\|^2} \le \E{\left\|\sum_{j=r(t)}^m \nabla f_{k_j}(x_{\alpha_j}) - \nabla f(x_j)\right\|^2} + \tau\sigma^2\\
        &= \E{\left\|\sum_{j=r(t)}^m \nabla f_{k_j}(x_{\alpha_j})\pm \nabla f_{k_j}(x_j) \pm \nabla f_{k_j}(x_{r(t)}) \pm \nabla f(x_{r(t)})- \nabla f(x_j)\right\|^2} 
        + \tau\sigma^2\\
        &\le 
        4\E{\left\|\sum_{j=r(t)}^m  \nabla f_{k_j}(x_{\alpha_j})- \nabla f_{k_j}(x_j)\right\|^2}
        + 4\E{\left\|\sum_{j=r(t)}^m \nabla f_{k_j}(x_j) - \nabla f_{k_j}(x_{r(t)}) \right\|^2}\\
        &\quad 
        + 4\E{\left\|\sum_{j=r(t)}^m  \nabla f_{k_j}(x_{r(t)}) - \nabla f(x_{r(t)})\right\|^2}
        + 4\E{\left\|\sum_{j=r(t)}^m  \nabla f(x_{r(t)})- \nabla f(x_j)\right\|^2} 
        + \tau\sigma^2
\end{align*}
We continue as follows
\begin{eqnarray*}
        \E{\|\wtilde{\Delta}_t^m\|^2}
        &\le& 4\E{\wtilde{\phi}_t^t(x_{r(t)})}
        + 4L^2\tau\sum_{j=r(t)}^{m}\E{\|x_j-x_{\alpha_j}\|^2}
        + 8L^2\tau\sum_{j=r(t)}^m\E{\|x_j-x_{r(t)}\|^2}\\
        && \;+\; \tau\sigma^2\\
        &\le& 
        4\E{\wtilde{\phi}_t^m(x_{r(t)}) }
        + 12L^2\tau\sum_{j=r(t)}^m\E{\|x_j-y_j\|^2+\|y_j-y_{\alpha_j}\|^2+\|y_{\alpha_j}-x_{\alpha_j}\|^2} \\
        &&  \;+\;
        24L^2\tau\sum_{j=r(t)}^m\E{\|x_j-y_j\|^2 + \|y_j-y_{r(t)}\|^2+\|y_{r(t)}-x_{r(t)}\|^2} 
        + \tau\sigma^2.
    \end{eqnarray*}
    For the terms of the form $\|x_a-y_a\|^2$ for some $a$ we use Lemma~\ref{lem:lemmaD1}. Using the above and the stepsize restriction, we have
    \begin{eqnarray}\label{eq:lemmaE2_2}
        \E{\|\wtilde{\Delta}_t^m\|^2 }&\le& 
         4\E{\wtilde{\phi}_t^m(x_{r(t)})} 
         + \frac{2}{25}(\tau_C-1)^2G^2+12L^2\tau\sum_{j=r(t)}^m\E{\|y_{j}-y_{\alpha_j}\|^2} 
         + \tau\sigma^2 \notag \\
         &&  \;+\; 24L^2\tau\sum_{j=r(t)}^m\E{\|y_{r(t)}-y_j\|^2}
         + \frac{2}{25}(\tau_C-1)\sigma^2,
    \end{eqnarray}
    For the last term, we have
    \begin{align*}
        &\sum_{j=r(t)}^m\E{\|y_{r(t)}-y_j\|^2} = \sum_{j=r(t)}^m\gamma^2\E{\left\|\sum_{l=r(t)}^{j-1}g_{k_l}(x_{\alpha_l})\right\|^2}\\
        &\le 
        2\gamma^2\sum_{j=r(t)}^m\E{\left\|\sum_{l=r(t)}^{j-1} g_{k_l}(x_{\alpha_l}) - \nabla f(x_l)\right\|^2}
        + 2\gamma^2\sum_{j=r(t)}^m\E{\left\|\sum_{l=r(t)}^{j-1}\nabla f(x_l)\right\|^2}\\
        &= 
        2\gamma^2\sum_{j=r(t)}^m\E{\|\wtilde{\Delta}_t^{j-1}\|^2}
        + 2\gamma^2\tau\sum_{j=r(t)}^m\sum_{l=r(t)}^{j-1}\E{\|\nabla f(x_l)\|^2}\\
        &\overset{\eqref{eq:lemmaE2_2}}{\le} 
        8\gamma^2\sum_{j=r(t)}^m\E{\wtilde{\phi}_t^{j-1}(x_{r(t)})}
        + \frac{4}{25}\gamma^2(\tau_C-1)^2G^2\tau+ 24\gamma^2L^2\tau \sum_{j=r(t)}^{m}\E{\sum_{l=r(t)}^{j-1}\|y_{l}-y_{\alpha_l}\|^2}\\
        &\quad 
        + 48\gamma^2L^2\tau \sum_{j=r(t)}^{m}\sum_{l=r(t)}^{j-1}\E{\|y_{r(t)}-y_l\|^2}
        + 2\gamma^2\tau\sum_{j=r(t)}^m\sum_{l=r(t)}^{j-1}\E{\|\nabla f(x_l)\|^2} 
        + 2\gamma^2\tau^2\sigma^2\\
        &\quad 
        + \frac{4}{25}\gamma^2(\tau_C-1)\sigma^2\\
        &\le 
        8\gamma^2\sum_{j=r(t)}^m\wtilde{\phi}_t^{j-1}(x_{r(t)}) 
        + \frac{4}{25}\gamma^2(\tau_C-1)^2G^2\tau + 24\gamma^2L^2\tau^2 \sum_{j=r(t)}^{m}\|y_{j}-y_{\alpha_j}\|^2\\
        &\quad 
        + 48\gamma^2L^2\tau^2 \sum_{j=r(t)}^{m}\|y_{r(t)}-y_j\|^2
        + 2\gamma^2\tau^2\sum_{j=r(t)}^m\|\nabla f(x_j)\|^2
        + 2\gamma^2\tau^2\sigma^2
        + \frac{4}{25}\gamma^2(\tau_C-1)\sigma^2\\
        &\le 8\gamma^2\sum_{j=r(t)}^m\wtilde{\phi}_t^{j-1}(x_{r(t)}) + \frac{4}{25}\gamma^2(\tau_C-1)^2G^2\tau + \frac{2}{75}\sum_{j=r(t)}^{m}\|y_{j}-y_{\alpha_j}\|^2\\
        &\quad 
        + \frac{4}{75}\sum_{j=r(t)}^{m}\|y_{r(t)}-y_j\|^2
        + 2\gamma^2\tau^2\sum_{j=r(t)}^m\|\nabla f(x_j)\|^2
        + 2\gamma^2\tau^2\sigma^2
        + \frac{4}{25}\gamma^2(\tau_C-1)\sigma^2.
    \end{align*}
    Hence, 
    \begin{align}
        \sum_{j=r(t)}^m\|y_{r(t)}-y_j\|^2  &\le 
        \frac{60}{7}\gamma^2\sum_{j=r(t)}^m\wtilde{\phi}_t^{j-1}(x_{r(t)}) 
        + \frac{12}{71}\gamma^2(\tau_C-1)^2G^2\tau + \frac{2}{71}\sum_{j=r(t)}^m\|y_j-y_{\alpha_j}\|^2\notag\\
        &
        + \frac{150}{71}\gamma^2\tau^2\sum_{j=r(t)}^m\|\nabla f(x_j)\|^2
        + \frac{150}{71}\gamma^2\tau^2\sigma^2
        + \frac{12}{71}\gamma^2(\tau_C-1)\sigma^2.\label{eq:lemmaE2_3}
    \end{align}
    Plugging \eqref{eq:lemmaE2_3} in \eqref{eq:lemmaE2_2} we get
    \begin{align*}
        \E{\|\wtilde{\Delta}_t^m\|^2} &\le 
        4\E{\wtilde{\phi}_t^m(x_{r(t)})}
        + \frac{2}{25}(\tau_C-1)^2G^2+12L^2\tau\sum_{j=r(t)}^m\E{\|y_{j}-y_{\alpha_j}\|^2}\\
        &\quad 
        + 24L^2\tau\left(\frac{60}{7}\gamma^2\sum_{j=r(t)}^m\E{\wtilde{\phi}_t^{j-1}(x_{r(t)})}
        + \frac{12}{71}\gamma^2(\tau_C-1)^2G^2\tau 
        + \frac{2}{71}\sum_{j=r(t)}^m\E{\|y_j-y_{\alpha_j}\|^2}\right.\\
        &\quad \left.
        + \frac{150}{71}\gamma^2\tau^2\sum_{j=r(t)}^m\E{\|\nabla f(x_j)\|^2}
        + \frac{150}{71}\gamma^2\tau^2\sigma^2
        + \frac{12}{71}\gamma^2(\tau_C-1)\sigma^2\right)\\
        &= 
        4\E{\wtilde{\phi}_t^m(x_{r(t)})}
        +\frac{48}{7}L\gamma\sum_{j=r(t)}^{m}\E{\wtilde{\phi}_t^{j-1}(x_{r(t)})} 
        + \frac{6}{71}(\tau_C-1)^2G^2 \\
        &\quad 
        + \frac{90}{7}L^2\tau\sum_{j=r(t)}^m\E{\|y_{j}-y_{\alpha_j}\|^2}
        + \frac{4}{71}\tau\sum_{j=r(t)}^{m}\|\nabla f(x_j)\|^2 
        + \frac{2}{35}\tau\sigma^2
        + \frac{6}{71}(\tau_C-1)\sigma^2,
    \end{align*}
    where we again use the fact that $L\gamma\tau \le \frac{1}{30}.$
\end{proof}

\begin{lemmasec}\label{lem:lemmaE3} If $30L\gamma\tau \le 1$ and $30L\gamma\tau_C \le 1$, then
    \begin{align}
    \sum_{t=1}^{T}\E{\|y_t-\wtilde y_t\|^2} &\le 
    \frac{44}{7}\gamma^2\sum_{t=1}^{T}\E{\wtilde{\phi}_t^{t-1}(x_{r(t)})} + \frac{6}{71}\gamma^2(\tau_C-1)^2G^2T 
    + \frac{1}{70}\sum_{t=1}^{T}\E{\|y_t-y_{\alpha_t}\|^2}\notag\\
    &\quad 
    + \frac{1}{15975L^2}\sum_{t=1}^{T}\E{\|\nabla f(x_t)\|^2}
    + \frac{2}{35}\gamma^2\tau \sigma^2T
    + \frac{6}{71}\gamma^2(\tau_C-1)\sigma^2.\label{eq:lemmaE3_1}
    \end{align}
\end{lemmasec}
\begin{proof}
Using $L\gamma\tau \le \frac{1}{30}$
    \begin{align*}
        &\sum_{t=1}^{T}\E{\|y_t-\wtilde y_t\|^2} = 
        \gamma^2\sum_{t=1}^{T}\E{\left\|\sum_{j=r(t)}^{t-1}g_{k_j}(x_{\alpha_j})- \nabla f(x_j)\right\|^2} 
        = \gamma^2\sum_{t=1}^{T}\E{\|\wtilde{\Delta}_t^{t-1}\|^2}\\
        &\overset{\eqref{eq:lemmaE2_1}}{\le} 
        4\gamma^2\sum_{t=1}^{T}\E{\wtilde{\phi}_t^{t-1}(x_{r(t)})}
        + \frac{6}{71}\gamma^2(\tau_C-1)^2G^2T 
        + \frac{48}{7}L\gamma^3\sum_{t=1}^{T}\sum_{j=r(t)}^{t-1}\E{\wtilde{\phi}_t^{j-1}(x_{r(t)})}\\
        &\quad
        +\frac{90}{7}L^2\gamma^2\tau\sum_{t=1}^{T}\sum_{j=r(t)}^m\E{\|y_j-y_{\alpha_j}\|^2}
        + \frac{4\gamma^2\tau}{71}\sum_{t=1}^{T}\sum_{j=r(t)}^{t-1}\E{\|\nabla f(x_j)\|^2}
        + \frac{2}{35}\gamma^2\tau \sigma^2T\\
        &\quad 
        + \frac{6}{71}\gamma^2(\tau_C-1)\sigma^2\\
        &\le \frac{44}{7}\gamma^2\sum_{t=1}^{T}\E{\wtilde{\phi}_t^{t-1}(x_{r(t)})}
        + \frac{6}{71}\gamma^2(\tau_C-1)^2G^2T 
        +\frac{90}{7}L^2\gamma^2\tau\sum_{t=1}^{T}\sum_{j=r(t)}^m\E{\|y_j-y_{\alpha_j}\|^2}\\
        &\quad 
        + \frac{4\gamma^2\tau}{71}\sum_{t=1}^{T}\sum_{j=r(t)}^{t-1}\E{\|\nabla f(x_j)\|^2}
        + \frac{2}{35}\gamma^2\tau \sigma^2T
        + \frac{6}{71}\gamma^2(\tau_C-1)\sigma^2.
        \end{align*}
        We continue as follows
        \begin{align*}
        &\sum_{t=1}^{T}\E{\|y_t-\wtilde y_t\|^2}
        \le 
        \frac{44}{7}\gamma^2\sum_{t=1}^{T}\E{\wtilde{\phi}_t^{t-1}(x_{r(t)})}
        + \frac{6}{71}\gamma^2(\tau_C-1)^2G^2T
        +\frac{1}{70}\sum_{t=1}^{T}\E{\|y_t-y_{\alpha_t}\|^2}\\
        &\quad 
        + \frac{1}{15975L^2}\sum_{t=1}^{T}\E{\|\nabla f(x_t)\|^2}
        + \frac{2}{35}\gamma^2\tau \sigma^2T
        + \frac{6}{71}\gamma^2(\tau_C-1)\sigma^2.
    \end{align*}    
\end{proof}

\begin{lemmasec}\label{lem:lemmaE4} If $30L\gamma\wtau_{\max} \le 1$, then
    \begin{align}\label{eq:lemmaE4_1}
        \sum_{t=1}^{T}\E{\|y_t-y_{\alpha_t}\|^2} &\le
        \frac{2}{99}\gamma^2(\tau_C-1)^2G^2T 
        + \frac{100}{33}\gamma^2\sum_{t=1}^{T}\E{\left\|\sum_{j=\alpha_t}^{t-1}\nabla f_{k_t}(x_{\alpha_j}) - \nabla f(x_{\alpha_j})\right\|^2}\notag\\
        & \quad 
        + \frac{1}{297L^2}\sum_{t=1}^{T}\E{\|\nabla f(x_t)\|^2}
        + \frac{\gamma}{15L}\sigma^2T
        + \frac{2}{99}\gamma^2(\tau_C-1)\sigma^2T.
    \end{align}
\end{lemmasec}
\begin{proof}
    We have 
    \begin{align*}
        \E{\|y_t-y_{\alpha_t}\|^2} &=  \gamma^2\E{\left\|\sum_{j=\alpha_t}^{t-1}g_{k_j}(x_{\alpha_j})\right\|^2}\\
        &\le \gamma^2\E{\left\|\sum_{j=\alpha_t}^{t-1}\nabla f_{k_j}(x_{\alpha_j}) \pm \nabla f(x_{\alpha_j}) \pm \nabla f(x_j)\right\|^2}
        + \gamma^2\tau_t\sigma^2\\
        &\le 
        3\gamma^2\E{\left\|\sum_{j=\alpha_t}^{t-1}\nabla f_{k_j}(x_{\alpha_j}) - \nabla f (x_{\alpha_j})\right\|^2} 
        + 3\gamma^2\E{\left\|\sum_{j=\alpha_t}^{t-1}\nabla f(x_{\alpha_j}) - \nabla f(x_j)\right\|^2}\\
        &\quad  
        + 3\gamma^2\E{\left\|\sum_{j=\alpha_t}^{t-1}\nabla f(x_{j})\right\|^2}
        + \gamma^2\tau_t\sigma^2\\
        &\le 3L^2\gamma^2\wtau_t\sum_{j=\alpha_t}^{t-1}\E{\|x_{\alpha_j} - x_j\|^2}
        + 3\gamma^2\E{\left\|\sum_{j=\alpha_t}^{t-1}\nabla f_{k_j}(x_{\alpha_j}) - \nabla f(x_{\alpha_j})\right\|^2}\\
        &\quad
        + 3\gamma^2\wtau_t\sum_{j=\alpha_t}^{t-1}\E{\|\nabla f(x_j)\|^2}
        + \gamma^2\tau_t\sigma^2.
    \end{align*}
    Then we add summation over the entire iterates and count the number of times each term appears 
    \begin{align*}
        \sum_{t=1}^{T}\E{\|y_t-y_{\alpha_t}\|^2}
        &\le 
        3L^2\gamma^2\wtau_{\max}\sum_{t=1}^{T}\sum_{j=\alpha_t}^{t-1}\E{\|x_{\alpha_j} - x_j\|^2}
        + 3\gamma^2\sum_{t=1}^{T}\E{\left\|\sum_{j=\alpha_t}^{t-1}\nabla f_{k_t}(x_{\alpha_j}) - \nabla f(x_{\alpha_j})\right\|^2}\\
        &\quad 
        + 3\gamma^2\wtau_{\max}\sum_{t=1}^{T}\sum_{j=\alpha_t}^{t-1}\E{\|\nabla f(x_j)\|^2}
        + \wtau_{\avg}\gamma^2\sigma^2T\\
        &\le 
        3L^2\gamma^2\wtau_{\max}^2\sum_{t=1}^{T}\E{\|x_{\alpha_t} - x_t\|^2}  
        + 3\gamma^2\sum_{t=1}^{T}\E{\left\|\sum_{j=\alpha_t}^{t-1}\nabla f_{k_t}(x_{\alpha_j}) - \nabla f(x_{\alpha_j})\right\|^2}\\
        &\quad 
        + 3\gamma^2\wtau_{\max}^2\sum_{t=1}^{T}\E{\|\nabla f(x_t)\|^2}
        + \wtau_{\avg}\gamma^2\sigma^2T\\
        &\le 
        \frac{1}{300}\sum_{t=1}^{T}\E{\|x_{\alpha_t} - x_t\|^2}
        + 3\gamma^2\sum_{t=1}^{T}\E{\left\|\sum_{j=\alpha_t}^{t-1}\nabla f_{k_t}(x_{\alpha_j}) - \nabla f(x_{\alpha_j})\right\|^2}\\
        &\quad 
        + \frac{1}{300L^2}\sum_{t=1}^{T}\E{\|\nabla f(x_t)\|^2}
        + \wtau_{\avg}\gamma^2\sigma^2T,
    \end{align*}
    where in the second inequality each term in the double sum appears at most $\wtau_{\max}$ times. Using Young's inequality we continue as follows
    \begin{align*}
    	\sum_{t=1}^{T}\E{\|y_t-y_{\alpha_t}\|^2}
        &\le 
        \frac{1}{100}\sum_{t=1}^{T}\E{\|x_{\alpha_t} - y_{\alpha_t}\|^2+\|y_{\alpha_t} - y_{t}\|^2 + \|x_{t} - y_{t}\|^2}
        + \wtau_{\avg}\gamma^2\sigma^2T\\
        &\quad 
        + 3\gamma^2\sum_{t=1}^{T}\E{\left\|\sum_{j=\alpha_t}^{t-1}\nabla f_{k_t}(x_{\alpha_j}) - \nabla f(x_{\alpha_j})\right\|^2}
        + \frac{1}{300L^2}\sum_{t=1}^{T}\E{\|\nabla f(x_t)\|^2}\\
        &\le 
        \frac{1}{50}\gamma^2T((\tau_C-1)^2G^2T 
        + (\tau_C-1)\sigma^2)
        + \frac{1}{100}\sum_{t=1}^{T}\E{\|y_{\alpha_t}-y_t\|^2} 
        + \wtau_{\avg}\gamma^2\sigma^2T\\
        &\quad 
        + 3\gamma^2\sum_{t=1}^{T}\E{\left\|\sum_{j=\alpha_t}^{t-1}\nabla f_{k_t}(x_{\alpha_j}) - \nabla f(x_{\alpha_j})\right\|^2}
        + \frac{1}{300L^2}\sum_{t=1}^{T}\E{\|\nabla f(x_t)\|^2},
    \end{align*}
    where we again use Lemma~\ref{lem:lemmaE1}. After cancellation, we get
    \begin{align*}
        \sum_{t=1}^{T}\E{\|y_t-y_{\alpha_t}\|^2} 
        &\le  
        \frac{2}{99}\gamma^2(\tau_C-1)^2G^2T 
        + \frac{100}{33}\gamma^2\sum_{t=1}^{T}\E{\left\|\sum_{j=\alpha_t}^{t-1}\nabla f_{k_t}(x_{\alpha_j}) - \nabla f(x_{\alpha_j})\right\|^2}\\
        &\quad 
        + \frac{1}{297L^2}\sum_{t=1}^{T}\E{\|\nabla f(x_t)\|^2}
        + \frac{\gamma}{15L}\sigma^2T
        + \frac{2}{99}\gamma^2(\tau_C-1)\sigma^2T,
    \end{align*}
    where we use the stepsize restriction.
\end{proof}

We can combine all previous lemmas into one.
\begin{lemmasec}\label{lem:lemmaC6} 
If $30\gamma L\tau \le 1$ and $30\gamma L \max\{\wtau_{\max},\tau_C\} \le 1$ hold, then
\begin{align}\label{eq:lemmaC6_1}
    \sum_{t=1}^{T}\E{\|y_t-\wtilde{y}_{t}\|^2} &\le 
     8\gamma^2 \wtilde{\Phi}
        + \gamma^2\wtilde{\Psi}
        + \frac{1}{10}\gamma^2(\tau_C-1)^2G^2T
        + \frac{1}{10}\gamma^2(\tau_C-1)^2G^2T
        + \frac{1}{9033L^2}B\notag \\
        &\quad + \frac{1}{525L}\gamma\sigma^2T.
\end{align}
\end{lemmasec}
\begin{proof}
    Summary of obtained inequalities
    \begin{eqnarray*}
        \sum_{t=1}^{T}\E{\|y_t-\wtilde{y}_t\|^2} &
        \overset{\text{Lemma}~\ref{lem:lemmaE3}}{\le}& 
        \frac{44}{7}\gamma^2 \wtilde{\Phi}
        + \frac{6}{71}\gamma^2(\tau_C-1)^2G^2T 
        + \frac{1}{70}\wtilde{A}+ \frac{1}{15975L^2}B
        + \frac{2}{35}\gamma^2\tau\sigma^2T\\
        && \;+\;
        \frac{6}{71}\gamma^2(\tau_C-1)\sigma^2T,\\
        \wtilde{A} 
        &\overset{\text{Lemma}~\ref{lem:lemmaE4}}{\le}& 
        \frac{2}{99}\gamma^2(\tau_C-1)^2G^2T
        + \frac{2}{99}\gamma^2(\tau_C-1)\sigma^2T
        + \frac{100}{33}\gamma^2\wtilde{\Psi} + \frac{1}{297L^2}B\\
        && \;+\;
        \frac{\gamma}{15L}\sigma^2T.
    \end{eqnarray*}
    Hence, using the stepsize restriction $30L\gamma L \tau_C \le 1$
    \begin{align*}
        &\sum_{t=1}^{T}\E{\|y_t-\wtilde{y}_t\|^2} \le \frac{44}{7}\gamma^2 \wtilde{\Phi}
        + \frac{6}{71}\gamma^2(\tau_C-1)^2G^2T 
        + \frac{1}{15975L^2}B
        + \frac{1}{525L}\gamma\sigma^2T
        + \frac{6}{71}\gamma^2(\tau_C-1)\sigma^2T\\
        &\quad + \frac{1}{70}\left(
            \frac{2}{99}\gamma^2(\tau_C-1)^2G^2T 
            + \frac{2}{99}\gamma^2(\tau_C-1)\sigma^2T 
            + \frac{100}{33}\gamma^2\wtilde{\Psi} 
            + \frac{1}{297L^2}B
            + \frac{\gamma}{15L}\sigma^2T
            \right)\\
        &\le 
        8\gamma^2 \wtilde{\Phi}
        + \frac{1}{10}\gamma^2(\tau_C-1)^2G^2T
        + \frac{1}{10}\gamma^2(\tau_C-1)\sigma^2T
        + \frac{1}{9033L^2}B
        + \frac{1}{525L}\gamma\sigma^2T
        + \frac{10}{231}\gamma^2\wtilde{\Psi}.
    \end{align*}
\end{proof}

\subsection{Proof of Theorem~\ref{th:theorem4}}

\theoremfourth*
\begin{proof}
Again, we analyze separately iterations of restarts and without.

{\bf Iterations without restart:} If restarts do not happen, namely $(t+1)\mod\tau\ne 0$ then from the smoothness assumption of $f$ and $\wtilde{y}_t$ update rule, we have
\begin{eqnarray*}
f(\wtilde{y}_{t+1})
&\le& f(\wtilde{y}_t)
      - \gamma\langle \nabla f(\wtilde{y}_t), \nabla f(x_t) \rangle
      + \frac{L\gamma^2}{2}\|\nabla f(x_t)\|^2 \\
&=&   f(\wtilde{y}_t)
      - \frac{\gamma}{2}\|\nabla f(\wtilde{y}_t)\|^2 - \frac{\gamma}{2}\|\nabla f(x_t)\|^2 + \frac{\gamma}{2}\|\nabla f(\wtilde{y}_t) - \nabla f(x_t)\|^2
      + \frac{L\gamma^2}{2}\|\nabla f(x_t)\|^2 \\
&\le& f(\wtilde{y}_t)
      - \frac{\gamma}{2}\|\nabla f(\wtilde{y}_t)\|^2 - \frac{\gamma}{2}\|\nabla f(x_t)\|^2 + \frac{L^2\gamma}{2}\|\wtilde{y}_t - x_t\|^2
      + \frac{L\gamma^2}{2}\|\nabla f(x_t)\|^2 \\
&\le& f(\wtilde{y}_t)
      - \frac{\gamma}{2}\|\nabla f(\wtilde{y}_t)\|^2 - \frac{\gamma}{3}\|\nabla f(x_t)\|^2 + L^2\gamma\|\wtilde{y}_t - y_t\|^2 + L^2\gamma\|x_t - y_t\|^2.
\end{eqnarray*}

{\bf Iterations with restart:} If a restart happens, namely $(t+1)\mod\tau = 0$ then
\begin{eqnarray*}
    \wtilde{y}_{t+1}
    &=& y_{t+1} = y_t - \gamma g_{k_t}(x_{\alpha_t}) \\
    &=& \wtilde{y}_t + (y_t - \wtilde{y}_t) - \gamma\nabla f(x_t) + (\gamma\nabla f(x_t) - \gamma g_{k_t}(x_{\alpha_t})) \\
    &=& \wtilde{y}_t -\gamma\nabla f(x_t) + \gamma \underbrace{\sum_{j=r(t)}^t \nabla f(x_j) - g_{k_j}(x_{\alpha_j})}_{= \wtilde{\Delta}_t^t}.
\end{eqnarray*}

Then we use smoothness of $f$ to get
\begin{eqnarray*}
&& f(\wtilde{y}_{t+1}) \\
&\le& f(\wtilde{y}_t) - \gamma\langle \nabla f(\wtilde{y}_t), \nabla f(x_t) - \wtilde{\Delta}_t^t \rangle + \frac{L\gamma^2}{2}\|\nabla f(x_t) - \wtilde{\Delta}_t^t\|^2 \\
&\le& f(\wtilde{y}_t)
    - \gamma\langle \nabla f(\wtilde{y}_t), \nabla f(x_t) \rangle
    + \gamma\langle \nabla f(\wtilde{y}_t), \wtilde{\Delta}_t^t \rangle
    + L\gamma^2\|\nabla f(x_t)\|^2
    + L\gamma^2\|\wtilde{\Delta}_t^t\|^2 \\
&\le& f(\wtilde{y}_t)
    -\frac{\gamma}{2}\|\nabla f(\wtilde{y}_t)\|^2
    -\frac{\gamma}{2}\|\nabla f(x_t)\|^2
    +\frac{\gamma}{2}\|\nabla f(\wtilde{y}_t) - \nabla f(x_t)\|^2 \\
    && \;+\; \frac{1}{240L}\|\nabla f(\wtilde{y}_t)\|^2 +  60L\gamma^2\|\wtilde{\Delta}_t^t\|^2
    + L\gamma^2\|\nabla f(x_t)\|^2
    + L\gamma^2\|\wtilde{\Delta}_t^t\|^2 \\
&\le& f(\wtilde{y}_t)
    - \frac{\gamma}{2}\|\nabla f(\wtilde{y}_t)\|^2
    - \frac{\gamma}{3}\|\nabla f(x_t)\|^2
    + \frac{L^2\gamma}{2}\|\wtilde{y}_t - x_t\|^2
    + \frac{1}{ 240L}\|\nabla f(\wtilde{y}_t)\|^2
    + 61L\gamma^2\|\wtilde{\Delta}_t^t\|^2\\
&\le& f(\wtilde{y}_t)
    - \frac{\gamma}{2}\|\nabla f(\wtilde{y}_t)\|^2
    - \frac{\gamma}{3}\|\nabla f(x_t)\|^2
    + L^2\gamma\|\wtilde{y}_t - y_t\|^2 + L^2\gamma\|x_t - y_t\|^2
    + \frac{1}{ 160L}\|\nabla f(\wtilde{y}_t)\|^2\\
    && \quad + 61L\gamma^2\|\wtilde{\Delta}_t^t\|^2,
\end{eqnarray*}
where in the third and fourth inequality Young's inequality is used. 

If we denote by $\xi_t$ the indicator function of restart event at $t+1$, namely $\xi_{k\tau-1}=1$ for all $k\ge1$ and is $0$ otherwise, then we can take expectation and  unify the descent inequality for both cases as follows:
\begin{eqnarray}
\E{f(\wtilde{y}_{t+1})}
&\le& \E{f(\wtilde{y}_t)}
    - \frac{\gamma}{2}\E{\|\nabla f(\wtilde{y}_t)\|^2}
    - \frac{\gamma}{3}\E{\|\nabla f(x_t)\|^2}
    + L^2\gamma\E{\|\wtilde{y}_t - y_t\|^2} \notag\\
    && \;+\; L^2\gamma\E{\|y_t - x_t\|^2}
    + \left(
        \frac{1}{ 240L}\E{\|\nabla f(\wtilde{y}_t)\|^2}
        +  61L\gamma^2\E{\|\wtilde{\Delta}_t^t\|^2} \right)\xi_t\label{eq:theorem4_1}
\end{eqnarray}
for all $t \ge 0.$

Next, we apply summation over the entire iterates and bound the terms that appear only in every $\tau$ iteration.

\begin{eqnarray*}
\frac{1}{L}\E{\|\nabla f(\wtilde{y}_t)\|^2}
&=& \frac{1}{L\tau}\sum_{j=0}^{\tau-1} \E{\|\nabla f(\wtilde{y}_t)\|^2} \\
&\le& \frac{2}{L\tau}\sum_{j=0}^{\tau-1} \E{\|\nabla f(\wtilde{y}_t) - \nabla f(\wtilde{y}_{t-j})\|^2}
    + \frac{2}{L\tau}\sum_{j=0}^{\tau-1}\E{\|\nabla f(\wtilde{y}_{t-j})\|^2} \\
&\le& \frac{2L}{\tau}\sum_{j=0}^{\tau-1} \E{\|\wtilde{y}_t - \wtilde{y}_{t-j}\|^2}
    + \frac{2}{L\tau}\sum_{j=0}^{\tau-1} \E{\|\nabla f(\wtilde{y}_{t-j})\|^2 }\\
&\le& \frac{2L\gamma^2}{\tau}\sum_{j=0}^{\tau-1} \E{\left\| \sum_{l=t-j}^{t-1} \nabla f(x_l) \right\|^2}
    + \frac{2}{L\tau}\sum_{j=0}^{\tau-1} \E{\|\nabla f(\wtilde{y}_{t-j})\|^2} \\
&\le& 2L\gamma^2 \sum_{j=0}^{\tau-1}\sum_{l=t-j}^{t-1} \E{\left\| \nabla f(x_l) \right\|^2}
    + \frac{2}{L\tau}\sum_{j=0}^{\tau-1} \E{\|\nabla f(\wtilde{y}_{t-j})\|^2} \\
&\le& 2L\gamma^2\tau \sum_{j=0}^{\tau-1} \E{\left\| \nabla f(x_{t-j}) \right\|^2}
    + \frac{2}{L\tau}\sum_{j=0}^{\tau-1} \E{\|\nabla f(\wtilde{y}_{t-j})\|^2} \\
&\le& \frac{\gamma}{15} \sum_{j=0}^{\tau-1} \E{\left\| \nabla f(x_{t-j}) \right\|^2}
    + 120\gamma \sum_{j=0}^{\tau-1} \E{\|\nabla f(\wtilde{y}_{t-j})\|^2}
\end{eqnarray*}
provided that $\frac{1}{60} \le L\gamma\tau \le \frac{1}{30}$ (e.g., $\tau=\lfloor \frac{1}{30L\gamma}\rfloor$). Then we can use this bound to derive

\begin{eqnarray}\label{eq:theorem4_2}
\sum_{t=1}^{T} \frac{1}{240L}\E{\|\nabla f(\wtilde{y}_t)\|^2} \xi_t
\le \frac{\gamma}{3600} \sum_{t=1}^{T}\E{\left\| \nabla f(x_{t}) \right\|^2}
    + \frac{\gamma}{2} \sum_{t=1}^{T}\E{\|\nabla f(\wtilde{y}_{t})\|^2}.
\end{eqnarray}
Then we use \eqref{eq:lemmaE2_1} to bound $\E{\|\Delta_t^t\|^2}$:
\begin{eqnarray*}
    L\gamma^2\sum_{t=1}^{T} \E{\|\wtilde{\Delta}_t^t\|^2}\xi_t
    &\le& 
    4L\gamma^2 \sum_{t=1}^{T}\E{\wtilde{\phi}_t^t(x_{r(t)})}\xi_t
    + \sum_{t=1}^{T}\frac{6L\gamma^2}{71}(\tau_C-1)^2G^2\xi_t 
    \\
    && \;+\; 
    \sum_{t=1}^{T}\frac{6L\gamma^2}{71}(\tau_C-1)\sigma^2\xi_t
    + \frac{90}{7}L^3\gamma^2\tau\sum_{t=1}^{T}\sum_{j=r(t)}^{t}\E{\|y_j-y_{\alpha_j}\|^2}\xi_t\\
    && 
    \;+\; \frac{48}{7}L^2\gamma^3\sum_{t=1}^{T}\sum_{j=r(t)}^{t-1}\E{\wtilde{\phi}_t^{j-1}(x_{r(t)})}\xi_t \\ 
    && 
    \;+\; \frac{4L\gamma^2\tau}{71}\sum_{t=1}^{T}\sum_{j=r(t)}^{t}\E{\|\nabla f(x_t)\|^2}\xi_t
    + \frac{2}{35}L\gamma^2\tau\sigma^2\frac{T}{\tau}.
\end{eqnarray*}
Notice that summation over the entire iterates with weights $\xi_t$ is equivalent to division by $\tau$.
\begin{align*}
&\sum_{t=1}^{T}\sum_{j=r(t)}^{t} \E{\wtilde{\phi}_t^{j-1}(x_{r(t)})}\xi_t
= \sum_{t=1}^{T}\sum_{j=r(t)}^{t-1} \E{\wtilde{\phi}_t^{j}(x_{r(t)})}\xi_t
= \sum_{t=1}^{T}\sum_{j=r(t)}^{t-1} \E{\wtilde{\phi}_j^{j}(x_{r(j)})}\xi_t\\
&\le \sum_{t=1}^{T} \E{\wtilde{\phi}_t^{t}(x_{r(t)})}.
\end{align*}
Hence,
\begin{eqnarray}\label{eq:theorem4_3}
    &&61L\gamma^2\sum_{t=1}^{T} \E{\|\wtilde{\Delta}_t^t\|^2}\xi_t\notag \\
    &\le&
    15100L^2\gamma^3\sum_{t=1}^{T}\E{\wtilde{\phi}_t^t(x_{r(t)})}
    + 6L\gamma^2\frac{T}{\tau}(\tau_C-1)^2G^2
    + 6L\gamma^2\frac{T}{\tau}(\tau_C-1)\sigma^2
     \notag \\
    && \;+\; 
    \frac{2}{35}L\gamma^2\sigma^2T
    + \frac{183}{7}L^2\gamma\sum_{t=1}^{T}\E{\|y_t-y_{\alpha_t}\|^2}
    + \frac{244}{71}L\gamma^2\tau \sum_{t=1}^{T} \E{\left\| \nabla f(x_{t}) \right\|^2}
    \notag\\
    &\overset{\text{Lemma }\ref{lem:lemmaE4}}{\le}&  
    15100L^2\gamma^3\sum_{t=1}^{T}\E{\wtilde{\phi}_t^t(x_{r(t)})}
    + 360L^2\gamma^3(\tau_C-1)^2G^2T
    + 360L^2\gamma^3(\tau_C-1)\sigma^2T
    \notag\\
    && \;+\; \frac{2}{35}L\gamma^2\sigma^2T
    + \frac{183}{7}L^2\gamma\left( 
    \frac{2}{99}\gamma^2(\tau_C-1)^2G^2T 
    + \frac{2}{99}\gamma^2(\tau_C-1)\sigma^2T
    + \frac{100}{33}\gamma^2\wtilde{\Psi}\right.\notag\\
    &&  \;+\; \left.\frac{1}{297L^2}B
    + \frac{\gamma}{15L}\sigma^2T
        \right) 
    + \frac{122}{1065}\gamma B\notag\\
    &\le&  
    15100L^2\gamma^3\sum_{t=1}^{T}\E{\wtilde{\phi}_t^t(x_{r(t)})}
    + 361L^2\gamma^3(\tau_C-1)^2G^2T
    + 361L^2\gamma^3(\tau_C-1)\sigma^2T
    \notag\\
    && \;+\; \frac{9}{5}L\gamma^2\sigma^2T
    + \frac{49837}{246015}\gamma B
    + 80L^2\gamma^3\wtilde{\Psi}.
\end{eqnarray}

Plugging \eqref{eq:theorem4_2} and \eqref{eq:theorem4_3}, lemmas~\ref{lem:lemmaE1} and \ref{lem:lemmaE3} into \eqref{eq:theorem4_1} and adding summation, we have 
\begin{eqnarray*}
\E{f(\wtilde{y}_{T+1}) - f(\wtilde{y}_1)}
&\le& 
    - \frac{\gamma}{2} \sum_{t=1}^{T}\E{\|\nabla f(\wtilde{y}_t)\|^2}
    - \frac{\gamma}{3} \sum_{t=1}^{T}\E{\|\nabla f(x_t)\|^2}\\
   && \;+\; L^2\gamma \sum_{t=1}^{T}\E{\|\wtilde{y}_t - y_t\|^2}
        + L^2\gamma \sum_{t=1}^{T}\E{\|x_t - y_t\|^2} \notag\\
    && \;+\;
        \frac{1}{240L} \sum_{t=1}^{T}\xi_t\E{\|\nabla f(\wtilde{x}_t)\|^2}
        +  61L\gamma^2 \sum_{t=1}^{T}\xi_t\E{\|\wtilde{\Delta}_t^t\|^2}\\
    &\le& 
    - \frac{\gamma}{2} \sum_{t=1}^{T}\E{\|\nabla f(\wtilde{y}_t)\|^2}
    - \frac{\gamma}{3}B
    + \frac{\gamma}{3600} B
    + \frac{\gamma}{2} \sum_{t=1}^{T}\E{\|\nabla f(\wtilde{y}_{t})\|^2}\\
    && \;+\; L^2\gamma\left(
        \frac{44}{7}\gamma^2 \wtilde{\Phi}
        + \frac{6}{71}\gamma^2(\tau_C-1)^2G^2T 
        + \frac{1}{70}\wtilde{A}+ \frac{1}{15975L^2}B
        + \frac{2}{35}\gamma^2\tau\sigma^2T \right.\\
    && \;+\;
        \left. \frac{6}{71}\gamma^2(\tau_C-1)\sigma^2T\right)    
    + L^2\gamma^3(\tau_C-1)^2G^2T
    + L^2\gamma^3(\tau_C-1)\sigma^2T\\
    && \;+\; 
    15100L^2\gamma^3\sum_{t=1}^{T}\E{\wtilde{\phi}_t^t(x_{r(t)})}
    + 361L^2\gamma^3(\tau_C-1)^2G^2T
    \notag\\
    && \;+\; 361L^2\gamma^3(\tau_C-1)\sigma^2T
    + 2L\gamma^2\sigma^2T
    + 80L^2\gamma^3\wtilde{\Psi}
    + \frac{49837}{246015}\gamma B.
\end{eqnarray*}
Rearranging terms we have
\begin{eqnarray*}
\E{f(\wtilde{y}_{T+1}) - f(\wtilde{y}_1)}
    &\le& 
    -\frac{\gamma}{3}B 
    + \frac{\gamma}{3600}B  
    + \frac{\gamma}{15975} B
    + \frac{49837\gamma}{246015}B
    + 4L\gamma^2\sigma^2T\\
    && \;+\; 364L^2\gamma^3(\tau_C-1)^2G^2T
    + 364L^2\gamma^3(\tau_C-1)\sigma^2T\\
    && \;+\; 8L^2\gamma^3\wtilde{\Phi}
    + 81L^2\gamma^3\wtilde{\Psi} 
    + \frac{\gamma}{20790}B
    +  15100L^2\gamma^3\sum_{t=1}^{T}\E{\wtilde{\phi}_t^t(x_{r(t)})}\\
    &\le& -\frac{\gamma}{7}B
    + 4L\gamma^2\sigma^2T
    + 8L^2\gamma^3\wtilde{\Phi}
    + 81L^2\gamma^3\wtilde{\Psi}\\
    && \;+\; 15100L^2\gamma^3\sum_{t=1}^{T}\E{\wtilde{\phi}_t^t(x_{r(t)})}
    + 364L^2\gamma^3(\tau_C-1)^2G^2T\\
    && \;+\; 364L^2\gamma^3(\tau_C-1)\sigma^2T.
\end{eqnarray*}

Using the same argument as in the proof of Theorem~\ref{th:theorem3}, the above leads to
\begin{eqnarray*}
    \frac{1}{T}\sum_{t=1}^{T}\E{\|\nabla f(x_t)\|^2} &\le& 
    \frac{7(f(y_1) - f^*)}{\gamma T} 
    + 2600L^2\gamma^2(\tau_C-1)^2G^2
    + 2600L\gamma\sigma^2
    \\
    && \;+\; 106000L^2\gamma^2\left(\frac{1}{\lfloor T/\tau\rfloor}\sum_{k=0}^{\lfloor T/\tau\rfloor} \wtilde{\sigma}_{k,\tau}^2 + \frac{1}{T}\wtilde{\nu}^2\right).
\end{eqnarray*}
\end{proof}

\subsection{Convergence Guarantees in Special Cases}\label{sec:special_cases_real_theorem4}

\subsubsection{Random Asynchronous SGD}

\begin{algorithm*}[t]
\caption{Random Asynchronous SGD}
\label{alg:random_asynchronous}
\begin{algorithmic}[1]
\State \textbf{Input:} $x_0\in \R^{d}$, stepsize $\gamma > 0$, set of assigned jobs $\cA_0 = \empty$, set of received jobs $\cR_{0} = \empty$
\State \textbf{Initialization:} for all jobs $(i, 0) \in \cA_1$, server assigns worker $i$ to compute a stochastic gradient $g_i(x_0)$ 
    \For{$t = 0,1,2,\dots, T-1$}
		\State once worker $i_t$ finishes a job $(i_t,\pi_t) \in \cA_{t+1}$, it sends $g_{i_t}(x_{\pi_t})$ to the server
        \State server updates the current model $x_{t+1} = x_t - \gamma g_{i_t}(x_{\pi_t})$ and the set $\cR_{t+1} = \cR_{t} \cup \{(i_t, \pi_t)\}$
        \State server assigns worker $k_{t+1} \sim \text{Uni}[1, \dots, n]$ to compute a gradient $g_{k_{t+1}}(x_{t+1})$ \hfill $\triangleleft$ i.e., $\alpha_t \equiv t$
        \State server updates the set $\cA_{t+2} = \cA_{t+1} \cup \{(k_{t+1}, t+1)\}$
    \EndFor
\end{algorithmic}	
\end{algorithm*}

The algorithm \cite{koloskova2022sharper} is almost identical to pure asynchronous SGD. The difference in the assigning process; the server chooses a new worker uniformly at random from all the workers regardless it is busy or not. Such choice allows to equalize the contribution from all workers. In this case, we always assign a new job at the last available model, i.e. $\wtau_t \equiv 0.$

Now we derive the convergence guarantees for this algorithm. 

\begin{propositionsec}\label{prop:random_async_sgd}
	Let Assumptions ~\ref{asmp:smoothness}, \ref{asmp:bound_var}, \ref{asmp:grad_sim},  and \ref{asmp:bound_grad} hold. Let the stepsize satisfy $30L\tau_C\gamma \le 1,$ and $\tau = \lfloor \frac{1}{30L\gamma}\rfloor$. Then the iterates of Algorithm~\ref{alg:random_asynchronous} satisfy 
    \begin{equation}
			\E{\|\nabla f(\hat{x}_t)\|^2} \le  
			\cO\left(\frac{F_1}{\gamma T} 
            + L\gamma\sigma^2
            + L\gamma\zeta^2 
            + L^2\tau_C^2\gamma^2G^2 \right),
	\end{equation}
    where $\hat{x}_t$ is chosen uniformly at random from $\{x_1, \dots, x_T\}$ and $F_1 = f(y_1) - f^*$. Moreover, if we tune the stepsize, then the iterates of random asynchronous SGD satisfy 
    \begin{equation}
			\E{\|\nabla f(\hat{x}_t)\|^2} \le 
			\cO\left( \frac{LF_1\tau_C}{T} 
            + \left(\frac{LF_1\sigma^2}{T}\right)^{1/2}
            + \left(\frac{LF_1\zeta^2}{T}\right)^{1/2}
            + \left(\frac{F_1L\tau_CG}{T}\right)^{2/3}\right).
	\end{equation}
\end{propositionsec}
\begin{proof}
	The proof is similar to that of pure asynchronous SGD. In this case we have
	\[
	\sum_{t=k\tau}^{k\tau +j} (\nabla f_{k_t}(x) =  \sum_{i\in \cS} (\nabla f_i(x) - \nabla f(x)),\]
	where $|\cS| \le \tau$. Therefore, we can continue in the following way using the fact that $k_t$ is always sampled independently
	\begin{eqnarray*}
		\E{\left\|\sum_{t=k\tau}^{k\tau+j} (\nabla f_{k_t}(x)  - \nabla f(x))\right\|^2} &=& \E{\left\|\sum_{i\in \cS } (\nabla f_i(x) - \nabla f(x))\right\|^2}\\
		&=& \sum_{i\in \cS} \E{\|\nabla f_i(x) -\nabla f(x)\|^2}\\
		&\overset{\text{Asm. }\ref{asmp:grad_sim}}{\le}& |\cS|\zeta^2 \le \tau\zeta^2.
	\end{eqnarray*}
	This gives $\sigma_{k,\tau}^2 \le \tau\zeta^2$ bound. Since $\alpha_t\equiv t$ in this case, then $\nu^2 =0.$ This means that the rate is given by
    \begin{eqnarray*}
			\E{\|\nabla f(\hat{x}_t)\|^2} \le 
			\cO\left( \frac{F_1}{\gamma T}
            + L\gamma\sigma^2
            + L\gamma\zeta^2 
            + L^2\tau_C^2\gamma^2G^2 \right).
		\end{eqnarray*}
    Now we need to tune the stepsize. Now we have three cases to consider. Indeed,
  \begin{itemize}
      \item if $\gamma = \Theta\left(\frac{1}{L\tau_C}\right)$, then 
      \begin{eqnarray*}
            \E{\|\nabla f(\hat{x}_t)\|^2} &\le& 
			\cO\left( \frac{F_1}{\frac{1}{L\tau_C}T} 
            + L(\sigma^2+\zeta^2)\left(\frac{F_1}{LT(\sigma^2+\zeta^2)}\right)^{1/2} \right.\\
            && \;+\; \left. L^2\tau_C^2G^2 \left(\frac{F_1}{L^2\tau_C^2TG^2}\right)^{2/3}\right)\\
           &\le&  \cO\left( \frac{LF_1\tau_C}{T} 
           + \left(\frac{LF_1(\sigma^2+\zeta^2)}{T}\right)^{1/2} + \left(\frac{F_1L\tau_CG}{T}\right)^{2/3}\right).
      \end{eqnarray*}
      \item if $\gamma = \Theta\left(\left(\frac{F_1}{LT\sigma^2}\right)^{1/2}\right)$, then 
      \begin{eqnarray*}
          \E{\|\nabla f(\hat{x}_t)\|^2} &\le& 
			\cO\left( \frac{F_1}{T}\left(\frac{LT\sigma^2}{F_1}\right)^{1/2} 
            + L\sigma^2\left(\frac{F_1}{LT\sigma^2}\right)^{1/2}\right.\\
            && \;+\; \left.  L\zeta^2\left(\frac{F_1}{LT\zeta^2}\right)^{1/2}
            + L^2\tau_C^2G^2 \left(\frac{F_1}{L^2\tau_C^2TG^2}\right)^{2/3}\right)\\
           &\le&  \cO\left(  \left(\frac{LF_1\sigma^2}{T}\right)^{1/2} 
           + \left(\frac{LF_1\zeta^2}{T}\right)^{1/2} 
           + \left(\frac{F_1L\tau_CG}{T}\right)^{2/3}\right).
      \end{eqnarray*}
      \item if $\gamma = \Theta\left(\left(\frac{F_1}{LT\zeta^2}\right)^{1/2}\right)$, then 
      \begin{eqnarray*}
          \E{\|\nabla f(\hat{x}_t)\|^2} &\le& 
			\cO\left( \frac{F_1}{T}\left(\frac{LT\zeta^2}{F_1}\right)^{1/2} 
            + L\sigma^2\left(\frac{F_1}{LT\sigma^2}\right)^{1/2}\right.\\
            && \;+\; \left. L\zeta^2\left(\frac{F_1}{LT\zeta^2}\right)^{1/2}
            + L^2\tau_C^2G^2 \left(\frac{F_1}{L^2\tau_C^2TG^2}\right)^{2/3}\right)\\
           &\le&  \cO\left(  \left(\frac{LF_1\sigma^2}{T}\right)^{1/2} 
           + \left(\frac{LF_1\zeta^2}{T}\right)^{1/2} 
           + \left(\frac{F_1L\tau_CG}{T}\right)^{2/3}\right).
      \end{eqnarray*}
      \item if $\gamma = \Theta\left(\left(\frac{F_1}{L^2\tau_C^2TG^2}\right)^{1/3}\right)$, then 
      \begin{eqnarray*}
            \E{\|\nabla f(\hat{x}_t)\|^2} &\le& 
			\cO\left( \frac{F_1}{T}\left(\frac{L^2\tau_C^2TG^2}{F_1}\right)^{1/3} 
            + L\sigma^2\left(\frac{F_1}{LT\sigma^2}\right)^{1/2}\right.\\
            && \;+\; L\zeta^2\left(\frac{F_1}{LT\zeta^2}\right)^{1/2}
            + \left.L^2\tau_C^2G^2 \left(\frac{F_1}{L^2\tau_C^2TG^2}\right)^{2/3}\right)\\
           &\le&  \cO\left(  \left(\frac{LF_1\sigma^2}{T}\right)^{1/2} 
           + \left(\frac{LF_1\zeta^2}{T}\right)^{1/2} 
           + \left(\frac{F_1L\tau_CG}{T}\right)^{2/3}\right).
      \end{eqnarray*}
  \end{itemize}
  To get the final rate after stepsize tunning we need to set the minimal stepsize over all cases.
\end{proof}

\subsubsection{Random asynchronous SGD with waiting}

\begin{algorithm*}[t]
\caption{Random Asynchronous SGD with waiting}
\label{alg:random_asynchronous_waiting}
\begin{algorithmic}[1]
\State \textbf{Input:} $x_0\in \R^{d}$, stepsize $\gamma > 0$, set of assigned jobs $\cA_0 = \empty$, set of received jobs $\cR_{0} = \empty$, batch size $b \ge 1$, gradient estimator $g = 0$ and number of received gradients $r = 0$
\State \textbf{Initialization:} for all jobs $(i, 0) \in \cA_1$, server assigns worker $i$ to compute a stochastic gradient $g_i(x_0)$ 
    \For{$t = 0,1,2,\dots, T-1$}
        \State server sets $g_{t, 0} = 0$ and $\cR_{t,0} = \cR_t$
        \For{$j = 1, \dots, b$}
            \State once worker $i_{t,j}$ finishes a job $(i_{t,j},\pi_{t,j}) \in \cA_{t+1}$, it sends $g_{i_{t,j}}(x_{\pi_{t,j}})$ to the server
            \State server updates $g_{t,j} = g_{t,j-1} + g_{i_{t,j}}(x_{\pi_{t,j}})$
            \State server updates the set $\cR_{t, j} = \cR_{t, j-1} \cup \{(i_{t,j}, \pi_{t,j})\}$
        \EndFor
        \State server updates the current model $x_{t+1} = x_t - \frac{\gamma}{b} g_{t, b}$ and set $\cR_{t+1} = \cR_{t,b}$
        \State server samples new batch of workers $\{k_{t+1, 1}, \dots, k_{t+1,b}\}$ of size $b$
        \State server assigns worker $k_{t+1, j}$ to compute a gradient $g_{k_{t+1,j}}(x_{t+1})$ for all $j\in[b]$ \hfill $\triangleleft$ i.e., $\alpha_{t,j} = t$
        \State server updates the set $\cA_{t+2} = \cA_{t+1} \cup \{(k_{t+1,1}, t+1)\} \cup \dots \cup \{(k_{t+1,b}, t+1)\}$
    \EndFor
\end{algorithmic}	
\end{algorithm*}

The algorithm we consider in this paragraph is a special case of FedBuff \cite{nguyen2022fedbuff} when we perform only one local step (because this is out of the scope of our work). The difference with the previous algorithm is that the server waits for the first $b$ workers before performing one step, i.e. one step has a form
\begin{equation}\label{eq:step_rand_async_sgd_waiting}
    z_{q+1} = z_q - \wtilde{\gamma}\sum_{k\in B_q}\nabla f_i(z_{\pi_{q,i}}),
\end{equation}
where $\wtilde{\gamma} = \frac{\gamma}{b}$. Next the server uniformly at random chooses new $b$ workers $B_q = \{k_{q,0}, \dots, k_{q,b-1}\}$, and send them new jobs. The update of $\{y_t\}_{t=0}^{T-1}$
\begin{eqnarray*}
    y_{qb+b} &=& y_{qb+b-1} - \wtilde{\gamma}\nabla f_{k_{q, b-1}}(x_{qb} )\\
    &\vdots&\\
    &=& y_{qb} - \wtilde{\gamma}\sum_{k \in B_q} \nabla f_{k}(x_{qb}),
\end{eqnarray*}
so that now it satisfies \eqref{eq:virtual_iterates_update}. As we can see, $\wtilde\tau_t$ changes from $0$ to $b-1$ within one batch, i.e. $\wtau_{\max} \le b.$ That is why we choose stepsize such that $30L\wtilde{\gamma} \max\{b,\tau_C\}\le 1$ and $6L\wtilde{\gamma} \le 1$, then all stepsize constraints are satisfied. Now we set $\tau = b\lfloor\frac{1}{30L\gamma}\rfloor$, so that $\tau$ is a multiple of $b$ which is needed to take correctly conditional expectations. This gives the bound $30L\frac{\gamma}{b}\tau\le 1$ we need in the proofs as well. 

\begin{propositionsec}\label{prop:random_async_sgd_waiting}
    Let Assumptions~\ref{asmp:smoothness}, \ref{asmp:bound_var}, \ref{asmp:grad_sim}, and \ref{asmp:bound_grad} hold. Let the stepsize satisfy $30L\gamma \le 1,$ and $\tau = b\lfloor\frac{1}{30L\gamma}\rfloor.$ Then the iterates of Algoritm~\ref{alg:random_asynchronous_waiting} with waiting satisfy
    \begin{equation}
            \E{\|\nabla f(\hat{x}_t)\|^2} \le 
			\cO\left( \frac{F_1}{\gamma T} 
            + L\gamma\frac{\zeta^2+\sigma^2}{b} 
            + L^2\gamma^2G^2\frac{\tau_C^2}{b^2} \right),
    \end{equation}
    where $\hat{x}_t$ is chosen uniformly at random from $\{x_1, \dots, x_b, \dots, x_{2b}, \dots, x_{Tb}\}$ and $F_1 = f(y_1) - f^*$. Moreover, if we tune the stepsize, then the iterates of random asynchronous SGD with waiting satisfy
    \begin{equation}
			\E{\|\nabla f(\hat{x}_t)\|^2} \le 
			\cO\left( \frac{LF_1\tau_C}{T} 
            + \left(\frac{LF_1\sigma^2}{bT}\right)^{1/2} 
            + \left(\frac{LF_1\zeta^2}{bT}\right)^{1/2} 
            + \left(\frac{F_1L\tau_CG}{Tb}\right)^{2/3}\right).
	\end{equation}
\end{propositionsec}
\begin{proof}
    We again start with stepsize restriction. The effective stepsize $\wtilde{\gamma} = \frac{\gamma}{b}$ should satisfy $6L\wtilde{\gamma} \le 1$ and $30L\wtilde{\gamma} \max\{b,\tau_C\} \le 1$. This is implied if $\gamma$  satisfies $30L\gamma\max\{b,\tau_C\} \le b$. Next, we choose $\tau = b\lfloor\frac{1}{30L\gamma}\rfloor$. First, this choice allows to apply conditional expectation correctly. Second, it satisfies the restriction $30L\wtilde{\gamma}\tau = 30L\frac{\gamma}{b}b\lfloor\frac{1}{30L\gamma}\rfloor \le 1.$

    Now we need to compute $\wtilde{\sigma}^2_{k,\tau}$ and $\wtilde{\nu}^2$ quantities. We start with the first one. Each chunk of size $\tau$ in this case consists of $m\ge 1$ full batches of size $b.$ Every batch is independent of others, therefore we need to compute sequence correlation within one batch and then sum altogether. The sequence correlation within one batch is the same as in the case of mini-batch SGD since a batch is chosen uniformly at random, thus, $\wtilde{\sigma}^2_{k,\tau} \le \tau\zeta^2.$ $\wtilde{\nu}^2$ term is similarly bounded by $Tb\zeta^2$ since for all delayed gradients within one batch the argument $\alpha_j$ is the same, and therefore we are able to take conditional expectation correctly. Thus, the rate we derive is
    \begin{eqnarray*}
        \frac{1}{Tb}\sum_{q=0}^{T-1}\sum_{l=0}^{b-1}\E{\|\nabla f(x_{qb+l})\|^2} &\le& 
			\cO\left( \frac{F_1}{\wtilde{\gamma} Tb}
            + L\wtilde{\gamma}\sigma^2
            + L^2\wtilde{\gamma}^2\tau\zeta^2 
            + L^2\wtilde{\gamma}^2b\zeta^2
            + L^2\tau_C^2\wtilde{\gamma}^2G^2 \right)\\
            &\le& \cO\left( \frac{F_1}{\gamma T} 
            + L\gamma\frac{\sigma^2+\zeta^2}{b} 
            + L^2\gamma^2G^2\frac{\tau_C^2}{b^2} \right).
    \end{eqnarray*} 
    It is left  to tune the stepsize. We have three cases similar to the previous case, therefore we skip computations for that part. 
    
\end{proof}

\noindent{\bf Remark.} We observe that waiting for $b$ workers improves the second and third terms. However, we pay for that by waiting time. In practice, there is a trade-off between the number of workers we need to wait and the convergence speed.

\subsubsection{Shuffled Asynchronous SGD {\bf [NEW]}}\label{sec:shuffled_acynchronous_SGD}

\begin{algorithm*}[t]
\caption{Shuffled Asynchronous SGD}
\label{alg:shuffled_asynchronous}
\begin{algorithmic}[1]
\State \textbf{Input:} $x_0\in \R^{d}$, stepsize $\gamma > 0$, set of assigned jobs $\cA_0 = \empty$, set of received jobs $\cR_{0} = \empty$, random perturbation of workers $\chi$, worker counter $r = 0$
\State \textbf{Initialization:} for all jobs $(i, 0) \in \cA_1$, server assigns worker $i$ to compute a stochastic gradient $g_i(x_0)$ 
    \For{$t = 0,1,2,\dots, T-1$}
        \State once worker $i_{t}$ finishes a job $(i_{t},\pi_{t}) \in \cA_{t+1}$, it sends $g_{i_{t}}(x_{\pi_{t}})$ to the server
        \State server updates the current model $x_{t+1} = x_t - \gamma g_{i_{t}}(x_{\pi_{t}})$ and the set $\cR_{t+1} = \cR_{t} \cup \{(i_{t}, \pi_{t})\}$
        \State server assigns worker $\chi(r)$ to compute a gradient $g_{\chi(r)}(x_{t+1})$ and updates $r \gets r+1$ \hfill $\triangleleft$ i.e., $\alpha_t \equiv t$
        \State server updates the set $\cA_{t+2} = \cA_{t+1} \cup \{(\chi(r), t+1)\}$
        \If{$r = n$}
        \State server samples new perturbation of workers $\chi$ and set $r = 0$
        \EndIf
    \EndFor
\end{algorithmic}	
\end{algorithm*}
This method is inspired by the superiority of SGD with random reshuffling and shuffle once. In random asynchronous SGD workers might have different amounts of work depending on the random seed. However, we want to ensure that all workers have the same number of jobs in order to have a balance between all of them, but we still want to do it in random order. Therefore, we sample a permutation $\chi_q$ of $[n]$ at epoch $q$ and then assign new jobs according to $\chi_q.$ We can sample permutation before every epoch, or just once in the beginning. Here we want to utilize all available resources, i.e. $\tau_C = n.$ 

Now we present the method formally. Before epoch $q$ we sample a permutation $\chi_q$ of $[n].$ Then one epoch of shuffled asynchronous SGD has the following form
\begin{eqnarray}
        x_{qn+l+1} &=& x_{qn+l} - \gamma\nabla f_{i_{qn+l}}(x_{\pi_{qn+l}}) \quad \forall l \in [0, n-1],
\end{eqnarray}
where $\pi_{qn+l}$ is the iteration counter of the model where the corresponding gradient has started to be computed. The sequence of virtual iterates $\{y_t\}_{t=0}^{T}$ is computed as follows
\begin{eqnarray}
    y_{qn+l+1} = y_{qn+l} - \gamma\nabla f_{\chi_q(l)}(x_{qn}).
\end{eqnarray}
Again, $\wtau_t \equiv 0,$ therefore, the stepsize restrictions are $30L\gamma n \le 1$ and $6L\gamma \le 1.$ We choose $\tau = n\lfloor\frac{1}{30Ln\gamma }\rfloor,$ so that $\tau$ is a multiple of $n$ in order to take conditional expectations correctly. Note that we satisfy the condition $30L\gamma\tau \le 30L\gamma n\frac{1}{30Ln\gamma} = 1$ as well. We are ready to apply Theorem~\ref{th:theorem4}.

\begin{propositionsec}\label{prop:shuffled_async_sgd}
    Let Assumptions~\ref{asmp:smoothness}, \ref{asmp:bound_var}, \ref{asmp:grad_sim}, and \ref{asmp:bound_grad} hold. Let the stepsize satisfy $30L\gamma n \le 1,$ and $\tau = n\lfloor\frac{1}{30Ln\gamma}\rfloor.$ Then the iterates of shuffled asynchronous SGD satisfy
    \begin{equation}
            \E{\|\nabla f(\hat{x}_t)\|^2} \le 
			\cO\left(\frac{F_1}{\gamma T}
            + L\gamma\sigma^2
            + L^2\gamma^2n\zeta^2 
            + L^2\gamma^2G^2\tau_C^2\right),
    \end{equation}
    where $\hat{x}_t$ is chosen uniformly at random from $\{x_1, \dots, x_n, \dots, x_{2n}, \dots, x_{Tn}\}$ and $F_1 = f(y_1) - f^*$. Moreover, if the stepsize is set as $\gamma = \min\left\{\frac{1}{30Ln}, \left(\frac{F_1}{L^2n\zeta^2T}\right)^{1/3}, \left(\frac{F_1}{L^2n^2G^2\zeta^2T}\right)^{1/3}\right\},$ then
    \begin{equation}
        \E{\|\nabla f(\hat{x}_t)\|^2} \le 
			\cO\left(\frac{LnF_1}{T} 
            + \left(\frac{LF_1\sigma^2}{T}\right)^{1/2}
            + \left(\frac{LF_1\sqrt{n}\zeta}{T}\right)^{2/3}+\left(\frac{LGnF_1}{T}\right)^{2/3}\right).
    \end{equation}
\end{propositionsec}
\begin{proof}
    We have already shown that the choice of stepsize $\gamma$ and period $\tau$ satisfy the conditions of Theorem~\ref{th:theorem4}. Now we need to bound $\wtilde{\sigma}^2_{k,\tau}$ and $\wtilde{\nu}^2$. This is done analogously to SGD with random reshuffling. We have that $\wtilde{\sigma}^2_{k,\tau} \le n\zeta^2,$ and $\wtilde{\nu}^2 = 0$ since we do not have delays during assigning. Thus the convergence guarantees for this method are 
    \begin{eqnarray*}
        \E{\|\nabla f(\hat{x}_t)\|^2} &\le& \cO\left(\frac{F_1}{\gamma T} 
        + L\gamma\sigma^2
        + L^2\gamma^2n\zeta^2 + L^2\gamma^2G^2n^2\right),
    \end{eqnarray*}
    since $\tau_C = n$ in this case. We observe that both third term depends on $\gamma^2$ while the convergence of random asynchronous SGD the term with $\zeta^2$ depends on $\gamma$ only.

    Similarly to the previous two cases, we can tune the stepsize. We skip this part as it is almost exactly the same.
        
\end{proof}

\begin{remark}
    Let us set $\tau_C = n$ in random asynchronous SGD \cite{koloskova2022sharper} in order to compare it with shuffled asynchronous SGD proposed in this work. The difference in the rates comes from $\left(\frac{LF_1\zeta^2}{T}\right)^{1/2}$ for random asynchronous SGD and $\left(\frac{LF_1\sqrt{n}\zeta}{T}\right)^{2/3}$ for shuffled asynchronous SGD. Both terms become dominate in highly heterogeneous regime, i.e. when $\zeta^2$ is large. If we want to achieve $\E{\|\nabla f(\hat{x}_t)\|^2} \le \varepsilon$, then random asynchronous SGD requires $\cO\left(\frac{LF_1\zeta^2}{\varepsilon^2}\right)$ iterations while it is $\cO\left(\frac{LF_1\sqrt{n}\zeta}{\varepsilon^{3/2}}\right)$ for the shuffled asynchronous SGD. This means that shuffled asynchronous SGD needs less iterations if $\zeta \ge \sqrt{n}\varepsilon^{1/2}$, i.e. if we consider strongly heterogeneous regime which is typically the case in Federated Learning.
\end{remark}

\subsubsection{Pure Asynchronous SGD}
We consider exactly the same algorithm that was covered in Proposition~\ref{prop:pure_async_sgd} but with an additional assumption of bounded gradients. In this case $\alpha_t \equiv t,$ i.e. $\wtilde{\nu}^2 = 0.$ Similarly, we bound $\wtilde{\sigma}^2_{k,\tau} \le \tau^2\zeta^2.$ Applying Theorem~\ref{th:theorem4} we derive the following convergence guarantees.

\begin{propositionsec}\label{prop:pure_async_sgd_bounded}
    Let Assumption~\ref{asmp:smoothness}, \ref{asmp:bound_var}, \ref{asmp:grad_sim}, and \ref{asmp:bound_grad} hold. Let the stepsize $\gamma$ satisfy conditions $30L\gamma\tau_C \le 1$ and $6L\gamma \le 1.$ Let $\tau = \lfloor\frac{1}{30L\gamma}\rfloor.$ Then the iterates of Algorithm~\ref{alg:pure_asynchronous} satisfy
    \begin{equation}
        \E{\|\nabla f(\hat{x}_t)\|^2} \le \cO\left(\frac{F_1}{\gamma T} + L^2\gamma^2\tau_C^2G^2 +\zeta^2\right),
    \end{equation}
    where $\hat{x}_t$ is chosen uniformly at random from $\{x_1, \dots, x_T\}$. Moreover, if we tune the stepsize, then the convergence rate is 
    \begin{equation}
        \E{\|\nabla f(\hat{x}_t)\|^2} \le \cO\left(\frac{L\tau_CF_1}{T} + \left(\frac{L\tau_CGF_1}{T}\right)^{2/3}+\zeta^2\right).
    \end{equation}
\end{propositionsec}
\begin{proof}
    Since $\wtau_t \equiv 0,$ then we have to satisfy two stepsize conditions that are indicated in the proposition statement. In fact, we only need to satisfy $30L\gamma\tau_C \le 1$ as then $6L\gamma \le 1$ will be automatically satisfied. Since we have $\wtilde{\sigma}^2_{k,\tau} \le \tau^2\zeta^2$ and $\wtilde{\nu}^2 =0,$ then applying Theorem~\ref{th:theorem4} we get
    \begin{eqnarray}
       \E{\|\nabla f(\hat{x}_t)\|^2}&\le& 
       \cO\left(\frac{F_1}{\gamma T} + L^2\gamma^2\tau_C^2G^2 + L^2\gamma^2\tau^2\zeta^2\right)\\
       &\le& \cO\left(\frac{F_1}{\gamma T} + L^2\gamma^2\tau_C^2G^2 +\zeta^2\right).
    \end{eqnarray}
    Now we switch to the tunning of the stepsize. We have two cases.
    \begin{itemize}
        \item if $\gamma = \Theta\left(\frac{1}{L\tau_C}\right),$ then 
        \begin{eqnarray*}
            \E{\|\nabla f(\hat{x}_t)\|^2}
            &\le& \cO\left(\frac{F_1}{\frac{1}{30L\tau_C} T} + L^2\tau_C^2G^2\left(\frac{F_1}{L^2\tau_C^2G^2T}\right)^{2/3}+\zeta^2\right)\\
            &\le& \cO\left(\frac{L\tau_CF_1}{T} + \left(\frac{L\tau_CGF_1}{T}\right)^{2/3}+\zeta^2\right).
        \end{eqnarray*}
        \item if $\gamma = \Theta\left(\left(\frac{F_1}{L^2\tau_C^2G^2T}\right)^{1/3}\right)$, then 
        \begin{eqnarray*}
            \E{\|\nabla f(\hat{x}_t)\|^2}
            &\le& \cO\left(\frac{F_1}{T} \left(\frac{L^2\tau_C^2G^2T}{F_1}\right)^{1/3}+ L^2\tau_C^2G^2\left(\frac{F_1}{L^2\tau_C^2G^2T}\right)^{2/3}+\zeta^2\right)\\
            &\le& \cO\left(\left(\frac{L\tau_CGF_1}{T}\right)^{2/3}+\zeta^2\right).
        \end{eqnarray*}
    \end{itemize}
    It is left to take the minimum over two cases. We highlight that the rate does not depend on the maximum delay, it only depends on $\tau_C$ which is proportional to $\tau_{\avg}$ \citep{koloskova2022sharper}.
\end{proof}

\subsubsection{Pure Asynchronous SGD with waiting}
We consider the same algorithm from Proposition~\ref{prop:pure_async_sgd_waiting}, but with additional gradient bound assumption~\ref{asmp:bound_grad}. If we apply Theorem~\ref{th:theorem4} we get the following statement. 
\begin{propositionsec}\label{prop:pure_async_sgd_waiting_bounded}
    Let Assumption~\ref{asmp:smoothness}, \ref{asmp:bound_var},  \ref{asmp:grad_sim}, and \ref{asmp:bound_grad} hold. Let the stepsize $\gamma$ satisfy conditions $30L\gamma\tau_C \le 1$ and $6L\gamma \le 1.$ Let $\tau = \lfloor\frac{1}{30L\gamma}\rfloor.$ Then the iterates of Algorithm~\ref{alg:pure_asynchronous_waiting} satisfy
    \begin{equation}
        \E{\|\nabla f(\hat{x}_t)\|^2} \le  \cO\left(\frac{F_1}{\gamma Tb} + L^2\gamma^2\tau_C^2G^2 + \zeta^2\right),
    \end{equation}
    where $\hat{x}_t$ is chosen uniformly at random from $\{x_1, \dots, x_b, x_{b+1}, \dots, x_{2b}, \dots, x_{Tb}\}$. Moreover, if we tune the stepsize, then  the convergence rate is 
    \begin{equation}
        \E{\|\nabla f(\hat{x}_t)\|^2} \le \cO\left(\frac{L\tau_CF_1}{Tb} + \left(\frac{L\tau_CGF_1}{Tb}\right)^{2/3}+\zeta^2\right).
    \end{equation}
\end{propositionsec}
\begin{proof}
    First, we start with the stepsize conditions. Note that when we rewrite the iterations of pure asynchronous SGD with waiting so that they suit~\eqref{eq:real_iterates_update}, then we increase $\tau_t$ by $b$ at most, but $\tau_C$ remains unchanged. This gives us even better improvement w.r.t. $b.$ Indeed, we still need to satisfy $30L\gamma\tau_C\le 1$ and $6L\gamma \le 1.$ If the first one holds, then the second one as well. Now we apply Theorem~\ref{th:theorem4}. Note that $T$ iterations of pure asynchronous SGD with waiting are $Tb$ iterations of Algorithm~\ref{alg:pseudocode}. Then we have 
    \begin{eqnarray*}
       \E{\|\nabla f(\hat{x}_t)\|^2} &\le& \cO\left(\frac{F_1}{\gamma Tb} + L^2\gamma^2\tau_C^2G^2 + L^2\gamma^2\tau^2\zeta^2\right)\\
        &\le& \cO\left(\frac{F_1}{\gamma Tb} + L^2\gamma^2\tau_C^2G^2 + \zeta^2\right).
    \end{eqnarray*}
    If we choose the stepsize $\gamma = \Theta\left(\min\left\{\frac{1}{L\tau_C}, \left(\frac{F_1}{L^2\tau_C^2G^2Tb}\right)^{1/3}\right\}\right),$ then we get the second statement of the proposition.

\end{proof}

\end{document}